\documentclass[10pt]{article}

\RequirePackage{amsthm}
\let\amsproof\proof
\let\endamsproof\endproof
\let\amsnewtheorem\newtheorem
\RenewDocumentCommand\newtheorem{s m o m o}{}

\usepackage[authoryear]{fardlab}

\let\proof\amsproof
\let\endproof\endamsproof
\let\newtheorem\amsnewtheorem

\usepackage{needspace}
\preto\section{\needspace{5\baselineskip}}
\preto\subsection{\needspace{4\baselineskip}}

\usepackage[utf8]{inputenc}
\usepackage{makecell}
\usepackage{colortbl}
\usepackage{wrapfig}
\usepackage{float}
\usepackage{placeins}
\usepackage{pifont}
\usepackage{etoc}
\usepackage{thmtools}
\usepackage{thm-restate}
\usepackage[capitalize,noabbrev]{cleveref}
\tcbuselibrary{most}

\theoremstyle{plain}
\newtheorem{theorem}{Theorem}[section]
\newtheorem{proposition}[theorem]{Proposition}
\newtheorem{lemma}[theorem]{Lemma}
\newtheorem{corollary}[theorem]{Corollary}
\theoremstyle{definition}
\newtheorem{definition}{Definition}
\newtheorem{assumption}[theorem]{Assumption}
\theoremstyle{remark}

\theoremstyle{plain}

\crefname{theorem}{Theorem}{Theorems}
\Crefname{theorem}{Theorem}{Theorems}
\crefname{proposition}{Proposition}{Propositions}
\Crefname{proposition}{Proposition}{Propositions}
\crefname{lemma}{Lemma}{Lemmas}
\Crefname{lemma}{Lemma}{Lemmas}
\crefname{corollary}{Corollary}{Corollaries}
\Crefname{corollary}{Corollary}{Corollaries}
\crefname{definition}{Definition}{Definitions}
\Crefname{definition}{Definition}{Definitions}
\crefname{assumption}{Assumption}{Assumptions}
\Crefname{assumption}{Assumption}{Assumptions}
\crefname{remark}{Remark}{Remarks}
\Crefname{remark}{Remark}{Remarks}
\crefname{figure}{Figure}{Figures}
\crefname{table}{Table}{Tables}
\crefname{section}{Section}{Sections}
\crefname{subsection}{Section}{Sections}
\crefname{subsubsection}{Section}{Sections}
\crefname{appendix}{Appendix}{Appendices}
\crefname{subappendix}{Appendix}{Appendices}
\crefname{subsubappendix}{Appendix}{Appendices}

\definecolor{improvetxt}{RGB}{34,110,54}
\definecolor{worsetxt}{RGB}{160,40,40}

\definecolor{tolGrey}{HTML}{DDDDDD}
\definecolor{tolBlue}{HTML}{77AADD}
\definecolor{tolRose}{HTML}{CC6677}
\definecolor{tolPurple}{HTML}{AA4499}
\tcbset{thmbox/.style={
    enhanced, arc=2pt, boxrule=0.4pt, colframe=black,
    left=6pt, right=6pt, top=4pt, bottom=4pt,
    before skip=8pt, after skip=8pt}}
\makeatletter
\newcommand{\tcbwrapenv}[2]{%
  \expandafter\pretocmd\csname #1\endcsname{%
    \begin{tcolorbox}[{savedelimiter={#1},#2,wrap@environment,%
      code={\def\tcb@end@tcolorboxenvironment{\end{tcolorbox}}}}]}{}%
    {\PackageError{main}{cannot wrap #1}{}}%
  \expandafter\apptocmd\csname end#1\endcsname{%
    \tcb@end@tcolorboxenvironment}{}%
    {\PackageError{main}{cannot close #1}{}}%
}
\makeatother
\tcolorboxenvironment{definition}{thmbox, colback=tolGrey!35}
\tcbwrapenv{theorem}{thmbox, colback=tolBlue!22}
\tcbwrapenv{proposition}{thmbox, colback=tolBlue!22}
\tcbwrapenv{corollary}{thmbox, colback=tolBlue!22}
\tcbwrapenv{lemma}{thmbox, colback=tolBlue!12}
\tcolorboxenvironment{assumption}{thmbox, colback=tolRose!18}
\tcolorboxenvironment{remark}{thmbox, colback=tolPurple!14}
\newcommand{\ciSE}[2]{\makecell{$#1$\\{\tiny$\pm#2$}}}

\newtcolorbox{runex}{
  colback=gray!5,
  colframe=gray!40,
  arc=2pt, boxrule=0.4pt,
  left=6pt, right=6pt,
  top=4pt, bottom=4pt,
  title=Running Example,
  fonttitle=\bfseries\small,
  coltitle=black,
  colbacktitle=gray!15,
}

\newtcolorbox{outputbox}[1]{
  colback=gray!5,
  colframe=gray!40,
  arc=2pt, boxrule=0.4pt,
  left=6pt, right=6pt,
  top=4pt, bottom=4pt,
  title={#1},
  fonttitle=\bfseries\small,
  coltitle=black,
  colbacktitle=gray!15,
  fontupper=\footnotesize,
}

\title{What Makes a Latent Thought Valid?\texorpdfstring{\\}{} Defining and Auditing Thought Representations in LLMs}
\shorttitle{What Makes a Latent Thought Valid? Defining and Auditing Thought Representations in LLMs}

\author{\texorpdfstring{Fahd Seddik\textsuperscript{1}\enspace Fatemeh Fard\textsuperscript{1}}{Fahd Seddik, Fatemeh Fard}}
\affil{1}{FARD Lab, University of British Columbia, Okanagan, Canada}
\makeatletter
\def\fard@authornotes{Correspondence: \href{mailto:fahd.seddik@ubc.ca}{fahd.seddik@ubc.ca}, \href{mailto:fatemeh.fard@ubc.ca}{fatemeh.fard@ubc.ca}}
\makeatother

\projecturl[fard-lab.github.io/valid-thoughts]{https://fard-lab.github.io/valid-thoughts}

\begin{document}
\etocdepthtag.toc{main}
\maketitle

\begin{abstract}
Chain-of-Thought (CoT) reasoning improves the performance of Large Language Models (LLMs), yet it writes steps as text and follows one reasoning path at a time, which limits the capabilities of such models.
Latent reasoning methods replace these steps with continuous thought representations and claim that a thought representation substitutes for the written reasoning, compresses trajectories, distinguishes semantic content, and achieves superposition of multiple exploration paths at once. These claims concern the content of the thought representation, whereas the evaluations of these methods rely on downstream task accuracy, token counts, and analyses specific to each method. We formalize these claims as a valid thought representation space and prove that a space is valid if and only if its thought representations attain four properties, namely causality, minimality, separability, and stability, each of which follows from one claim. Translating these properties into four measures and auditing three training-free and three trained methods shows gaps in these claims that accuracy may not necessarily reveal. Training-free methods fail all four properties and do not encode the reasoning they are meant to replace, whereas trained methods generally improve substitution. The four properties and measures therefore provide a theoretically grounded evaluation beyond downstream task accuracy and expose failure modes that may guide future development of latent reasoning methods.
\end{abstract}

\section{Introduction}
\label{introduction}

A language model may reason without writing its intermediate steps, by conditioning each step on a continuous thought representation instead of a token, a step called a latent step.
Latent reasoning methods construct such thought representations to replace part or all of the written reasoning~\citep{hao2024training,zhang2025softthinkingunlockingreasoning}.
The replacement is motivated by substitution, in that a thought representation encodes the content of the reasoning steps that the model would otherwise write as text~\citep{ccot,deng2025llm}.
A second motivation is superposition, in that one thought representation holds several candidate reasoning paths together, as in a breadth-first search~\citep{hao2024training,deng2025llm}.
These methods report a maintained or improved downstream task accuracy with fewer decoded tokens~\citep{zhang2025softthinkingunlockingreasoning,hao2024training,tan2025colar,mondorf2024accuracyevaluatingreasoningbehavior}.
Both motivations are claims about the content of the thought representation, whereas the reported results measure task accuracy and the number of tokens.

\begin{figure}[t]
  \centering
  \includegraphics[width=\linewidth]{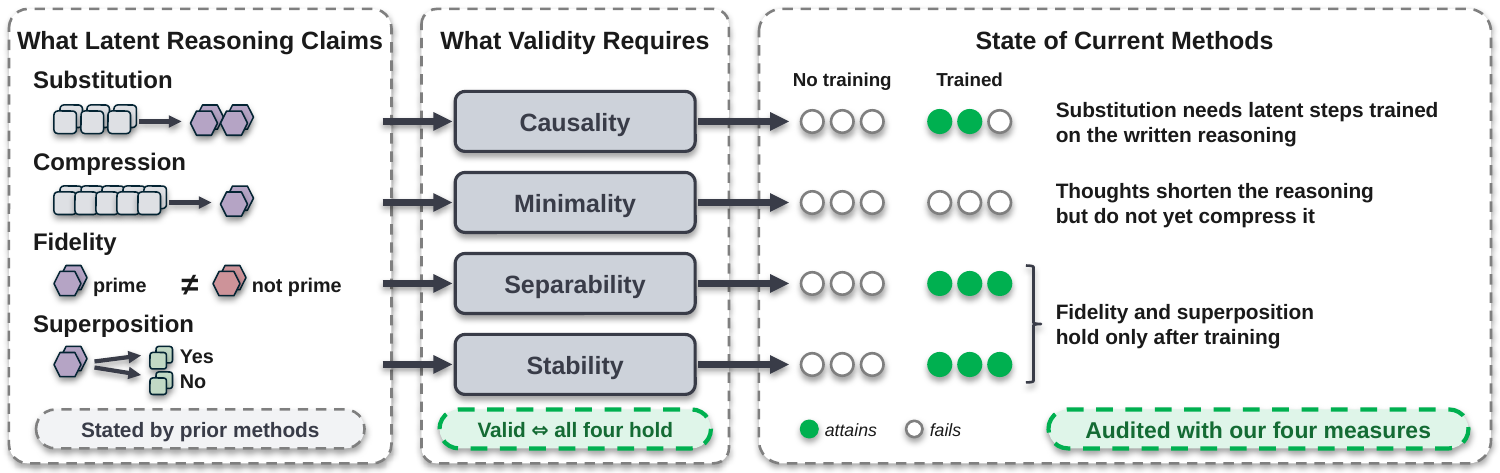}
  \caption{Latent reasoning methods claim that a thought representation substitutes for written reasoning, compresses it, keeps distinct semantic content distinct, and holds several answers together.
  Each claim corresponds to one property, and no current method attains all four.}
  \label{fig:teaser}
\end{figure}

Each latent reasoning method introduces a dedicated mechanism for constructing thought representations.
Coconut uses the last hidden state as the input embedding of the next position~\citep{hao2024training}, Soft Thinking weights the token embeddings by the next-token distribution~\citep{zhang2025softthinkingunlockingreasoning}, and Cache-to-Cache fuses a projection of the key-value cache of one model into another~\citep{fu2025cache}.
Each method further optimizes its mechanism toward an objective of its own design, such as the number of reasoning steps that condition on a thought representation, a token budget, or the imitation of written reasoning~\citep{hao2024training,shen2025codicompressingchainofthoughtcontinuous,wei2025simcot}. These differences in objectives and evaluation criteria hinder comparison across methods and the identification of stable design principles~\citep{yu2026latentspace}.
Therefore, which qualities a good thought representation should attain remains an open question. Audits find downstream task accuracy unchanged when the model bypasses the thought representations~\citep{rizvi-martel2026the,sahoo2026when,dilgren2026are}.
Since thought representations are opaque to direct human inspection, outcome-level performance is insufficient to guide progress in latent reasoning~\citep{yu2026latentspace}.
Downstream task accuracy thus cannot attribute a gain or a loss to a specific quality of the thought representation.

Prior work nevertheless applies individual requirements on a thought representation implicitly.
Audits treat an answer that bypasses the thought representation, or that follows an inconsistent pathway, as a failure of latent reasoning~\citep{cui2026how,dilgren2026are,sahoo2026when}.
Methods that supervise thought representations against the written reasoning steps expect a thought representation to preserve the content of the replaced steps~\citep{shen2025codicompressingchainofthoughtcontinuous,wei2025simcot}.
Compression of the reasoning chain is the stated aim of thought representations that replace written reasoning~\citep{ccot,tan2025colar}.
Thought representations that lose their semantic diversity~\citep{wei2025simcot}, and distinct reasoning paths that collapse to one interpretation~\citep{rizvi-martel2026the}, count as failures as well.
These expectations recur across the construction, supervision, and audit of thought representations, which points to a common set of requirements on a thought representation. Furthermore, a systematic theory of latent representations remains underdeveloped~\citep{yu2026latentspace}.

\begin{quote}\itshape
What must a thought representation attain to substitute for written reasoning?
\end{quote}

We answer this question by defining when a thought representation space is valid.
The space holds the thought representations that a model assigns to its contexts.
In a valid thought representation space, the thought representation of a context substitutes for the generated text of that context while preserving the distribution over the subsequent tokens.
Two contexts further share a thought representation if and only if the model assigns the two contexts the same distribution over outputs. Validity thus states the claims of latent reasoning literature (in \cref{fig:teaser}), namely substitution, compression, fidelity, and superposition, as requirements on the thought representation space.
We show that validity is equivalent to four properties of a thought representation, namely causality, minimality, separability, and stability, which unify the individual requirements of prior work.
Hence, a thought representation that lacks any of the four properties would not satisfy the claims that motivate its method. We translate each property into a measure computed on the thought representation.
The measures identify the property that a thought representation lacks and quantify it, independently of downstream task accuracy and of human inspection.

The main contributions of this paper are as follows.
\begin{itemize}[leftmargin=*, itemsep=0.3em, parsep=0pt, topsep=0.3em]
\item \textbf{Definition of a valid thought representation.} Four properties that a thought representation may attain, namely causality, minimality, separability, and stability (\cref{sec:valid,sec:properties}).
We prove that each property is attainable, independent of the other three, and that, over a finite set of inputs, a model and a generator attain all four properties in their strongest form if and only if their thought representation space is valid (\cref{sec:property-set}).
\item \textbf{Measures independent of downstream task accuracy.} The lack of standardized evaluation protocols is a central open challenge of latent reasoning~\citep{yu2026latentspace}. We supply such a protocol by translating each property into a measure computed on the thought representation (\cref{sec:properties}). 
\item \textbf{Audit of existing thought representations.}
We audit six latent reasoning methods, three without training on five open-weight models over the 23 reasoning tasks of Big-Bench Extra Hard (BBEH)~\citep{kazemi2025bigbenchextrahard} and three trained on GSM8K~\citep{cobbe2021gsm8k} (\cref{sec:experiments}).
Every training-free method fails all four properties.
On the trained methods, the measures track the questions answered incorrectly and distinguish methods that reach similar accuracy.
\end{itemize}

\section{Valid Thought Representations}
\label{sec:valid}

\subsection{Semantic Equivalence}
\label{sec:semantic-equivalence}

Let $f_\theta$ denote a language model that generates an output $y \in \mathcal{Y}$ for an input $x \in \mathcal{X}$. Let $\Phi : \mathcal{Y} \to \mathcal{S}$ map each generated output to the element of the semantic output space $\mathcal{S}$ that the output expresses.

\begin{definition}[Semantic Equivalence]
\label{def:semantic-equivalence}
Two outputs $y, y' \in \mathcal{Y}$ are semantically equivalent, written $y \sim_{\mathrm{sem}} y'$, when
\begin{equation}
\label{eq:semantic-equivalence}
\Phi(y) = \Phi(y').
\end{equation}
\end{definition}

\subsection{Valid Thought Representation Space}
\label{sec:valid-space}

A thought representation is a latent state that mediates the transformation from the input space $\mathcal{X}$ to the semantic output space $\mathcal{S}$.
Given an input $x$, the model induces a distribution $P(Y \mid x)$ over possible outputs.
A thought representation encodes the sufficient statistics of that distribution.
We explicitly exclude interpretability, which is a property of an observer.
Optimality in mediating the transformation may render a thought representation opaque to human understanding.
Let $\mathcal{T}$ denote the thought representation space, the set of thought representations that the model assigns to its inputs.

\begin{definition}[Thought Generator]
\label{def:generator}
Let $g : \mathcal{X} \to \mathcal{T}$ assign a thought representation to each input, and let $y_i$ denote the output that $f_\theta$ generates for $x_i$ under a fixed decoding rule.
The generator induces an equivalence relation on $\mathcal{X}$ through the semantic output space, in that for any pair of inputs $x_i, x_j \in \mathcal{X}$,
\begin{equation}
\label{eq:generator}
g(x_i) = g(x_j) \quad \text{if and only if} \quad y_i \sim_{\mathrm{sem}} y_j.
\end{equation}
\end{definition}

\begin{runex}
Let $x$ denote the question ``Is $13$ prime, and why?'', and assume that $f_\theta$ generates the outputs
\begin{align*}
y_1 &= \text{``Yes. $13$ is prime because no integer from $2$ to $12$ divides it.''} \\
y_1' &= \text{``Yes, $13$ has no divisors other than $1$ and itself.''} \\
y_2 &= \text{``No. $13$ is divisible by $3$.''}
\end{align*}
The pair $y_1$ and $y_1'$ encodes one answer in different wording and satisfies \cref{def:semantic-equivalence}, whereas $y_2$ answers the same question with a different value.
\end{runex}

The generator is many-to-one, which compresses $\mathcal{X}$ into $\mathcal{T}$ and preserves distinctness only between inputs whose outputs differ in semantic content.
Let $x'$ denote the input ``Determine whether $13$ is a prime number and justify the answer.'', and assume that $f_\theta$ generates an output $y'$ for $x'$ with $y' \sim_{\mathrm{sem}} y_1$.
\Cref{def:generator} then gives $g(x') = g(x)$.

A valid thought representation space maps every thought representation to a semantic output.
Conversely, every semantic output that the model can generate is mapped to a thought representation.
We therefore posit a bijection between the thought representation space $\mathcal{T}$ and the semantic output space $\mathcal{S}$.
However, before generation ends, the output is undetermined.
In this case, a thought representation encodes the distribution over the outputs that may follow its context.
Once generation ends, the output is fixed and it encodes the semantic content of that output.
Therefore, the bijection is stated on the distribution induced by a context.

Let $\mathcal{C}$ denote the set of contexts, where a context $c = [x; u]$ is an input $x \in \mathcal{X}$ together with the segment $u$ generated after that input.
The segment $u$ is empty before generation begins.
The generated output is then $y = [u; v]$, where $v$ denotes the tokens generated after the context.
The generator applies to every context, in that $\mathbf{T}_c = g(c)$ denotes the thought representation of a context and $\mathcal{T} = g(\mathcal{C})$ denotes the thought representation space.
Let $E(u)$ denote the token embeddings of the segment $u$ in $f_\theta$.
Let $[x; \mathbf{T}]$ denote the context in which $\mathbf{T}$ replaces the embeddings $E(u)$ of the generated segment.
Let $\Delta(\mathcal{Y})$ and $\Delta(\mathcal{S})$ denote the probability distributions on the output space and on the semantic output space,
\begin{equation}
\label{eq:simplex}
\Delta(\mathcal{A}) = \Big\{\, \tau : \mathcal{A} \to [0, 1] \ \text{ with } \ \sum_{a \in \mathcal{A}} \tau(a) = 1 \,\Big\}, \qquad \mathcal{A} \in \{\mathcal{Y}, \mathcal{S}\}.
\end{equation}

\begin{definition}[Valid Thought Representation Space]
\label{def:valid-space}
Let $\mathcal{C}_{\mathrm{in}}$ and $\mathcal{C}_{\mathrm{out}}$ denote the contexts whose generated segment is empty and the contexts whose generated segment is a complete output, let $\mathcal{T}_{\mathrm{in}}$ and $\mathcal{T}_{\mathrm{out}}$ denote their thought representations, and let $\mu_c$ denote the distribution that a context induces,
\begin{align}
\mathcal{C}_{\mathrm{in}} &= \{\, x : x \in \mathcal{X} \,\}, & \mathcal{C}_{\mathrm{out}} &= \{\, [x; y] : x \in \mathcal{X}, \; y \in \mathcal{Y} \,\}, \label{eq:context-strata} \\
\mathcal{T}_{\mathrm{in}} &= g(\mathcal{C}_{\mathrm{in}}), & \mathcal{T}_{\mathrm{out}} &= g(\mathcal{C}_{\mathrm{out}}), \label{eq:strata} \\
\mu_x &= P(Y \mid x) \in \Delta(\mathcal{Y}), & \mu_{[x;y]} &= P\big(\Phi(Y) \mid [x; y]\big) \in \Delta(\mathcal{S}). \label{eq:induced}
\end{align}
The space $\mathcal{T}$ is valid when the thought representation substitutes for the text at every context $[x; u]$ and every sequence $v$ of subsequent tokens,
\begin{equation}
\label{eq:valid-substitution}
P(v \mid [x; u]) = P\big(v \mid [x; \mathbf{T}_{[x;u]}]\big),
\end{equation}
and when, for each $i \in \{\mathrm{in}, \mathrm{out}\}$, two contexts in $\mathcal{C}_i$ have one thought representation if and only if the two contexts induce one distribution,
\begin{equation}
\label{eq:valid-bijection}
g(c_1) = g(c_2) \quad \Longleftrightarrow \quad \mu_{c_1} = \mu_{c_2}, \qquad c_1, c_2 \in \mathcal{C}_i .
\end{equation}
The map $q_i(g(c)) = \mu_c$ is then a bijection from $\mathcal{T}_i$ onto the distributions $\{\, \mu_c : c \in \mathcal{C}_i \,\}$ induced by the contexts in $\mathcal{C}_i$.
\end{definition}

On completed outputs, the bijection of \cref{def:valid-space} pairs each semantic output that the model generates with one thought representation (\cref{cor:semantic-bijection}).
\Cref{fig:valid-space} illustrates \cref{def:valid-space} on the running example.

\begin{figure}[t]
  \centering
  \includegraphics[width=\linewidth]{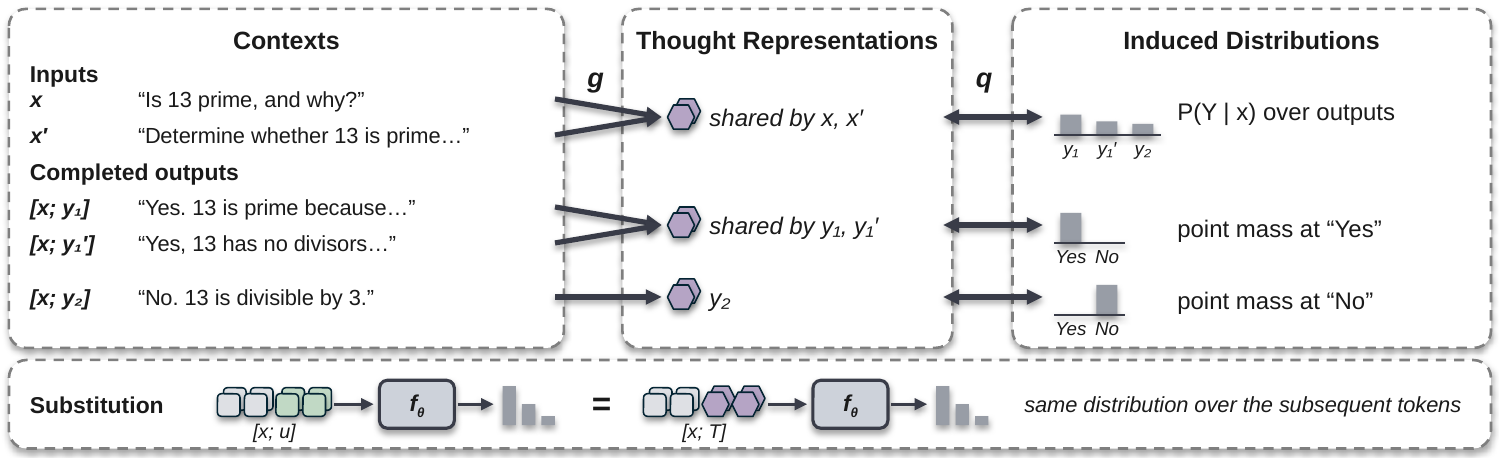}
  \caption{\textbf{Valid thought representation space.} The generator $g$ maps the inputs $x$ and $x'$, and maps the completed outputs $y_1$ and $y_1'$, to one representation, where each corresponds to exactly one induced distribution. Substituting a thought representation for its text preserves the distribution.}
  \label{fig:valid-space}
  \vspace{-1.0\baselineskip}
\end{figure}

\section{Characterizing Valid Thought Representations}
\label{sec:properties}

\subsection{The Four Properties and Measures}

\textbf{Causality.}
As thought representations are used to substitute explicit text reasoning, the thought representation should encode the necessary information about that reasoning.
Therefore, this motivates the first property to compare such substitution of text with a thought representation.

\begin{definition}[Causality]
\label{def:causality}
Causality holds when substituting $\mathbf{T}$ for the embeddings $E(u)$ inside $f_\theta$ leaves the distribution over $v$ unchanged,
\begin{equation}
\label{eq:causality}
P(v \mid [x; u]) = P(v \mid [x; \mathbf{T}]).
\end{equation}
\end{definition}

In the notation of \cref{sec:valid-space}, $\mathbf{T} = \mathbf{T}_{[x;u]}$ is the thought representation of the context $[x; u]$.
Let $u = y_1$ of the running example, and let $v$ denote the segment ``Therefore, $13$ is prime.'' generated after $u$.
Causality requires the distribution over $v$ to be the same whether $f_\theta$ conditions on the tokens of $y_1$ or on $\mathbf{T}$.

\textbf{Measure.}
Since the equality of \cref{def:causality} cannot be computed on a finite sample, the causality error reports the divergence between the two distributions as a surrogate.
The causality error takes $v$ as the final $n$ tokens of a generated output $y = [u; v]$ and averages the divergence between the next-token distributions $P$ of $f_\theta$ over these tokens,
\begin{equation}
\label{eq:causality-error}
\mathrm{KL}_{\mathrm{caus}}(\mathbf{T}) = \frac{1}{n} \sum_{t=1}^{n} D_{\mathrm{KL}}\Big( P(\cdot \mid [x; u], v_{<t}) \,\Big\|\, P\big(\cdot \mid [x; \mathbf{T}], v_{<t}\big) \Big).
\end{equation}
Let $\mathrm{KL}_{\mathrm{caus}}(\varnothing)$ denote the causality error omitting $\mathbf{T}$.
The reasoning recovered by $\mathbf{T}$ would be,
\begin{equation}
\label{eq:reasoning-recovered}
R(\mathbf{T}) = 1 - \frac{\mathrm{KL}_{\mathrm{caus}}(\mathbf{T})}{\mathrm{KL}_{\mathrm{caus}}(\varnothing)},
\end{equation}
which equals $1$ for $\mathbf{T} = E(u)$ and $0$ when $\mathbf{T} = \varnothing$.

\textbf{Minimality.} Latent reasoning methods compress reasoning trajectories into fewer steps than the written reasoning.
The number of steps limits how much information a thought representation may hold, but not which information it holds.
Therefore, the second property states compression on information, in that a thought representation keeps only what the output needs and discards the rest of the input. The mutual information $I(X; \mathbf{T})$ measures the content of the input that $\mathbf{T}$ encodes. The mutual information $I(\mathbf{T}; Y)$ measures the content of the generated output that $\mathbf{T}$ encodes.

\begin{definition}[Minimality]
\label{def:minimality}
Let $I_0$ be a threshold on the information that $\mathbf{T}$ retains about $Y$.
Minimality holds when $\mathbf{T}$ minimizes $I(X; \mathbf{T})$ subject to a constraint on $I(\mathbf{T}; Y)$,
\begin{equation}
\label{eq:minimality-constrained}
\min_{\mathbf{T}} \; I(X; \mathbf{T}) \quad \text{subject to} \quad I(\mathbf{T}; Y) \geq I_0,
\end{equation}
which the information bottleneck expresses as
\begin{equation}
\label{eq:minimality}
\min_{\mathbf{T}} \; J(\beta_{\mathrm{IB}}), \qquad J(\beta_{\mathrm{IB}}) = I(X; \mathbf{T}) - \beta_{\mathrm{IB}} \, I(\mathbf{T}; Y).
\end{equation}
\end{definition}

The weight $\beta_{\mathrm{IB}} > 0$ replaces the threshold $I_0$ by a penalty on $I(\mathbf{T}; Y)$.
Let $x''$ denote the input ``Hamlet was written by Shakespeare around $1600$. Is $13$ prime, and why?'', and assume that $f_\theta$ generates the output $y_1$ of the running example for it.
Since the literary sentence did not contribute to $y_1$, minimality requires $\mathbf{T}$ to discard the literary sentence.

\textbf{Measure.}
Since the two mutual informations of \cref{def:minimality} cannot be computed, the information bottleneck residual reports a surrogate built from two cross-entropies of $f_\theta$.
Let $\mathrm{CE}(Y \mid \mathbf{T})$ denote the cross-entropy of $f_\theta$ on the output after a prompt that supplies $\mathbf{T}$ and asks for the answer.
Let $\mathrm{CE}(X \mid Y, \mathbf{T})$ denote the cross-entropy of $f_\theta$ on the input after a prompt that also holds the output and asks to restate the question.
We compute the residual $\Delta_{\mathrm{IB}} = \mathrm{CE}(X \mid Y, \mathbf{T}) - \mathrm{CE}(Y \mid \mathbf{T})$, derived in Appendix~\ref{app:ib-derivation}.
The output information $r_Y$ and the extra input information $r_X$ normalize the two cross-entropies,
\begin{equation}
\label{eq:minimality-ratios}
r_Y = \frac{\mathrm{CE}(Y \mid \varnothing) - \mathrm{CE}(Y \mid \mathbf{T})}{\mathrm{CE}(Y \mid \varnothing) - \mathrm{CE}(Y \mid Y)}, \qquad
r_X = \frac{\mathrm{CE}(X \mid Y, \varnothing) - \mathrm{CE}(X \mid Y, \mathbf{T})}{\mathrm{CE}(X \mid Y, \varnothing) - \mathrm{CE}(X \mid Y, X)}.
\end{equation}
Each ratio equals $0$ for $\mathbf{T} = \varnothing$ and $1$ for the corresponding text.
The output text attains $r_Y = 1$ and $r_X = 0$, the values of a minimal thought representation.

\textbf{Separability.}
As semantically distinct outputs should have distinct thought representations, we should enforce separability among thought representations. A model with a finite vocabulary and a bounded output length generates finitely many outputs, which bounds the number of semantic outputs as well.
To compare thought representations by the semantic content of their outputs, we assume that the semantic output space would exist in a vector space with a metric $d_{\mathcal{S}}$, which at $d_{\mathcal{S}} = 0$ denotes semantic equivalence,

\begin{assumption}[Semantic Output Space]
\label{as:semantic-space}
A finite subset of a vector space, $\mathcal{S} \subset \mathbb{R}^d$, where $d_{\mathcal{S}}$ is a metric on $\mathbb{R}^d$,
\begin{equation}
\label{eq:semantic-zero}
d_{\mathcal{S}}\big(\Phi(y), \Phi(y')\big) = 0 \quad \text{if and only if} \quad y \sim_{\mathrm{sem}} y'.
\end{equation}
\end{assumption}

The metric is further assumed to order semantic outputs by subject, in that two answers to one question are closer to each other than to an answer to another question.

\begin{assumption}[Semantic Proximity]
\label{as:semantic-proximity}
For two outputs $y_1$ and $y_2$ that answer one question with different values, and an output $y_3$ that answers a question on another subject,
\begin{equation}
\label{eq:semantic-order}
d_{\mathcal{S}}\big(\Phi(y_1), \Phi(y_2)\big) < d_{\mathcal{S}}\big(\Phi(y_1), \Phi(y_3)\big).
\end{equation}
\end{assumption}

\begin{definition}[Separability]
\label{def:separability}
Let $c_1$ and $c_2$ be contexts whose outputs are not semantically equivalent.
Separability holds when some projection $\phi : \mathcal{T} \to \mathbb{R}^d$ separates the two representations by a margin $\delta \geq 0$,
\begin{equation}
\label{eq:separability}
d_{\mathcal{S}}\big(\phi(\mathbf{T}_{c_1}), \phi(\mathbf{T}_{c_2})\big) > \delta.
\end{equation}
\end{definition}

Inputs whose outputs are semantically equivalent share one representation under \cref{def:generator} and are indistinguishable under every projection.
Let $x'''$ denote the input ``Is $14$ prime, and why?'', and assume that $f_\theta$ generates the output $y_2' = {}$``No. $14 = 2 \times 7$.'', which is not semantically equivalent to $y_1$.
The two inputs differ by one token.
Separability nevertheless requires a projection $\phi$ that has $d_{\mathcal{S}} > \delta$ for them.

\textbf{Measure.}
Since neither $\Phi$ nor $d_{\mathcal{S}}$ can be computed on a trained model, the separability score replaces $\Phi(y)$ with the embedding of the output $y$ under a text embedding model, and $d_{\mathcal{S}}$ with the cosine similarity between embeddings.
Let $\bar{\mathbf{e}}_c$ denote the mean embedding of the outputs generated for a context $c$.
The projection $\phi$ is a linear map fitted by ridge regression from the thought representations of training contexts to their mean output embeddings.
On two held-out contexts $c_1$ and $c_2$, the projection separates $c_1$ from $c_2$ at margin $\delta$ when the projected thought representation of $c_1$ is closer to the outputs of $c_1$ than to the outputs of $c_2$,
\begin{equation}
\label{eq:separability-gap}
\cos\big(\phi(\mathbf{T}_{c_1}), \bar{\mathbf{e}}_{c_1}\big) - \cos\big(\phi(\mathbf{T}_{c_1}), \bar{\mathbf{e}}_{c_2}\big) > \delta.
\end{equation}
Since contexts from different tasks already differ in their text, the separability score is computed over pairs of held-out contexts from the same task.

\begin{figure}[t]
  \centering
  \includegraphics[width=\linewidth]{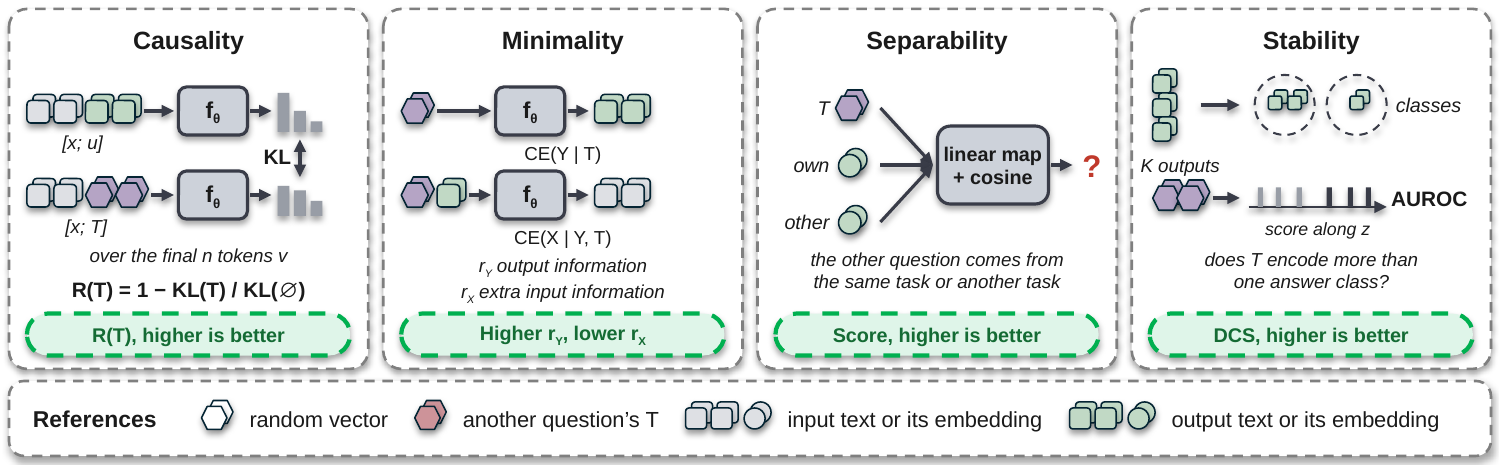}
  \caption{\textbf{Measures.} What enters each measure and what it returns, with the references.}
  \label{fig:measures}
\end{figure}

\textbf{Stability.} As one input admits many outputs and decoding samples one of them, a thought representation should encode the distribution over semantic outputs and remain invariant to the wording sampled. Therefore, the fourth property states one condition for the wording and one condition for the distribution.

\begin{definition}[Stability]
\label{def:stability}
Let $\mathbf{T}_y$ denote the representation associated with a generated output $y$.
Stability holds when semantically equivalent outputs give equal representations, which is \emph{lexical invariance},
\begin{equation}
\label{eq:lexical-invariance}
\mathbf{T}_{y_i} \approx \mathbf{T}_{y_j} \quad \text{whenever} \quad y_i \sim_{\mathrm{sem}} y_j,
\end{equation}
and when the semantic distribution that the input induces is recoverable from the representation, which is \emph{mode collapse resistance}, in that a map $q$ from $\mathcal{T}$ to the distributions over $\mathcal{S}$ satisfies, for every input $x$,
\begin{equation}
\label{eq:mode-collapse}
q(\mathbf{T}_x) = P\big(\Phi(Y) \mid X = x\big).
\end{equation}
\end{definition}

In the notation of \cref{sec:valid-space}, $\mathbf{T}_y = \mathbf{T}_{[x;y]}$ is the thought representation of the completed context $[x; y]$.
Lexical invariance removes the dependence of $\mathbf{T}$ on the wording of the sampled output.
Mode collapse resistance keeps the uncertainty of the semantic distribution in $\mathbf{T}$, hence a thought representation extracted from one sample reports the semantic distribution of the input.
The outputs $y_1$ and $y_1'$ of the running example are equivalent and therefore give approximately equal representations.
If the model answers ``Yes'' in some samples and ``No'' in others, stability then requires $\mathbf{T}$ to encode both classes.

\textbf{Measure.} Since the semantic distribution of \cref{def:stability} cannot be computed, the distributional consistency score approximates the distribution by $K$ outputs generated for the input.
Binarizing the cosine similarity between the embeddings of these outputs at a threshold $\tau$ partitions the outputs into classes $\mathcal{E}_1, \mathcal{E}_2, \dots$, whose sizes give the semantic entropy of the input $\mathbb{H}_x$~\citep{kuhn2023semantic}. A difference-of-means~\citep{cencerrado2026answerneededpredictingllm} direction recovers whether the input has more than one class,
\begin{equation}
\label{eq:dcs-probe}
z = \frac{1}{|\mathcal{D}_+|} \sum_{\mathbf{T} \in \mathcal{D}_+} \mathbf{T} - \frac{1}{|\mathcal{D}_-|} \sum_{\mathbf{T} \in \mathcal{D}_-} \mathbf{T},
\end{equation}
where $\mathcal{D}_+$ holds the thought representations of the inputs with $\mathbb{H}_x > 0$ and $\mathcal{D}_-$ the rest.
The Distributional Consistency Score (DCS) is the Area Under the Receiver Operating Characteristic curve (AUROC) of the score $\langle z, \mathbf{T} \rangle$ for separating $\mathcal{D}_+$ from $\mathcal{D}_-$.
\Cref{fig:measures} summarizes the four measures.

\subsection{Unified Framework}
\label{sec:property-set}

We prove consistency, independence, completeness, and necessity of the four properties. Details in Appendix~\ref{app:proofs}.

\textbf{Consistency.}
Let $\mathcal{P}$ denote the set of the four properties. Since properties that contradict one another would not describe a thought representation, we first study whether the properties in $\mathcal{P}$ are consistent, that is, whether one model and one thought representation may attain all of them.
The construction takes the thought representation of a context to be the distribution that the context induces over the semantic output space.
The model conditions the tokens generated after each context only on that distribution.

\begin{restatable}[Consistency]{proposition}{restPropConsistency}
\label{prop:consistency}
The properties in $\mathcal{P}$ are jointly achievable, in that one model and one thought representation may attain all four requirements.
\end{restatable}

\textbf{Independence.}
Consistency, however, does not imply that the properties in $\mathcal{P}$ are free of redundancy.
A consistent set may still hold a property that the remaining three imply, which would make that property redundant.
Therefore, we prove that no property in $\mathcal{P}$ is implied by the conjunction of the other three.
The four properties are requirements on a pair $(f_\theta, g)$ of a model and a generator.
The generator $g$ of such a pair is any map from contexts to a thought representation space, which may violate \cref{def:generator}.
For each property, Appendix~\ref{app:proofs} edits the thought representation of the consistency model once, such that the edited pair attains the other three properties and lacks that property.

\begin{restatable}[Independence]{corollary}{restCorIndependence}
\label{cor:independence}
No property in $\mathcal{P}$ is implied by the conjunction of the other three.
\end{restatable}

\textbf{Completeness.}
Only afterwards do we prove completeness, in the sense that $\mathcal{P}$ yields a valid thought representation space.
\Cref{def:valid-space} states the target space of the four properties.
Causality gives the substitution condition~\eqref{eq:valid-substitution}.
Separability and lexical invariance give the bijection~\eqref{eq:valid-bijection} on completed outputs.
Minimality gives the bijection on inputs.

\begin{restatable}[Completeness]{theorem}{restThmCompleteness}
\label{thm:completeness}
Let $\mathcal{X}$ be finite, and let a pair $(f_\theta, g)$ attain the four properties with the threshold of~\eqref{eq:minimality-constrained} at its largest attainable value, $I_0 = I(X; Y)$, and with lexical invariance~\eqref{eq:lexical-invariance} holding with equality.
Then its thought representation space is valid.
\end{restatable}

\textbf{Necessity.}
Completeness shows only that the four properties suffice for validity.
Therefore, we also prove the converse, that validity implies each of the four properties.
The substitution condition of \cref{def:valid-space} is causality.
The bijection on inputs gives minimality and mode collapse resistance.
Since a projection resolves any two distinct thought representations, both bijections give separability.
The bijection on completed outputs gives lexical invariance with equality.

\begin{restatable}[Necessity]{proposition}{restPropNecessity}
\label{prop:necessity}
Let $\mathcal{X}$ be finite, and let the thought representation space of a pair $(f_\theta, g)$ be valid.
Then $(f_\theta, g)$ attains causality at every context, minimality with the threshold $I_0 = I(X; Y)$, separability on $\mathcal{C}_{\mathrm{in}}$ and on $\mathcal{C}_{\mathrm{out}}$, and stability with lexical invariance~\eqref{eq:lexical-invariance} holding with equality.
\end{restatable}

Together, \cref{thm:completeness,prop:necessity} characterize validity, in that a thought representation space is valid if and only if the pair that generates it attains the four properties under the conditions of \cref{prop:necessity}.

\section{Experiments}
\label{sec:experiments}

\subsection{Experimental Setup}
\label{sec:setup}

\begin{figure}[h!]
  \centering
  \includegraphics[width=\linewidth]{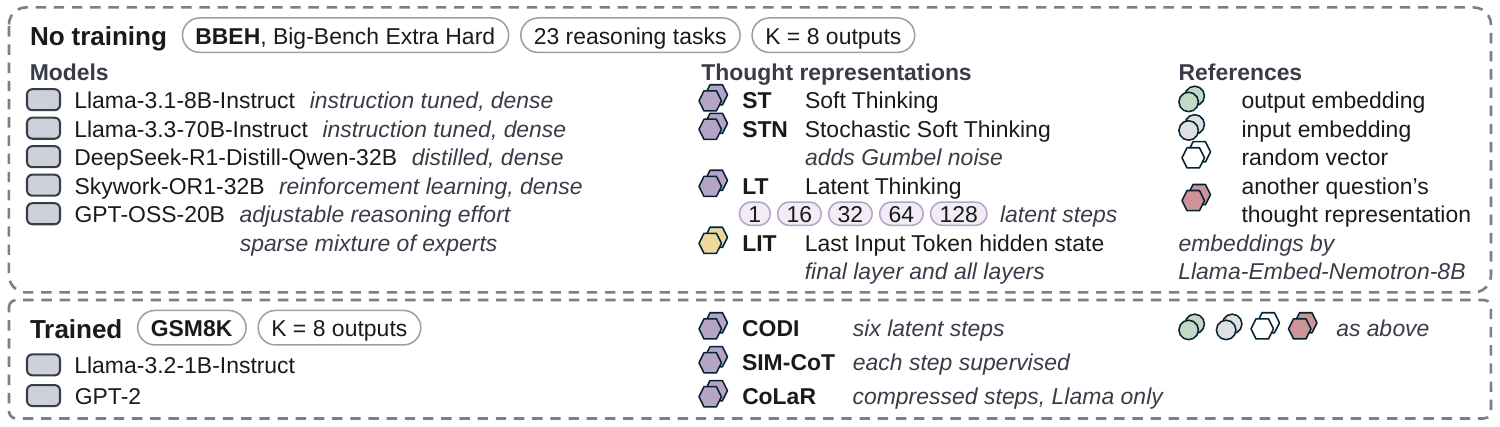}
  \caption{\textbf{Experimental setup.} The two audits, with the models, thought representations, and references.}
  \label{fig:setup}
\end{figure}

\textbf{Thought representations.}
\Cref{fig:setup} summarizes the two audits.
The audit without training takes the hidden state of the last input token, which encodes information about the tokens generated after the prompt~\citep{lugoloobi2026llms}, from the final layer and from all layers~\citep{cai2026t2mlr}.
The training-free methods are Soft Thinking (ST)~\citep{zhang2025softthinkingunlockingreasoning}, Stochastic Soft Thinking (STN), which adds Gumbel noise~\citep{wu2026llms}, and Latent Thinking (LT)~\citep{hao2024training,zou2025latentcollaborationmultiagentsystems}, each at $1$, $16$, $32$, $64$, and $128$ latent steps.
The trained audit takes CODI~\citep{shen2025codicompressingchainofthoughtcontinuous}, SIM-CoT~\citep{wei2025simcot}, and CoLaR~\citep{tan2025colar}, which fine-tune Llama-3.2-1B-Instruct~\citep{grattafiori2024llama3herdmodels} and GPT-2~\citep{radford2019language} to reason in latent space (details in Appendix~\ref{app:details}).

\textbf{Baselines.} To establish baseline comparisons for thought representations, we utilize reference representations with known information content to evaluate specific properties. (i) \emph{The input embedding} encodes only the question. A thought representation attains separability and stability when its measure exceeds this baseline, indicating it captures information about the outputs that the question itself lacks. (ii) \emph{A random vector} replaces the latent steps to ensure no question-specific information is present. A thought representation attains causality if the reasoning recovered from it is closer to that of the written reasoning ($R = 1$) than to the reasoning recovered from this random baseline (confirming the effect stems from the content of the steps, not just their presence). (iii) \emph{The output and question texts} serve as reference points in the $(r_Y, r_X)$ space relative to $\mathrm{CE}(Y \mid \varnothing)$ and $\mathrm{CE}(X \mid Y, \varnothing)$. A thought representation attains minimality when its coordinate is closer to the output text than to either the question text or the origin. Finally, to control for general task information that is not specific to the prompt, we utilize the thought representation of a different question from the same task as an additional baseline.

\textbf{Models and data.}
The five models used for the training-free audit span four model families, four training paradigms, two architectures, and four sizes~\citep{grattafiori2024llama3herdmodels,deepseekai2025deepseekr1incentivizingreasoningcapability,he2025skywork,skywork-or1-2025,openai2025gptoss120bgptoss20bmodel}.
Each model generates $K = 8$ outputs per problem of the $23$ reasoning tasks of BBEH~\citep{kazemi2025bigbenchextrahard}.
The trained methods generate $K = 8$ outputs per question of GSM8K~\citep{cobbe2021gsm8k}.
To compare the effect of training on one model and one task, the three training-free methods also run on Llama-3.2-1B-Instruct over GSM8K.

\subsection{Results}
\label{sec:results-main}

\begin{figure}[t]
  \centering
  \includegraphics[width=\linewidth]{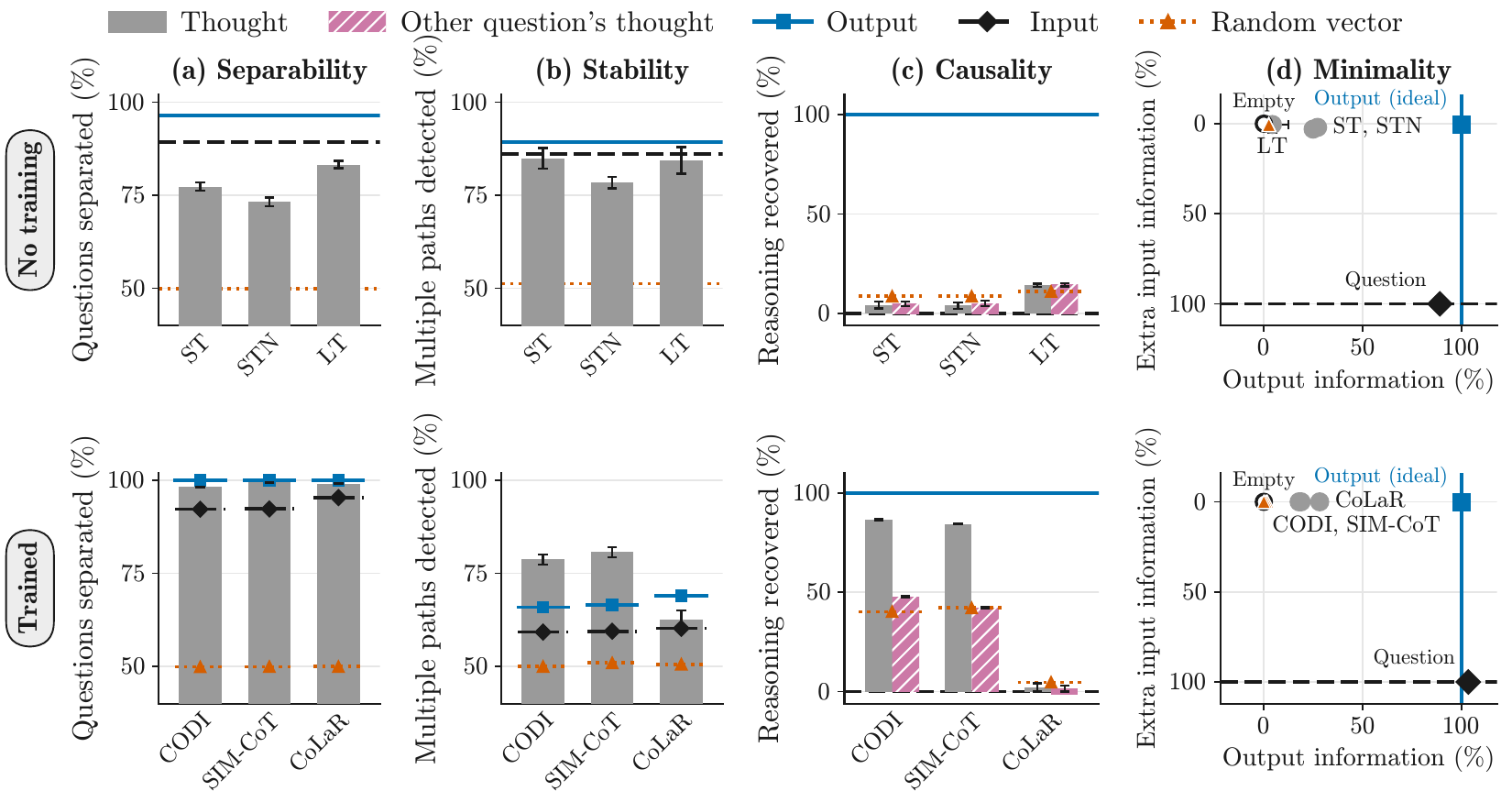}
  \caption{\textbf{Performance.} Top, no training at best step, averaged over models. Bottom, trained methods.}
  \label{fig:properties}
  \vspace{-1.0\baselineskip}
\end{figure}

\Cref{tab:attainment} marks the properties that each method attains.
Every training-free method fails all four properties, whereas CODI and SIM-CoT attain every property except minimality.

\begin{wraptable}{r}{0.5\linewidth}
\caption{Properties attained by each method.}
\label{tab:attainment}
\centering
\small
\setlength{\tabcolsep}{3.5pt}
\renewcommand{\arraystretch}{1.15}
\begin{tabular*}{\linewidth}{@{\extracolsep{\fill}}l*{6}{c}@{}}
    \toprule
    & \multicolumn{3}{c}{No training} & \multicolumn{3}{c}{Trained} \\
    \cmidrule(lr){2-4}\cmidrule(l){5-7}
    Property & ST & STN & LT & CODI & SIM-CoT & CoLaR \\
    \midrule
    Causality & \textcolor{black!35}{\ding{55}} & \textcolor{black!35}{\ding{55}} & \textcolor{black!35}{\ding{55}} & \textcolor{improvetxt}{\ding{51}} & \textcolor{improvetxt}{\ding{51}} & \textcolor{black!35}{\ding{55}} \\
    Minimality & \textcolor{black!35}{\ding{55}} & \textcolor{black!35}{\ding{55}} & \textcolor{black!35}{\ding{55}} & \textcolor{black!35}{\ding{55}} & \textcolor{black!35}{\ding{55}} & \textcolor{black!35}{\ding{55}} \\
    Separability & \textcolor{black!35}{\ding{55}} & \textcolor{black!35}{\ding{55}} & \textcolor{black!35}{\ding{55}} & \textcolor{improvetxt}{\ding{51}} & \textcolor{improvetxt}{\ding{51}} & \textcolor{improvetxt}{\ding{51}} \\
    Stability & \textcolor{black!35}{\ding{55}} & \textcolor{black!35}{\ding{55}} & \textcolor{black!35}{\ding{55}} & \textcolor{improvetxt}{\ding{51}} & \textcolor{improvetxt}{\ding{51}} & \textcolor{improvetxt}{\ding{51}} \\
    \bottomrule
\end{tabular*}
\end{wraptable}

\textbf{Causality.} 
As shown in \cref{fig:properties}(c), training-free methods fail to attain causality. They are generally unable to recover the written reasoning, and the part they do recover is not specific to the question (performing similarly to the control representation of a different question). In contrast, the trained thought representations of CODI and SIM-CoT successfully attain causality as they recover the majority of the reasoning and exceed the performance of both baseline controls.

\textbf{Minimality.} 
Regardless of whether a method is trained or training-free, none successfully attain minimality (\cref{fig:properties}(d)). Averaged over models, $\mathrm{CE}(Y \mid \mathbf{T})$ of every thought representation remains closer to $\mathrm{CE}(Y \mid \varnothing)$ than to $\mathrm{CE}(Y \mid Y)$. Furthermore, all methods exhibit near-zero values for $r_X$. Because $r_Y$ is also consistently low, this near-zero $r_X$ does not imply that the thought representation is successfully discarding irrelevant input; rather, it indicates that the model recovers only a limited portion of either kind of information.

\textbf{Separability and stability.} 
While training-free thought representations separate questions and identify reasoning paths more effectively than a random vector, their capabilities are bounded by the input embedding (\cref{fig:properties}(a,b)). The hidden state of the last input token shows the same bound (Appendix~\ref{app:results}). Consequently, these representations merely encode the question itself, lacking the reasoning that distinguishes specific outputs or the uncertainty surrounding them. Trained representations, however, achieve separability that surpasses the input embedding baseline. Furthermore, the thought representations of CODI and SIM-CoT detect multiple reasoning paths more accurately than even the output embedding, demonstrating that training successfully embeds both the generated reasoning and its associated uncertainty into the representation.

\subsection{Implications for Method Design}
\label{sec:analysis}

\begin{figure}[h!]
  \centering
  \includegraphics[width=\linewidth]{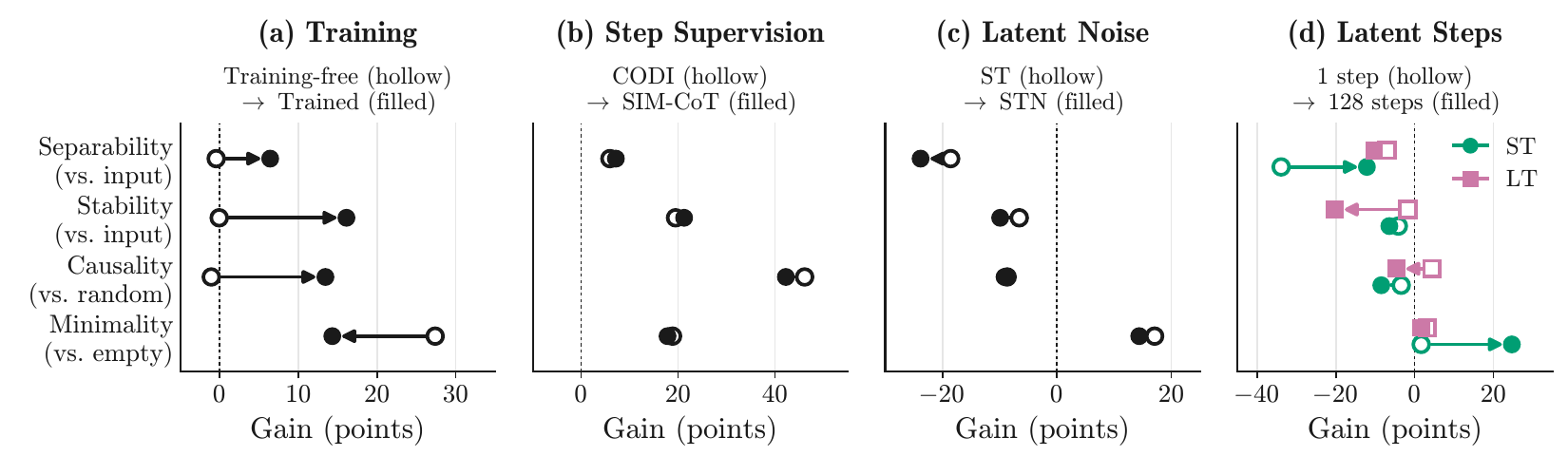}
  \caption{\textbf{Design choices.} Each arrow runs from the gain of one variant to that of the other.}
  \label{fig:design-choices}
\end{figure}

\begin{figure}[t]
  \centering
  \includegraphics[width=\linewidth]{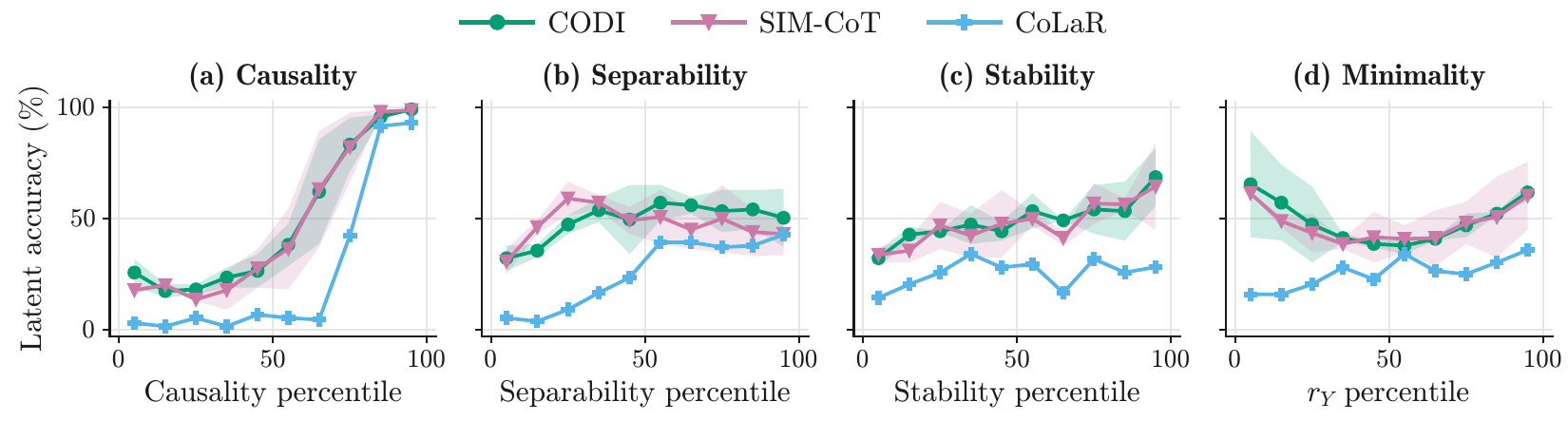}
  \caption{\textbf{Measures and accuracy.} Accuracy of each trained method on GSM8K test questions ranked by a property of their thought representation, with minimality ranked by its output information $r_Y$. CODI and SIM-CoT are averaged over Llama-3.2-1B-Instruct and GPT-2, and the band spans the two models.}
  \label{fig:measures-accuracy}
\end{figure}

\textbf{Design choices.} \Cref{fig:design-choices} isolates the impact of individual design choices by comparing pairs of variants that differ by a single factor. We evaluate these choices based on their gains over our established baselines.

The analysis reveals distinct trade-offs for each modification. Training improves causality, separability, and stability, though it comes at the cost of reduced minimality (\cref{fig:design-choices}a). Supervising each latent step, as SIM-CoT does over CODI, lowers the causality gain and leaves separability, stability, and minimality almost unchanged (\cref{fig:design-choices}b). Gumbel noise yields a small effect (\cref{fig:design-choices}c). The impact of adding more latent steps depends on the method (\cref{fig:design-choices}d). For Soft Thinking, additional steps improve separability and minimality at the expense of stability and causality, whereas for Latent Thinking, more steps degrade separability, stability, and causality. Every design choice thus negatively impacts at least one property, yet no single property is harmed by all choices. The cost of any design choice therefore becomes visible only when all four measures are evaluated together.

\textbf{The measures and accuracy.} To relate our proposed measures to the downstream accuracy of latent reasoning methods, \cref{fig:measures-accuracy} ranks the evaluated questions for each trained method according to the specific properties of their thought representations. This analysis demonstrates that incorrect answers consistently correlate with low scores across every measure. Specifically, downstream accuracy rises in tandem with causality (\cref{fig:measures-accuracy}a), and questions exhibiting the lowest separability and stability correspondingly yield the lowest accuracy (\cref{fig:measures-accuracy}b, c). Accuracy similarly increases alongside $r_Y$ whenever $\mathrm{CE}(Y \mid \mathbf{T})$ is below $\mathrm{CE}(Y \mid \varnothing)$ (\cref{fig:measures-accuracy}d). However, because $r_Y$ remains low for most representations, this $r_Y$ ranking primarily spans the lower end of the scale. Beyond simply predicting correctness, these measures provide a deeper diagnostic tool that distinguishes between methods of similar accuracy. For example, CODI and SIM-CoT reach similar accuracy on Llama-3.2-1B-Instruct, yet they differ in causality and minimality (\cref{tab:trained-results}).

\section{Related Work}
\label{sec:related}

\textbf{Constructing and supervising thought representations.}
Latent reasoning methods construct thought representations in a single model, through recurrence in depth, or across models.
In a single model, a thought representation is the last hidden state~\citep{hao2024training}, a mixture of token embeddings~\citep{zhang2025softthinkingunlockingreasoning,butt2026soft}, or a compressed token~\citep{ccot,tan2025colar}.
Recurrent-depth models iterate a block of layers on a latent state before decoding~\citep{jolicoeurmartineau2025morerecursivereasoningtiny,wang2026tiny,koishekenov2025encodethinkdecodescaling,zeng2025ponderlm2pretrainingllmlatent}.
Across models, a thought representation may be a key-value cache passed from one model to another~\citep{zou2025latentcollaborationmultiagentsystems,fu2025cache}.
Prior work trains these thought representations against the hidden states of a teacher that conditions on the written reasoning~\citep{shen2025codicompressingchainofthoughtcontinuous}, against the text of each reasoning step~\citep{wei2025simcot}, toward a compressed latent chain~\citep{tan2025colar,deng2025llm}, or through a reward on the answer~\citep{butt2026soft}.
Every mechanism and training signal above is specific to its method.
In contrast, \cref{def:valid-space} states validity on the thought representation, independently of its construction and training.

\begin{wrapfigure}{r}{0.35\linewidth}
  \vspace{-1.0\baselineskip}
  \centering
  \includegraphics[width=\linewidth]{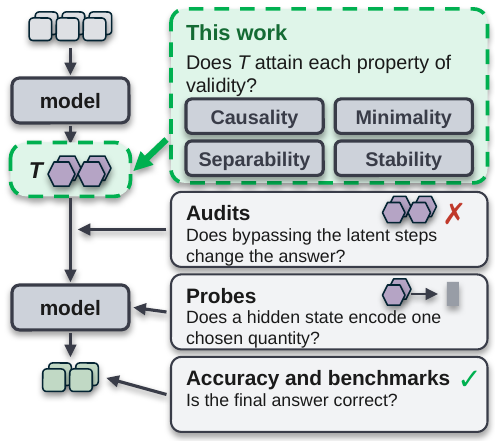}
  \caption{Positioning of this work.}
  \label{fig:positioning}
  \vspace{-1.0\baselineskip}
\end{wrapfigure}
\textbf{Evaluating thought representations.}
Audits of latent reasoning intervene on latent steps, decode latent tokens, or trace reasoning paths~\citep{dilgren2026are,cui2026how,rizvi-martel2026the,sahoo2026when,zhang2025latent,liang2026latentcot}.
Probing studies recover task correctness, uncertainty, or future tokens from hidden states~\citep{afzal2025knowingsayingllmrepresentations,cencerrado2026answerneededpredictingllm,lugoloobi2026llms,du2026latent}.
A benchmark of latent-space reasoning measures these abilities through task performance~\citep{hagendorff2025beyond}.
Each audit and each probe examines one method on its own data.
In contrast, the measures of \cref{sec:properties} are computed on the thought representation of any method and reported on one scale against shared references.
\Cref{fig:positioning} contrasts the objects of these evaluations.

\textbf{Theory of latent reasoning.} Latent reasoning lacks a systematic theory~\citep{yu2026latentspace}.
Theoretical analyses characterize what latent thought can compute, including the expressiveness of latent thought relative to chain of thought~\citep{xu2026formal}, the parallel search that continuous thoughts support through superposition~\citep{zhu2025reasoning}, and the benefits and limits of latent chain of thought~\citep{zou2026the}.
These analyses state what latent thought can compute, whereas \cref{sec:properties} unifies the requirements that a thought representation must attain.

\section{Conclusion}
\label{sec:conclusion}

Latent reasoning methods claim that a thought representation substitutes for written reasoning, compresses it, separates distinct outputs, and achieves superposition, whereas their evaluations report accuracy and token counts.
We state these claims as the validity of a thought representation space, prove that validity is equivalent to four properties, namely causality, minimality, separability, and stability, and translate each property into a measure computed on these representations for any method.
Training-free thought representations encode the question but not the reasoning they replace and consequently support none of the four claims.
Training improves substitution, fidelity, and superposition, yet no audited method compresses the reasoning beyond simply shortening the trajectory.
Incorrect answers correlate with a low score on each measure, which points to the property that a thought representation lacks, and methods of similar accuracy can be compared across the four measures.
Therefore, accuracy and token counts do not specify which claims a method satisfies, and the four measures answer that question.

\textbf{Implications.}
The results imply that
(i) an evaluation of a latent reasoning method should report the four measures alongside its accuracy,
(ii) the evaluation should include the thought representation of another question and a random vector as controls,
(iii) a method that claims substitution should train its latent steps on the written reasoning, and
(iv) its step supervision, noise, and number of latent steps should be selected against all four measures.

\textbf{Limitations.}
Each measure reports a surrogate for its property, which the audit compares with references of known information content.
Causality and minimality are computed in the model, hence they apply to thought representations in the input embedding space, whereas separability and stability apply to every thought representation.
The audit approximates the output distribution of a model by $K$ generated outputs, hence it measures stability on those outputs.
Since every audited method gives one thought representation per input, the audit measures stability through mode collapse resistance, and lexical invariance has no measure yet.
The audited tasks are English reasoning tasks, hence the findings may differ for multilingual workloads and for generations outside reasoning.
The trained audit covers the released checkpoints of each method, and extending it to further models is future work.

\FloatBarrier

\bibliography{references}

@inproceedings{afzal2025knowingsayingllmrepresentations,
    title = "Knowing Before Saying: {LLM} Representations Encode Information About Chain-of-Thought Success Before Completion",
    author = "Afzal, Anum and Matthes, Florian and Chechik, Gal and Ziser, Yftah",
    editor = "Che, Wanxiang and Nabende, Joyce and Shutova, Ekaterina and Pilehvar, Mohammad Taher",
    booktitle = "Findings of the Association for Computational Linguistics: ACL 2025",
    month = jul,
    year = "2025",
    address = "Vienna, Austria",
    publisher = "Association for Computational Linguistics",
    url = "https://aclanthology.org/2025.findings-acl.662/",
    doi = "10.18653/v1/2025.findings-acl.662",
    pages = "12791--12806",
    ISBN = "979-8-89176-256-5",
}

@article{Du_2025,
   title={{Human-like object concept representations emerge naturally in multimodal large language models}},
   volume={7},
   ISSN={2522-5839},
   url={http://dx.doi.org/10.1038/s42256-025-01049-z},
   DOI={10.1038/s42256-025-01049-z},
   number={6},
   journal={Nature Machine Intelligence},
   publisher={Springer Science and Business Media LLC},
   author={Du, Changde and Fu, Kaicheng and Wen, Bincheng and Sun, Yi and Peng, Jie and Wei, Wei and Gao, Ying and Wang, Shengpei and Zhang, Chuncheng and Li, Jinpeng and Qiu, Shuang and Chang, Le and He, Huiguang},
   year={2025},
   month=jun, pages={860–875} 
}

@InProceedings{zeng2025ponderlm2pretrainingllmlatent,
  title = 	 {{P}onder{LM}-2: Pretraining {LLM} with Latent Thoughts in Continuous Space},
  author =       {Zeng, Boyi and Li, He and Song, Shixiang and Wang, Yixuan and Wang, Zitong and He, Ziwei and Wang, Xinbing and Lin, Zhouhan},
  booktitle = 	 {Proceedings of the 43rd International Conference on Machine Learning},
  pages = 	 {153459--153482},
  year = 	 {2026},
  editor = 	 {Zhang, Tong and Dudik, Miroslav and Jaggi, Martin and Agarwal, Alekh and Li, Sharon and Schuurmans, Dale and Zhu, Jerry and Berkenkamp, Felix and Dong, Hanze and Bietti, Alberto},
  volume = 	 {306},
  series = 	 {Proceedings of Machine Learning Research},
  month = 	 {06--11 Jul},
  publisher =    {PMLR},
  url = 	 {https://proceedings.mlr.press/v306/zeng26l.html}
}

@inproceedings{hao2024training,
  title={Training Large Language Models to Reason in a Continuous Latent Space},
  author={Shibo Hao and Sainbayar Sukhbaatar and DiJia Su and Xian Li and Zhiting Hu and Jason E Weston and Yuandong Tian},
  booktitle={Second Conference on Language Modeling},
  year={2025},
  url={https://openreview.net/forum?id=Itxz7S4Ip3}
}

@inproceedings{shen2025codicompressingchainofthoughtcontinuous,
    title = "{CODI}: Compressing Chain-of-Thought into Continuous Space via Self-Distillation",
    author = "Shen, Zhenyi and Yan, Hanqi and Zhang, Linhai and Hu, Zhanghao and Du, Yali and He, Yulan",
    editor = "Christodoulopoulos, Christos and Chakraborty, Tanmoy and Rose, Carolyn and Peng, Violet",
    booktitle = "Proceedings of the 2025 Conference on Empirical Methods in Natural Language Processing",
    month = nov,
    year = "2025",
    address = "Suzhou, China",
    publisher = "Association for Computational Linguistics",
    url = "https://aclanthology.org/2025.emnlp-main.36/",
    doi = "10.18653/v1/2025.emnlp-main.36",
    pages = "677--693",
    ISBN = "979-8-89176-332-6",
}

@inproceedings{zhang2025softthinkingunlockingreasoning,
      title={{Soft Thinking: Unlocking the Reasoning Potential of {LLM}s in Continuous Concept Space}},
      author={Zhen Zhang and Xuehai He and Weixiang Yan and Ao Shen and Chenyang Zhao and Xin Eric Wang},
      booktitle={The Thirty-ninth Annual Conference on Neural Information Processing Systems},
      year={2026},
      url={https://openreview.net/forum?id=ByQdHPGKgU},
}

@inproceedings{wu2026llms,
title={{{LLM}s are Single-threaded Reasoners: Demystifying the Working Mechanism of Soft Thinking}},
author={Junhong Wu and Jinliang Lu and Zixuan Ren and Gangqiang Hu and Zhi Wu and Dai Dai and Hua Wu},
booktitle={The Fourteenth International Conference on Learning Representations},
year={2026},
url={https://openreview.net/forum?id=ASLuOoP78o}
}

@inproceedings{butt2026soft,
title={{Soft Tokens, Hard Truths}},
author={Natasha Butt and Ariel Kwiatkowski and Ismail Labiad and Julia Kempe and Yann Ollivier},
booktitle={The Fourteenth International Conference on Learning Representations},
year={2026},
url={https://openreview.net/forum?id=9JjKTp8Jmy}
}

@inproceedings{sun2025transformerlayerspainters,
    title = {{Transformer Layers as Painters}},
    author = {Sun, Qi and Pickett, Marc and Nain, Aakash Kumar and Jones, Llion},
    booktitle = {Proceedings of the AAAI Conference on Artificial Intelligence},
    volume = {39},
    pages = {25219--25227},
    year = {2025},
    doi = {10.1609/aaai.v39i24.34708},
    url = {https://ojs.aaai.org/index.php/AAAI/article/view/34708},
}

@article{ameisen2025circuit,
  author={Ameisen, Emmanuel and Lindsey, Jack and Pearce, Adam and Gurnee, Wes and Turner, Nicholas L. and Chen, Brian and Citro, Craig and Abrahams, David and Carter, Shan and Hosmer, Basil and Marcus, Jonathan and Sklar, Michael and Templeton, Adly and Bricken, Trenton and McDougall, Callum and Cunningham, Hoagy and Henighan, Thomas and Jermyn, Adam and Jones, Andy and Persic, Andrew and Qi, Zhenyi and Ben Thompson, T. and Zimmerman, Sam and Rivoire, Kelley and Conerly, Thomas and Olah, Chris and Batson, Joshua},
  title={{Circuit Tracing: Revealing Computational Graphs in Language Models}},
  journal={Transformer Circuits Thread},
  year={2025},
  url={https://transformer-circuits.pub/2025/attribution-graphs/methods.html}
}

@misc{jolicoeurmartineau2025morerecursivereasoningtiny,
      title={{Less is More: Recursive Reasoning with Tiny Networks}}, 
      author={Alexia Jolicoeur-Martineau},
      year={2025},
      eprint={2510.04871},
      archivePrefix={arXiv},
      primaryClass={cs.LG},
      url={https://arxiv.org/abs/2510.04871}, 
}

@inproceedings{bandarkar2025layerswappingzeroshotcrosslingual,
    title = {Layer Swapping for Zero-Shot Cross-Lingual Transfer in Large Language Models},
    author = {Lucas Bandarkar and Benjamin Muller and Pritish Yuvraj and Rui Hou and Nayan Singhal and Hongjiang Lv and Bing Liu},
    booktitle = {The Thirteenth International Conference on Learning Representations},
    year = {2025},
    url = {https://openreview.net/forum?id=vQhn4wrQ6j},
}

@inproceedings{alabi-etal-2024-hidden,
    title = "The Hidden Space of Transformer Language Adapters",
    author = "Alabi, Jesujoba  and
      Mosbach, Marius  and
      Eyal, Matan  and
      Klakow, Dietrich  and
      Geva, Mor",
    editor = "Ku, Lun-Wei  and
      Martins, Andre  and
      Srikumar, Vivek",
    booktitle = "Proceedings of the 62nd Annual Meeting of the Association for Computational Linguistics (Volume 1: Long Papers)",
    month = aug,
    year = "2024",
    address = "Bangkok, Thailand",
    publisher = "Association for Computational Linguistics",
    url = "https://aclanthology.org/2024.acl-long.356/",
    doi = "10.18653/v1/2024.acl-long.356",
    pages = "6588--6607"
}

@inproceedings{feng2025monitoring,
title={{Monitoring Latent World States in Language Models with Propositional Probes}},
author={Jiahai Feng and Stuart Russell and Jacob Steinhardt},
booktitle={The Thirteenth International Conference on Learning Representations},
year={2025},
url={https://openreview.net/forum?id=0yvZm2AjUr}
}

@article{sui2025stopoverthinkingsurveyefficient,
    title = {Stop Overthinking: A Survey on Efficient Reasoning for Large Language Models},
    author = {Yang Sui and Yu-Neng Chuang and Guanchu Wang and Jiamu Zhang and Tianyi Zhang and Jiayi Yuan and Hongyi Liu and Andrew Wen and Shaochen Zhong and Na Zou and Hanjie Chen and Xia Hu},
    journal = {Transactions on Machine Learning Research},
    issn = {2835-8856},
    year = {2025},
    url = {https://openreview.net/forum?id=HvoG8SxggZ},
}

@article{feng2025efficientreasoningmodelssurvey,
    title = {Efficient Reasoning Models: A Survey},
    author = {Sicheng Feng and Gongfan Fang and Xinyin Ma and Xinchao Wang},
    journal = {Transactions on Machine Learning Research},
    issn = {2835-8856},
    year = {2025},
    url = {https://openreview.net/forum?id=sySqlxj8EB},
}

@inproceedings{mondorf2024accuracyevaluatingreasoningbehavior,
    title = {Beyond Accuracy: Evaluating the Reasoning Behavior of Large Language Models -- A Survey},
    author = {Philipp Mondorf and Barbara Plank},
    booktitle = {First Conference on Language Modeling ({COLM})},
    year = {2024},
    url = {https://openreview.net/forum?id=Lmjgl2n11u},
}

@misc{lcm2024,
      title={{Large Concept Models: Language Modeling in a Sentence Representation Space}},
      author={{LCM team} and Loïc Barrault and Paul-Ambroise Duquenne and Maha Elbayad and Artyom Kozhevnikov and Belen Alastruey and Pierre Andrews and Mariano Coria and Guillaume Couairon and Marta R. Costa-jussà and David Dale and Hady Elsahar and Kevin Heffernan and João Maria Janeiro and Tuan Tran and Christophe Ropers and Eduardo Sánchez and Robin San Roman and Alexandre Mourachko and Safiyyah Saleem and Holger Schwenk},
      year={2024},
      eprint={2412.08821},
      archivePrefix={arXiv},
      primaryClass={cs.CL},
      url={https://arxiv.org/abs/2412.08821},
}

@misc{dragunov2025sonarllm,
      title={{SONAR-LLM: Autoregressive Transformer that Thinks in Sentence Embeddings and Speaks in Tokens}},
      author={Nikita Dragunov and Temurbek Rahmatullaev and Elizaveta Goncharova and Andrey Kuznetsov and Anton Razzhigaev},
      year={2025},
      eprint={2508.05305},
      archivePrefix={arXiv},
      primaryClass={cs.CL},
      url={https://arxiv.org/abs/2508.05305},
}

@article{cobbe2021gsm8k,
  title={{Training verifiers to solve math word problems}},
  author={Cobbe, Karl and Kosaraju, Vin and Bavarian, Mojtaba and Schulman, John and Moss, Joshua and Brundage, Miles and others},
  journal={arXiv preprint arXiv:2110.14168},
  year={2021}
}

@misc{koishekenov2025encodethinkdecodescaling,
      title={{Encode, Think, Decode: Scaling test-time reasoning with recursive latent thoughts}}, 
      author={Yeskendir Koishekenov and Aldo Lipani and Nicola Cancedda},
      year={2025},
      eprint={2510.07358},
      archivePrefix={arXiv},
      primaryClass={cs.LG},
      url={https://arxiv.org/abs/2510.07358}, 
}

@misc{duquenne2023sonarsentencelevelmultimodallanguageagnostic,
      title={{SONAR: Sentence-Level Multimodal and Language-Agnostic Representations}}, 
      author={Paul-Ambroise Duquenne and Holger Schwenk and Benoît Sagot},
      year={2023},
      eprint={2308.11466},
      archivePrefix={arXiv},
      primaryClass={cs.CL},
      url={https://arxiv.org/abs/2308.11466}, 
}

@misc{grattafiori2024llama3herdmodels,
      title={{The Llama 3 Herd of Models}},
      author={Grattafiori, Aaron and Dubey, Abhimanyu and Jauhri, Abhinav and Pandey, Abhinav and Kadian, Abhishek and Al-Dahle, Ahmad and Letman, Aiesha and Mathur, Akhil and Schelten, Alan and Vaughan, Alex and others},
      year={2024},
      eprint={2407.21783},
      archivePrefix={arXiv},
      primaryClass={cs.AI},
      url={https://arxiv.org/abs/2407.21783}, 
}

@inproceedings{kazemi2025bigbenchextrahard,
    title = "{BIG}-Bench Extra Hard",
    author = "Kazemi, Mehran and Fatemi, Bahare and Bansal, Hritik and Palowitch, John and Anastasiou, Chrysovalantis and Mehta, Sanket Vaibhav and Jain, Lalit K and Aglietti, Virginia and Jindal, Disha and Chen, Peter and Dikkala, Nishanth and Tyen, Gladys and Liu, Xin and Shalit, Uri and Chiappa, Silvia and Olszewska, Kate and Tay, Yi and Tran, Vinh Q. and Le, Quoc V and Firat, Orhan",
    editor = "Che, Wanxiang and Nabende, Joyce and Shutova, Ekaterina and Pilehvar, Mohammad Taher",
    booktitle = "Proceedings of the 63rd Annual Meeting of the Association for Computational Linguistics (Volume 1: Long Papers)",
    month = jul,
    year = "2025",
    address = "Vienna, Austria",
    publisher = "Association for Computational Linguistics",
    url = "https://aclanthology.org/2025.acl-long.1285/",
    doi = "10.18653/v1/2025.acl-long.1285",
    pages = "26473--26501",
    ISBN = "979-8-89176-251-0",
}

@InProceedings{zou2025latentcollaborationmultiagentsystems,
  title = 	 {Latent Collaboration in Multi-Agent Systems},
  author =       {Zou, Jiaru and Qiu, Ruizhong and Li, Gaotang and Yang, Xiyuan and Tieu, Katherine and Lu, Pan and Shen, Ke and Tong, Hanghang and Choi, Yejin and He, Jingrui and Zou, James and Wang, Mengdi and Yang, Ling},
  booktitle = 	 {Proceedings of the 43rd International Conference on Machine Learning},
  pages = 	 {168813--168847},
  year = 	 {2026},
  editor = 	 {Zhang, Tong and Dudik, Miroslav and Jaggi, Martin and Agarwal, Alekh and Li, Sharon and Schuurmans, Dale and Zhu, Jerry and Berkenkamp, Felix and Dong, Hanze and Bietti, Alberto},
  volume = 	 {306},
  series = 	 {Proceedings of Machine Learning Research},
  month = 	 {06--11 Jul},
  publisher =    {PMLR},
  url = 	 {https://proceedings.mlr.press/v306/zou26k.html}
}

@article{he2025skywork,
  title={{Skywork Open Reasoner 1 Technical Report}},
  author={He, Jujie and Liu, Jiacai and Liu, Chris Yuhao and Yan, Rui and Wang, Chaojie and Cheng, Peng and Zhang, Xiaoyu and Zhang, Fuxiang and Xu, Jiacheng and Shen, Wei and Li, Siyuan and Zeng, Liang and Wei, Tianwen and Cheng, Cheng and An, Bo and Liu, Yang and Zhou, Yahui},
  journal={arXiv preprint arXiv:2505.22312},
  year={2025}
}

@misc{skywork-or1-2025,
  title={{Skywork Open Reasoner Series}},
  author = {He, Jujie and Liu, Jiacai and Liu, Chris Yuhao and Yan, Rui and Wang, Chaojie and Cheng, Peng and Zhang, Xiaoyu and Zhang, Fuxiang and Xu, Jiacheng and Shen, Wei and Li, Siyuan and Zeng, Liang and Wei, Tianwen and Cheng, Cheng and Liu, Yang and Zhou, Yahui},
  howpublished={\url{https://capricious-hydrogen-41c.notion.site/Skywork-Open-Reaonser-Series-1d0bc9ae823a80459b46c149e4f51680}},
  note={Notion Blog},
  year={2025}
}

@article{deepseekai2025deepseekr1incentivizingreasoningcapability,
    title = {{DeepSeek-R1} incentivizes reasoning in {LLM}s through reinforcement learning},
    author = {{DeepSeek-AI}},
    journal = {Nature},
    volume = {645},
    number = {8081},
    pages = {633--638},
    year = {2025},
    publisher = {Nature Publishing Group},
    doi = {10.1038/s41586-025-09422-z},
    url = {https://www.nature.com/articles/s41586-025-09422-z},
}

@misc{openai2025gptoss120bgptoss20bmodel,
      title={{gpt-oss-120b \& gpt-oss-20b Model Card}},
      author={OpenAI},
      year={2025},
      eprint={2508.10925},
      archivePrefix={arXiv},
      primaryClass={cs.CL},
      url={https://arxiv.org/abs/2508.10925},
}

@inproceedings{cai2026t2mlr, title = {{T2MLR: Transformer with Temporal Middle-Layer Recurrence}}, author = {Ziyang Cai and Xingyu Zhu and Yihe Dong and Yinghui He and Sanjeev Arora}, booktitle = {LIT Workshop @ ICLR 2026}, year = {2026}, url = {https://openreview.net/forum?id=fQbk1EQWBO}}

@inproceedings{sahoo2026when, title = {{When Shallow Wins: Silent Failures and the Depth-Accuracy Paradox in Latent Reasoning}}, author = {Subramanyam Sahoo and Aman Chadha and Vinija Jain and Divya Chaudhary}, booktitle = {LIT Workshop @ ICLR 2026}, year = {2026}, url = {https://arxiv.org/abs/2603.03475}}

@inproceedings{rizvi-martel2026the, title = {{The Illusion of Superposition in Latent CoT via Soft Thinking}}, author = {Michael Rizvi-Martel and Marius Mosbach}, booktitle = {LIT Workshop @ ICLR 2026}, year = {2026}, url = {https://openreview.net/forum?id=FvPx9Nzvnw}}

@inproceedings{lugoloobi2026llms, title = {{LLMs Encode Their Failures: Predicting Success from Pre-Generation Activations}}, author = {William Lugoloobi and Thomas Foster and William Bankes and Chris Russell}, booktitle = {LIT Workshop @ ICLR 2026}, year = {2026}, url = {https://arxiv.org/abs/2602.09924}}

@inproceedings{wang2026tiny, title = {{Tiny Recursive Reasoning with Mamba-2 Attention Hybrid}}, author = {Wenlong Wang and Fergal Reid}, booktitle = {LIT Workshop @ ICLR 2026}, year = {2026}, url = {https://arxiv.org/abs/2602.12078}}

@inproceedings{dilgren2026are, title = {{Are Latent Reasoning Models Easily Interpretable?}}, author = {Connor Dilgren and Sarah Wiegreffe}, booktitle = {LIT Workshop @ ICLR 2026}, year = {2026}, url = {https://arxiv.org/abs/2604.04902}}

@inproceedings{cui2026how, title = {{How Do Latent Reasoning Methods Perform Under Weak and Strong Supervision?}}, author = {Yingqian Cui and Zhenwei Dai and Bing He and Zhan Shi and Hui Liu and Rui Sun and Zhiji Liu and Yue Xing and Jiliang Tang and Benoit Dumoulin}, booktitle = {LIT Workshop @ ICLR 2026}, year = {2026}, url = {https://arxiv.org/abs/2602.22441}}

@misc{cencerrado2026answerneededpredictingllm,
      title={No Answer Needed: Predicting {LLM} Answer Accuracy from Question-Only Linear Probes},
      author={Iván Vicente Moreno Cencerrado and Arnau Padrés Masdemont and Anton Gonzalvez Hawthorne and David Demitri Africa and Lorenzo Pacchiardi},
      year={2026},
      eprint={2509.10625},
      archivePrefix={arXiv},
      primaryClass={cs.CL},
      url={https://arxiv.org/abs/2509.10625},
}

@inproceedings{kuhn2023semantic,
  title     = {Semantic Uncertainty: Linguistic Invariances for Uncertainty Estimation in Natural Language Generation},
  author    = {Lorenz Kuhn and Yarin Gal and Sebastian Farquhar},
  booktitle = {The Eleventh International Conference on Learning Representations},
  year      = {2023},
  url       = {https://openreview.net/forum?id=VD-AYtP0dve},
}

@inproceedings{barber2003variational,
  author    = {Barber, David and Agakov, Felix},
  title     = {The {IM} algorithm: a variational approach to Information Maximization},
  year      = {2003},
  publisher = {{MIT} Press},
  address   = {Cambridge, MA, USA},
  booktitle = {Proceedings of the 17th International Conference on Neural Information Processing Systems},
  pages     = {201--208},
  numpages  = {8},
  location  = {Whistler, British Columbia, Canada},
  series    = {{NIPS}'03}
}

@misc{babakhin2025llamaembednemotron8buniversaltextembedding,
      title={Llama-Embed-Nemotron-8B: A Universal Text Embedding Model for Multilingual and Cross-Lingual Tasks},
      author={Yauhen Babakhin and Radek Osmulski and Ronay Ak and Gabriel Moreira and Mengyao Xu and Benedikt Schifferer and Bo Liu and Even Oldridge},
      year={2025},
      eprint={2511.07025},
      archivePrefix={arXiv},
      primaryClass={cs.CL},
      url={https://arxiv.org/abs/2511.07025},
}

@misc{chen2025reasoninglanguagecomprehensivesurvey,
      title={Reasoning Beyond Language: A Comprehensive Survey on Latent Chain-of-Thought Reasoning},
      author={Xinghao Chen and Anhao Zhao and Heming Xia and Xuan Lu and Hanlin Wang and Yanjun Chen and Wei Zhang and Jian Wang and Wenjie Li and Xiaoyu Shen},
      year={2025},
      eprint={2505.16782},
      archivePrefix={arXiv},
      primaryClass={cs.CL},
      url={https://arxiv.org/abs/2505.16782},
}

@misc{zou2026the,
      title={The Theoretical Benefits and Limitations of Latent Chain-of-Thought Reasoning},
      author={Jiaxuan Zou and Yaozhong Xiong and Yong Liu},
      year={2026},
      url={https://openreview.net/forum?id=q7Nhu2Fw11},
}

@misc{yu2026latentspace,
      title={The Latent Space: Foundation, Evolution, Mechanism, Ability, and Outlook},
      author={Xinlei Yu and Zhangquan Chen and Yongbo He and Tianyu Fu and Guanting Dong and Cheng Yang and Chengming Xu and Yue Ma and Xiaobin Hu and Zhe Cao and Jie Xu and Guibin Zhang and Jiale Tao and Jiayi Zhang and Siyuan Ma and Kaituo Feng and Haojie Huang and Youxing Li and Ronghao Chen and Huacan Wang and Chenglin Wu and Zikun Su and Xiaogang Xu and Kelu Yao and Kun Wang and Chen Gao and Yue Liao and Ruqi Huang and Tao Jin and Zhucun Xue and Cheng Tan and Jiangning Zhang and Wenqi Ren and Yanwei Fu and Yong Liu and Yu Wang and Xiangyu Yue and Yu-Gang Jiang and Shuicheng Yan},
      year={2026},
      eprint={2604.02029},
      archivePrefix={arXiv},
      primaryClass={cs.AI},
      url={https://arxiv.org/abs/2604.02029},
}

@article{ccot,
  title={{Compressed Chain of Thought: Efficient Reasoning Through Dense Representations}},
  author={Cheng, Jeffrey and Van Durme, Benjamin},
  journal={arXiv preprint arXiv:2412.13171},
  year={2024}
}

@article{deng2025llm,
  title={{LLM Latent Reasoning as Chain of Superposition}},
  author={Deng, Jingcheng and Pang, Liang and Wei, Zihao and Xu, Shicheng and Duan, Zenghao and Xu, Kun and Song, Yang and Shen, Huawei and Cheng, Xueqi},
  journal={arXiv preprint arXiv:2510.15522},
  year={2025}
}

@inproceedings{tan2025colar,
 author = {Tan, Wenhui and Li, Jiaze and Ju, Jianzhong and Luo, Zhenbo and Song, Ruihua and Luan, Jian},
 booktitle = {Advances in Neural Information Processing Systems},
 doi = {10.52202/085713-0164},
 editor = {D. Belgrave and C. Zhang and H. Lin and R. Pascanu and P. Koniusz and M. Ghassemi and N. Chen},
 pages = {4646--4668},
 publisher = {Curran Associates, Inc.},
 title = {Think Silently, Think Fast: Dynamic Latent Compression of {LLM} Reasoning Chains},
 url = {https://proceedings.neurips.cc/paper_files/paper/2025/file/0706261aedab63814a2b73c32564b4c4-Paper-Conference.pdf},
 volume = {38, Main Conference},
 year = {2025}
}

@inproceedings{fu2025cache,
 author = {Fu, Tianyu and Min, Zihan and Zhang, Hanling and Yan, Jichao and Dai, Guohao and Ouyang, Wanli and Wang, Yu},
 booktitle = {International Conference on Learning Representations},
 editor = {C. Vondrick and B. Hariharan and C. Raffel and L. Pinto and D. Yang and A. Faust},
 pages = {43130--43158},
 title = {Cache-to-Cache: Direct Semantic Communication Between Large Language Models},
 url = {https://proceedings.iclr.cc/paper_files/paper/2026/file/474ada926b331d78f06d95e8913111cc-Paper-Conference.pdf},
 volume = {2026},
 year = {2026}
}

@inproceedings{wei2025simcot,
 author = {Wei, Xilin and Liu, Xiaoran and Zang, Yuhang and Dong, Xiaoyi and Cao, Yuhang and Wang, Jiaqi and Qiu, Xipeng and Lin, Dahua},
 booktitle = {International Conference on Learning Representations},
 editor = {C. Vondrick and B. Hariharan and C. Raffel and L. Pinto and D. Yang and A. Faust},
 pages = {56721--56742},
 title = {{SIM-CoT}: Supervised Implicit Chain-of-Thought},
 url = {https://proceedings.iclr.cc/paper_files/paper/2026/file/5d087955ee13fe9a7402eedec879b9c3-Paper-Conference.pdf},
 volume = {2026},
 year = {2026}
}

@misc{hagendorff2025beyond,
  title={{Beyond Chains of Thought: Benchmarking Latent-Space Reasoning Abilities in Large Language Models}},
  author={Hagendorff, Thilo and Fabi, Sarah},
  year={2025},
  eprint={2504.10615},
  archivePrefix={arXiv},
  url={https://arxiv.org/abs/2504.10615}
}

@inproceedings{xu2026formal,
  title={{A Formal Comparison Between Chain of Thought and Latent Thought}},
  author={Xu, Kevin and Sato, Issei},
  booktitle={International Conference on Machine Learning (ICML)},
  year={2026},
  eprint={2509.25239},
  archivePrefix={arXiv},
  url={https://arxiv.org/abs/2509.25239}
}

@inproceedings{zhu2025reasoning,
  title={{Reasoning by Superposition: A Theoretical Perspective on Chain of Continuous Thought}},
  author={Zhu, Hanlin and Hao, Shibo and Hu, Zhiting and Jiao, Jiantao and Russell, Stuart and Tian, Yuandong},
  booktitle={Advances in Neural Information Processing Systems (NeurIPS)},
  year={2025},
  eprint={2505.12514},
  archivePrefix={arXiv},
  url={https://arxiv.org/abs/2505.12514}
}

@techreport{radford2019language,
  title={{Language Models are Unsupervised Multitask Learners}},
  author={Radford, Alec and Wu, Jeffrey and Child, Rewon and Luan, David and Amodei, Dario and Sutskever, Ilya},
  institution={OpenAI},
  year={2019}
}

@misc{zhang2025latent,
      title={{Do Latent Tokens Think? A Causal and Adversarial Analysis of Chain-of-Continuous-Thought}},
      author={Yuyi Zhang and Boyu Tang and Tianjie Ju and Sufeng Duan and Gongshen Liu},
      year={2025},
      eprint={2512.21711},
      archivePrefix={arXiv},
      primaryClass={cs.CL},
      url={https://arxiv.org/abs/2512.21711},
}

@misc{liang2026latentcot,
      title={{Do Latent-CoT Models Think Step-by-Step? A Mechanistic Study on Sequential Reasoning Tasks}},
      author={Jia Liang and Liangming Pan},
      year={2026},
      eprint={2602.00449},
      archivePrefix={arXiv},
      primaryClass={cs.AI},
      url={https://arxiv.org/abs/2602.00449},
}

@inproceedings{du2026latent,
title={{Latent Thinking Optimization: Your Latent Reasoning Language Model Secretly Encodes Reward Signals in Its Latent Thoughts}},
author={Hanwen Du and Yuxin Dong and Xia Ning},
booktitle={The Fourteenth International Conference on Learning Representations},
year={2026},
url={https://openreview.net/forum?id=2jkAk3EP0v}
}

\newpage
\appendix
\crefalias{section}{appendix}
\crefalias{subsection}{subappendix}
\crefalias{subsubsection}{subsubappendix}
\etocdepthtag.toc{appendix}

\etocsettagdepth{main}{none}
\etocsettagdepth{appendix}{subsection}
\etocsettocstyle{\section*{Appendix Contents}}{\bigskip}
\tableofcontents
\newpage
\section{Proofs}
\label{app:proofs}

Let $X$ denote the random input, with distribution $P(X = x)$ over $\mathcal{X}$, and let $Y$ denote the output that $f_\theta$ generates for it.
As in \cref{def:valid-space}, $\mu_c$ denotes the distribution that a context $c$ induces~\eqref{eq:induced}, hence $\mu_x = P(Y \mid x)$ is the output distribution of an input $x$.
Let $\mathbb{H}$ denote the Shannon entropy, and let $\mathbf{1}[\,\cdot\,]$ take the value one when its condition holds and zero otherwise.
Let the thought representation space $\mathcal{T}$ lie in $\mathbb{R}^m$.
The projection $\phi$ of \cref{def:separability} belongs to the class $\mathcal{H}$ of projections that factor through an intermediate space $\mathbb{R}^k$ of some dimension $k$,
\begin{equation}
\label{eq:hypothesis-class}
\begin{aligned}
\mathcal{H} = \big\{\, \phi = \phi_2 \circ \phi_1 \;:\; &\phi_1 : \mathcal{T} \to \mathbb{R}^k \text{ and } \phi_2 : \mathbb{R}^k \to \mathbb{R}^d \text{ linear}, &\lVert \phi_1 \rVert \leq 1, \; \lVert \phi_2 \rVert \leq 1 \,\big\},
\end{aligned}
\end{equation}
where the norm of a linear map is the largest factor by which it stretches a vector,
\begin{equation}
\label{eq:operator-norm}
\lVert \phi \rVert = \sup_{a \neq 0} \frac{\lVert \phi(a) \rVert_2}{\lVert a \rVert_2}, \qquad\text{hence}\qquad \lVert \phi(a) \rVert_2 \leq \lVert \phi \rVert \, \lVert a \rVert_2 .
\end{equation}
The bound on $\mathcal{H}$ makes separability at a margin $\delta > 0$ a property of $\mathbf{T}$.
An unrestricted class of projections stretches the difference of any two distinct representations without bound, which leads to a trivial solution that enforces no property on $\mathbf{T}$.

\subsection{Consistency}
\label{app:proofs-consistency}

This section proves Proposition~\ref{prop:consistency} by constructing one model and one thought representation that attain the four properties.
The thought representation of a context is the distribution that the context induces over the semantic output space.
The model conditions each continuation only on that thought representation.

\begin{lemma}[Decomposition of Mutual Information]
\label{lem:decomposition}
Let $\gamma$ be a map defined on $\mathcal{X}$ and let $Z = \gamma(X)$.
Then
\begin{equation}
\label{eq:decomposition}
I(X; Y) = I(Z; Y) + I(X; Y \mid Z).
\end{equation}
\end{lemma}

\begin{proof}[Proof of Lemma~\ref{lem:decomposition}]
Expanding $I(X, Z; Y)$ by the chain rule in the two orders,
\begin{align}
I(X, Z; Y) &= I(Z; Y) + I(X; Y \mid Z), \label{eq:chain-one} \\
I(X, Z; Y) &= I(X; Y) + I(Z; Y \mid X). \label{eq:chain-two}
\end{align}
Since $Z = \gamma(X)$ is a function of $X$, the last term vanishes,
\begin{equation}
\label{eq:chain-zero}
I(Z; Y \mid X) = 0,
\end{equation}
hence \eqref{eq:chain-two} reduces to $I(X, Z; Y) = I(X; Y)$.
Substituting this into \eqref{eq:chain-one} gives \eqref{eq:decomposition}.
\end{proof}

\restPropConsistency*

\begin{proof}[Proof of Proposition~\ref{prop:consistency}]
\textbf{Construction.}
Let $\mathcal{X}$ and $\mathcal{Y}$ be finite, let $M \geq 2$ denote the number of semantic outputs, and let the semantic output space be the standard basis of $\mathbb{R}^M$,
\begin{equation}
\label{eq:cons-space}
\mathcal{S} = \{s_1, \dots, s_M\} \subset \mathbb{R}^M, \qquad s_i = e_i, \qquad d_{\mathcal{S}}(s, s') = \lVert s - s' \rVert_2 ,
\end{equation}
where $e_i$ denotes the $i$-th standard basis vector, which meets Assumption~\ref{as:semantic-space}.
Let $U$ denote the prefix of an output.
Define the generator of Definition~\ref{def:generator} as the conditional distribution of the semantic class, and let $\mathbf{T}_c = g(c)$ denote the representation of a context $c$,
\begin{equation}
\label{eq:cons-generator}
g(x) = P\big(\Phi(Y) \mid X = x\big), \qquad g([x; u]) = P\big(\Phi(Y) \mid X = x, U = u\big).
\end{equation}
Let every output be a single token.
Let class $i$ hold two tokens, $y^{(i)}$ and $y^{(i)}_{\sqcup}$, the second being the first followed by a space, with $\Phi(y^{(i)}) = \Phi(y^{(i)}_{\sqcup}) = s_i$.
Let every input give positive probability to at least two classes, and let one input give two classes different probabilities.
The only contexts are then the inputs, whose representations lie off the vertices $e_1, \dots, e_M$, and the completed contexts.
Construct $f_\theta$ such that generation is conditioned on the representation of its context,
\begin{align}
P(v \mid c) &= P(v \mid \mathbf{T}_c), \label{eq:cons-model} \\
P(v \mid [x; \mathbf{T}]) &= P(v \mid \mathbf{T}), \label{eq:cons-substitution}
\end{align}
where the model generates one token of each class with half of its probability off the vertices, and ends generation at a vertex,
\begin{equation}
\label{eq:cons-explicit}
P(v \mid \tau) = \sum_{i=1}^{M} \frac{\tau[i]}{2} \Big( \mathbf{1}\big[v = y^{(i)}\big] + \mathbf{1}\big[v = y^{(i)}_{\sqcup}\big] \Big), \qquad P(v \mid e_i) = \mathbf{1}[v = \emptyset],
\end{equation}
with $\tau \in \Delta(\mathcal{S}) \setminus \{e_1, \dots, e_M\}$, $\tau[i]$ the $i$-th entry of $\tau$, and $\emptyset$ the empty sequence.
The construction is self-consistent.
At an input $x$, the model conditions on $\mathbf{T}_x = g(x)$, which lies off the vertices, and \eqref{eq:cons-explicit} gives class $i$ the probability that \eqref{eq:cons-generator} assigns to it,
\begin{equation}
\label{eq:cons-self}
P\big(\Phi(Y) = s_i \mid X = x\big) = \frac{g(x)[i]}{2} + \frac{g(x)[i]}{2} = g(x)[i].
\end{equation}
At a completed context $[x; y]$ with $\Phi(y) = s_i$, the class is fixed, hence $g([x; y]) = e_i$, and the model ends generation, hence the output is complete.
Since $\mathbf{T}_c = g(c)$ is a function of $c$, conditioning on the pair $(c, \mathbf{T}_c)$ equals conditioning on $c$,
\begin{equation}
\label{eq:cons-markov-one}
P(v \mid c, \mathbf{T}_c) = P(v \mid c),
\end{equation}
and substituting \eqref{eq:cons-model} into \eqref{eq:cons-markov-one},
\begin{equation}
\label{eq:cons-markov}
P(v \mid c, \mathbf{T}_c) = P(v \mid \mathbf{T}_c),
\end{equation}
which is the Markov chain $X \to \mathbf{T} \to Y$.

\textbf{Causality.}
By \eqref{eq:cons-model} with $c = [x; u]$, and by \eqref{eq:cons-substitution} with $\mathbf{T} = \mathbf{T}_{[x;u]}$,
\begin{align}
P(v \mid [x; u]) &= P(v \mid \mathbf{T}_{[x;u]}), && \eqref{eq:cons-model} \label{eq:cons-causality-one} \\
P(v \mid [x; \mathbf{T}_{[x;u]}]) &= P(v \mid \mathbf{T}_{[x;u]}), && \eqref{eq:cons-substitution} \label{eq:cons-causality-two}
\end{align}
hence the two left sides are equal, which is \eqref{eq:causality}.

\textbf{Minimality.}
Let $\mathbf{T}$ denote $g(X)$.
By Lemma~\ref{lem:decomposition} with $Z = \mathbf{T}$, together with \eqref{eq:cons-markov},
\begin{equation}
\label{eq:cons-sufficiency}
I(X; Y) = I(\mathbf{T}; Y) + \underbrace{I(X; Y \mid \mathbf{T})}_{= 0}, \qquad\text{hence}\qquad I(\mathbf{T}; Y) = I(X; Y).
\end{equation}
Since $I(\mathbf{T}'; Y) \leq I(X; Y)$ holds for every $\mathbf{T}' = g'(X)$ by the data processing inequality, the constraint in \eqref{eq:minimality-constrained} is met at its largest attainable value $I_0 = I(X; Y)$.
Let $g'$ be a generator whose representation $\mathbf{T}' = g'(X)$ meets that constraint.
By Lemma~\ref{lem:decomposition} with $Z = \mathbf{T}'$, together with $I(\mathbf{T}'; Y) = I(X; Y)$,
\begin{equation}
\label{eq:cons-cond-indep}
I(X; Y \mid \mathbf{T}') = 0,
\end{equation}
which states that $X$ and $Y$ are conditionally independent given $\mathbf{T}'$,
\begin{equation}
\label{eq:cons-cond-dist}
P(Y \mid X = x) = P\big(Y \mid \mathbf{T}' = g'(x)\big).
\end{equation}
Summing both sides of \eqref{eq:cons-cond-dist} over the outputs with $\Phi(y) = s$,
\begin{equation}
\label{eq:cons-pushforward}
P\big(\Phi(Y) = s \mid X = x\big) = P\big(\Phi(Y) = s \mid \mathbf{T}' = g'(x)\big), \qquad s \in \mathcal{S},
\end{equation}
hence, by \eqref{eq:cons-generator}, the representation $\mathbf{T}_x$ is a function of $\mathbf{T}'$,
\begin{equation}
\label{eq:cons-function}
\mathbf{T}_x = q'\big(g'(x)\big), \qquad q'(\tau) = P\big(\Phi(Y) \mid \mathbf{T}' = \tau\big).
\end{equation}
By \eqref{eq:cons-function} the chain $X \to \mathbf{T}' \to \mathbf{T}$ is Markov, and the data processing inequality gives
\begin{equation}
\label{eq:cons-minimum}
I(X; \mathbf{T}) \leq I(X; \mathbf{T}'),
\end{equation}
which is the minimum in \eqref{eq:minimality-constrained}.

\textbf{Separability.}
Let $k = M$, and let $\phi_1$ and $\phi_2$ be the identity on $\mathbb{R}^M$, which is linear and of unit operator norm, hence $\phi \in \mathcal{H}$ and $\phi(\mathbf{T}) = \mathbf{T}$.
Let the decoding rule of Definition~\ref{def:generator} generate, at a context $c$, an output $y_c$ of the class of largest probability under $\mathbf{T}_c$, with ties broken toward the smallest index,
\begin{equation}
\label{eq:cons-decoding}
\Phi(y_c) = s_{j(\mathbf{T}_c)}, \qquad j(\tau) = \min \Big\{ i : \tau[i] = \max_{k} \tau[k] \Big\}.
\end{equation}
The semantic output of a context is a function of its representation, hence two contexts whose outputs are not semantically equivalent have distinct representations,
\begin{equation}
\label{eq:cons-distinct}
\Phi(y_{c_1}) \neq \Phi(y_{c_2}) \quad \Longrightarrow \quad \mathbf{T}_{c_1} \neq \mathbf{T}_{c_2}.
\end{equation}
Since $\mathcal{X}$ and $\mathcal{Y}$ are finite, the contexts are finite in number.
The smallest distance between two such representations is therefore positive,
\begin{equation}
\label{eq:cons-min-distance}
\delta_0 = \min \Big\{ \lVert \mathbf{T}_{c_1} - \mathbf{T}_{c_2} \rVert_2 : \Phi(y_{c_1}) \neq \Phi(y_{c_2}) \Big\} > 0 .
\end{equation}
For every margin $\delta < \delta_0$,
\begin{equation}
\label{eq:cons-separability}
d_{\mathcal{S}}\big(\phi(\mathbf{T}_{c_1}), \phi(\mathbf{T}_{c_2})\big) = \lVert \mathbf{T}_{c_1} - \mathbf{T}_{c_2} \rVert_2 \geq \delta_0 > \delta,
\end{equation}
which is \eqref{eq:separability}.

\textbf{Stability.}
For an output $y$ sampled for the input $x$, let $i(y)$ denote the index with $s_{i(y)} = \Phi(y)$.
Conditioning on $Y = y$ fixes the value of $\Phi(Y)$, hence \eqref{eq:cons-generator} gives
\begin{equation}
\label{eq:cons-point-mass}
\mathbf{T}_y = g([x; y]) = P\big(\Phi(Y) \mid X = x, Y = y\big) = e_{i(y)} .
\end{equation}
For outputs with $y_i \sim_{\mathrm{sem}} y_j$, Definition~\ref{def:semantic-equivalence} gives $\Phi(y_i) = \Phi(y_j)$ and therefore $i(y_i) = i(y_j)$, hence
\begin{equation}
\label{eq:cons-lexical}
\mathbf{T}_{y_i} = e_{i(y_i)} = e_{i(y_j)} = \mathbf{T}_{y_j},
\end{equation}
which is \eqref{eq:lexical-invariance}.
Every value of $g$ is a distribution over $\mathcal{S}$ by \eqref{eq:cons-generator}, hence the identity on $\mathcal{T}$ is a map from $\mathcal{T}$ to the distributions over $\mathcal{S}$.
Taking $q$ to be that map,
\begin{equation}
\label{eq:cons-mode}
q(\mathbf{T}_x) = \mathbf{T}_x = g(x) = P\big(\Phi(Y) \mid X = x\big),
\end{equation}
which is \eqref{eq:mode-collapse}.

The four requirements hold for the model and the representation of \eqref{eq:cons-generator}, which proves the proposition.
\end{proof}

\subsection{Independence}
\label{app:proofs-independence}

Each proof exhibits one pair, constructed from the model of Proposition~\ref{prop:consistency} by a single edit to its thought representation. Each of the four attains the remaining three properties.

Throughout this section the semantic output space, the generator and the model are those of Proposition~\ref{prop:consistency}, where $M \geq 2$,
\begin{align}
\mathcal{S} &= \{s_1, \dots, s_M\} \subset \mathbb{R}^M, \qquad s_i = e_i, \qquad d_{\mathcal{S}}(s, s') = \lVert s - s' \rVert_2, \label{eq:ind-space} \\
g(c) &= P\big(\Phi(Y) \mid c\big) \in \Delta(\mathcal{S}), \label{eq:ind-generator} \\
\Delta(\mathcal{S}) &= \Big\{ \tau \in \mathbb{R}^M : \tau[i] \geq 0 \text{ for every } i, \; \textstyle\sum_{i=1}^{M} \tau[i] = 1 \Big\}, \label{eq:ind-simplex} \\
P(v \mid c) &= P(v \mid \mathbf{T}_c), \label{eq:ind-model} \\
P(v \mid [x; \mathbf{T}]) &= P(v \mid \mathbf{T}). \label{eq:ind-substitution}
\end{align}
The display \eqref{eq:ind-simplex} states \eqref{eq:simplex} in the coordinates of this construction.
Each construction below replaces $g$ by a map $g'$, states the conditional $P'$ of its model through \eqref{eq:ind-model} and \eqref{eq:ind-substitution}, and leaves the rest of the construction fixed.
Let $\mathbf{T} = g(X)$ and $\mathbf{T}' = g'(X)$ denote the two representations of the random input, and let $\mathbf{T}'_c = g'(c)$ denote the edited representation of a context $c$.

\begin{lemma}[Reparameterization]
\label{lem:reparameterization}
Let $B \in \mathbb{R}^{M \times M}$ be invertible, let $g'(c) = B g(c)$, and let the model generate as in \eqref{eq:ind-model}, under which $(X, Y)$ keeps its distribution.
Then $g'$ attains minimality and stability whenever $g$ attains them.
\end{lemma}

\begin{proof}[Proof of Lemma~\ref{lem:reparameterization}]
Since an invertible map leaves mutual information unchanged,
\begin{align}
I(X; \mathbf{T}') &= I(X; B\mathbf{T}) = I(X; \mathbf{T}), \label{eq:rep-objective} \\
I(\mathbf{T}'; Y) &= I(B\mathbf{T}; Y) = I(\mathbf{T}; Y), \label{eq:rep-constraint}
\end{align}
hence the objective and the constraint of \eqref{eq:minimality-constrained} take the same values at $\mathbf{T}'$ as at $\mathbf{T}$, and $\mathbf{T}'$ attains the minimum whenever $\mathbf{T}$ attains it.
For two outputs with $y_a \sim_{\mathrm{sem}} y_b$,
\begin{equation}
\label{eq:rep-lexical}
\mathbf{T}_{y_a} = \mathbf{T}_{y_b} \quad \Longrightarrow \quad B \mathbf{T}_{y_a} = B \mathbf{T}_{y_b} \quad \Longrightarrow \quad \mathbf{T}'_{y_a} = \mathbf{T}'_{y_b},
\end{equation}
which is \eqref{eq:lexical-invariance}.
Let $q$ denote the map of \eqref{eq:mode-collapse} for $g$, and let $q'$ denote $q \circ B^{-1}$, which maps the thought representation space to the distributions over $\mathcal{S}$.
Then
\begin{equation}
\label{eq:rep-mode}
q'(\mathbf{T}'_x) = q\big(B^{-1} B \mathbf{T}_x\big) = q(\mathbf{T}_x) = P\big(\Phi(Y) \mid X = x\big),
\end{equation}
which is \eqref{eq:mode-collapse}.
\end{proof}

\begin{lemma}[Augmentation]
\label{lem:augmentation}
Let $\omega$ map contexts to $\mathbb{R}^r$, let
\begin{equation}
\label{eq:aug-generator}
g'(c) = [\, g(c) ; \omega(c) \,] \in \mathbb{R}^{M+r},
\end{equation}
and let the model of the pair condition on the first block, for every value $\tau$ of $g$ and every $\zeta \in \mathbb{R}^r$,
\begin{align}
P'(v \mid c) &= P(v \mid \mathbf{T}_c), \label{eq:aug-model} \\
P'\big(v \mid [x; [\tau; \zeta]]\big) &= P(v \mid \tau). \label{eq:aug-substitution}
\end{align}
Then $g'$ attains causality and separability whenever $g$ attains them.
\end{lemma}

\begin{proof}[Proof of Lemma~\ref{lem:augmentation}]
At a context $[x; u]$, let $\tau$ denote $g([x; u])$ and let $\zeta$ denote $\omega([x; u])$.
Then
\begin{align}
P'(v \mid [x; u]) &= P(v \mid \mathbf{T}_{[x;u]}) = P(v \mid \tau), && \eqref{eq:aug-model} \label{eq:aug-causality-one} \\
P'\big(v \mid [x; \mathbf{T}'_{[x;u]}]\big) &= P'\big(v \mid [x; [\tau; \zeta]]\big) = P(v \mid \tau), && \eqref{eq:aug-substitution} \label{eq:aug-causality-two}
\end{align}
hence the two left sides are equal, which is \eqref{eq:causality}.
Let $\rho : \mathbb{R}^{M+r} \to \mathbb{R}^{M}$ denote the projection $\rho([\tau; \zeta]) = \tau$, whose operator norm is at most one,
\begin{equation}
\label{eq:aug-projection}
\lVert \rho([\tau; \zeta]) \rVert_2 = \lVert \tau \rVert_2 \leq \lVert [\tau; \zeta] \rVert_2, \qquad\text{hence}\qquad \lVert \rho \rVert \leq 1 .
\end{equation}
Let $\phi = \phi_2 \circ \phi_1$ attain the margin of \eqref{eq:separability} for $g$, and let $\phi' = \phi_2 \circ (\phi_1 \rho)$.
Its first factor meets the bound of \eqref{eq:hypothesis-class}, hence $\phi' \in \mathcal{H}$,
\begin{align}
\lVert \phi_1 \rho \rVert &\leq \lVert \phi_1 \rVert \lVert \rho \rVert \leq 1, \label{eq:aug-class} \\
\phi'(\mathbf{T}'_c) &= \phi_2 \phi_1 \rho\big([ g(c) ; \omega(c) ]\big) = \phi_2 \phi_1 g(c) = \phi(\mathbf{T}_c), \label{eq:aug-value}
\end{align}
hence $\phi'$ attains the margin of Proposition~\ref{prop:consistency},
\begin{equation}
\label{eq:aug-margin}
d_{\mathcal{S}}\big(\phi'(\mathbf{T}'_{c_1}), \phi'(\mathbf{T}'_{c_2})\big) = d_{\mathcal{S}}\big(\phi(\mathbf{T}_{c_1}), \phi(\mathbf{T}_{c_2})\big) > \delta,
\end{equation}
which is \eqref{eq:separability}.
\end{proof}

\begin{proposition}[Coordinate Permutation]
\label{prop:permutation}
There exists a pair that attains minimality, separability and stability, and that does not attain causality.
\end{proposition}

\begin{proof}[Proof of Proposition~\ref{prop:permutation}]
\textbf{Construction.}
Let $x_0$ be an input and $i \neq j$ two classes with $g(x_0)[i] \neq g(x_0)[j]$, let $\Pi \in \mathbb{R}^{M \times M}$ be the permutation matrix that swaps the coordinates $i$ and $j$, let
\begin{equation}
\label{eq:perm-generator}
g'(c) = \Pi g(c),
\end{equation}
and let the model be that of \eqref{eq:ind-model} and \eqref{eq:ind-substitution}, unchanged.

\textbf{Minimality and stability.}
By Lemma~\ref{lem:reparameterization} with $B = \Pi$, which is invertible.
The model is unchanged, hence $(X, Y)$ keeps its distribution.

\textbf{Separability.}
A permutation matrix reorders the entries of a vector, hence
\begin{equation}
\label{eq:perm-isometry}
\lVert \Pi a \rVert_2 = \lVert a \rVert_2 \quad \text{for every } a \in \mathbb{R}^M, \qquad \lVert \Pi \rVert = \lVert \Pi^{-1} \rVert = 1 .
\end{equation}
Let $\phi = \phi_2 \circ \phi_1$ attain the margin of \eqref{eq:separability} for $g$, and let $\phi' = \phi_2 \circ (\phi_1 \Pi^{-1})$.
Its first factor meets the bound of \eqref{eq:hypothesis-class}, hence $\phi' \in \mathcal{H}$,
\begin{align}
\lVert \phi_1 \Pi^{-1} \rVert &\leq \lVert \phi_1 \rVert \lVert \Pi^{-1} \rVert \leq 1, \label{eq:perm-class} \\
\phi'(\mathbf{T}'_c) &= \phi_2 \phi_1 \Pi^{-1} \Pi g(c) = \phi_2 \phi_1 g(c) = \phi(\mathbf{T}_c), \label{eq:perm-value}
\end{align}
hence $\phi'$ attains the margin of Proposition~\ref{prop:consistency}.

\textbf{Causality.}
By \eqref{eq:ind-model}, \eqref{eq:ind-substitution} and \eqref{eq:perm-generator},
\begin{align}
P'(v \mid x_0) &= P\big(v \mid g(x_0)\big), \label{eq:perm-left} \\
P'\big(v \mid [x_0; \mathbf{T}'_{x_0}]\big) &= P\big(v \mid \Pi g(x_0)\big). \label{eq:perm-right}
\end{align}
By \eqref{eq:cons-explicit} at $v = y^{(i)}$,
\begin{equation}
\label{eq:perm-violation}
P'\big(y^{(i)} \mid x_0\big) = \tfrac{g(x_0)[i]}{2} \neq \tfrac{g(x_0)[j]}{2} = P'\big(y^{(i)} \mid [x_0; \mathbf{T}'_{x_0}]\big),
\end{equation}
hence \eqref{eq:causality} fails at the context $x_0$.
\end{proof}

\begin{proposition}[Input Hash]
\label{prop:hash}
There exists a pair that attains causality, separability and stability, and that does not attain minimality.
\end{proposition}

\begin{proof}[Proof of Proposition~\ref{prop:hash}]
\textbf{Construction.}
Let $\mathcal{X}$ be finite, and let $\kappa : \mathcal{X} \to \{1, \dots, r\}$ be a bijection with $r = \lvert \mathcal{X} \rvert$.
Let the block hold an injective encoding of the input while the generated segment is empty,
\begin{equation}
\label{eq:hash-block}
\xi(x) = e_{\kappa(x)} \in \mathbb{R}^r, \qquad \omega(x) = \xi(x), \qquad \omega([x; u]) = 0 \ \text{ for } u \neq \emptyset,
\end{equation}
and let $g'$ and the model be those of Lemma~\ref{lem:augmentation}.

\textbf{Causality and separability.}
By Lemma~\ref{lem:augmentation}.

\textbf{Stability.}
Let $y_a$ and $y_b$ be generated outputs with $y_a \sim_{\mathrm{sem}} y_b$, and let $i(y)$ denote the index of \eqref{eq:cons-point-mass}.
The generated segment of their contexts is not empty, hence
\begin{align}
\mathbf{T}'_{y_a} &= g'([x_a; y_a]) = [\, g([x_a; y_a]) ; 0 \,] && \eqref{eq:hash-block} \nonumber \\
&= [\, e_{i(y_a)} ; 0 \,] && \eqref{eq:cons-point-mass} \nonumber \\
&= [\, e_{i(y_b)} ; 0 \,] && \Phi(y_a) = \Phi(y_b) \nonumber \\
&= [\, g([x_b; y_b]) ; 0 \,] = \mathbf{T}'_{y_b}, && \eqref{eq:cons-point-mass} \label{eq:hash-lexical}
\end{align}
which is \eqref{eq:lexical-invariance}.
Let $q'$ denote the projection onto the first block, $q'([\tau; \zeta]) = \tau$, which maps the thought representation space to the distributions over $\mathcal{S}$ by \eqref{eq:ind-generator}.
Then
\begin{equation}
\label{eq:hash-mode}
q'(\mathbf{T}'_x) = q'\big([\, g(x) ; \xi(x) \,]\big) = g(x) = P\big(\Phi(Y) \mid X = x\big),
\end{equation}
which is \eqref{eq:mode-collapse}.

\textbf{Minimality.}
Let $\mathbf{T}$ denote $g(X)$, and let $\mathbf{T}'$ denote $g'(X) = [\, \mathbf{T} ; \xi(X) \,]$.
By the chain rule for mutual information,
\begin{equation}
\label{eq:hash-chain}
I(X; \mathbf{T}') = I\big(X; [\mathbf{T}; \xi(X)]\big) = I(X; \mathbf{T}) + I\big(X; \xi(X) \mid \mathbf{T}\big).
\end{equation}
Expanding the second term, and using that $\xi(X)$ is a function of $X$,
\begin{align}
I\big(X; \xi(X) \mid \mathbf{T}\big) &= \mathbb{H}\big(\xi(X) \mid \mathbf{T}\big) - \mathbb{H}\big(\xi(X) \mid X, \mathbf{T}\big), \label{eq:hash-cmi} \\
\mathbb{H}\big(\xi(X) \mid X, \mathbf{T}\big) &= 0 . \label{eq:hash-zero}
\end{align}
Since $\xi$ is injective, it has a left inverse $\xi^{-1}$ on $\xi(\mathcal{X})$,
\begin{equation}
\label{eq:hash-inverse}
\xi^{-1}(\xi(x)) = x \qquad \text{for every } x \in \mathcal{X},
\end{equation}
hence $X$ is a function of $\xi(X)$.
The entropy of a function of a random variable is at most the entropy of that variable, which bounds the conditional entropy in both directions,
\begin{align}
\mathbb{H}\big(\xi(X) \mid \mathbf{T}\big) &\leq \mathbb{H}(X \mid \mathbf{T}), && \xi(X) \text{ is a function of } X \label{eq:hash-entropy-one} \\
\mathbb{H}(X \mid \mathbf{T}) &\leq \mathbb{H}\big(\xi(X) \mid \mathbf{T}\big). && X = \xi^{-1}\big(\xi(X)\big) \text{ by } \eqref{eq:hash-inverse} \label{eq:hash-entropy-two}
\end{align}
Substituting the resulting equality and \eqref{eq:hash-zero} into \eqref{eq:hash-cmi}, and then into \eqref{eq:hash-chain},
\begin{equation}
\label{eq:hash-excess}
I\big(X; \xi(X) \mid \mathbf{T}\big) = \mathbb{H}(X \mid \mathbf{T}), \qquad I(X; \mathbf{T}') = I(X; \mathbf{T}) + \mathbb{H}(X \mid \mathbf{T}) .
\end{equation}
Let the construction hold two inputs $x \neq x'$ of positive probability that induce one distribution over $\mathcal{S}$, as the two inputs of \cref{sec:valid-space} do, hence one representation by \eqref{eq:cons-generator},
\begin{equation}
\label{eq:hash-collision}
g(x) = g(x'), \qquad x \neq x', \qquad P(X = x) > 0, \qquad P(X = x') > 0.
\end{equation}
Given $\mathbf{T} = g(x)$, the input then takes at least two values of positive probability, hence the added term is positive,
\begin{equation}
\label{eq:hash-positive}
\mathbb{H}(X \mid \mathbf{T}) > 0, \qquad I(X; \mathbf{T}') > I(X; \mathbf{T}) .
\end{equation}
Since $\mathbf{T} = \rho(\mathbf{T}')$ is a function of $\mathbf{T}'$ under the projection $\rho$ of \eqref{eq:aug-projection}, both representations meet the constraint of \eqref{eq:minimality-constrained},
\begin{equation}
\label{eq:hash-feasible}
I(\mathbf{T}'; Y) \geq I(\mathbf{T}; Y) = I(X; Y) \geq I_0,
\end{equation}
by the data processing inequality and by Proposition~\ref{prop:consistency}.
A representation that meets the constraint therefore reaches a smaller value of the objective than $\mathbf{T}'$, and $\mathbf{T}'$ does not attain the minimum in \eqref{eq:minimality-constrained}.
\end{proof}

\begin{proposition}[Collapse]
\label{prop:collapse}
For every margin $\delta \geq 0$, there exists a pair that attains causality, minimality and stability, and that does not attain separability.
\end{proposition}

\begin{proof}[Proof of Proposition~\ref{prop:collapse}]
\textbf{Construction.}
Every input gives positive probability to at least two classes, hence two completed contexts hold outputs of different semantic content,
\begin{equation}
\label{eq:col-two}
[x_1; y_1], [x_2; y_2] \in \mathcal{C}_{\mathrm{out}}, \qquad \Phi(y_1) \neq \Phi(y_2).
\end{equation}
Let the reserved vector be the origin of $\mathbb{R}^M$.
Every value of $g$ lies in $\Delta(\mathcal{S})$, whose entries sum to one, whereas the entries of the origin sum to zero,
\begin{equation}
\label{eq:col-reserved}
\tau^* = 0 \in \mathbb{R}^M, \qquad \sum_{i} g(c)[i] = 1 \neq 0 = \sum_{i} \tau^*[i], \qquad\text{hence}\qquad \tau^* \notin g(\mathcal{C}).
\end{equation}
Let the generator keep $g$ off the completed contexts and map every completed context to $\tau^*$, and let the model end generation at $\tau^*$,
\begin{align}
g'(c) &= g(c) \ \text{ for } c \notin \mathcal{C}_{\mathrm{out}}, \qquad g'(c) = \tau^* \ \text{ for } c \in \mathcal{C}_{\mathrm{out}}, \label{eq:col-generator} \\
P'(v \mid c) &= P'\big(v \mid g'(c)\big), \label{eq:col-model} \\
P'(v \mid [x; \mathbf{T}]) &= P(v \mid \mathbf{T}) \ \text{ for } \mathbf{T} \neq \tau^*, \qquad P'(v \mid [x; \tau^*]) = \mathbf{1}[v = \emptyset]. \label{eq:col-substitution}
\end{align}
Off the completed contexts the model generates as before, and at a completed context it generates nothing,
\begin{equation}
\label{eq:col-same}
P'(v \mid c) = P\big(v \mid g(c)\big) = P(v \mid c) \ \text{ for } c \notin \mathcal{C}_{\mathrm{out}}, \qquad P'(v \mid c) = \mathbf{1}[v = \emptyset] \ \text{ for } c \in \mathcal{C}_{\mathrm{out}},
\end{equation}
hence the pair keeps the inputs, the outputs, and the output distributions of Proposition~\ref{prop:consistency}.

\textbf{Causality.}
At a context $[x; u] \notin \mathcal{C}_{\mathrm{out}}$, let $\mathbf{T}' = g'([x; u]) = g([x; u]) \neq \tau^*$.
Then
\begin{align}
P'(v \mid [x; u]) &= P\big(v \mid g([x; u])\big), && \eqref{eq:col-model}, \eqref{eq:col-substitution} \nonumber \\
P'(v \mid [x; \mathbf{T}']) &= P(v \mid \mathbf{T}') = P\big(v \mid g([x; u])\big). && \eqref{eq:col-substitution} \nonumber
\end{align}
At a completed context $[x; y]$, where $\mathbf{T}' = \tau^*$,
\begin{equation}
P'(v \mid [x; y]) = P'(v \mid \tau^*) = \mathbf{1}[v = \emptyset] = P'(v \mid [x; \tau^*]). \nonumber
\end{equation}
In both cases the two distributions are equal, which is \eqref{eq:causality} at every context.

\textbf{Minimality and mode collapse resistance.}
An input has an empty generated segment and is not a completed context, hence
\begin{equation}
\label{eq:col-inputs}
g'(x) = g(x), \qquad P'(Y \mid X = x) = P(Y \mid X = x).
\end{equation}
The triples $(X, g'(X), Y)$ and $(X, g(X), Y)$ therefore have one joint distribution,
\begin{equation}
P'(X = x, \mathbf{T}' = \tau, Y = y) = P(X = x) \, \mathbf{1}[g(x) = \tau] \, P(Y = y \mid X = x) = P(X = x, \mathbf{T} = \tau, Y = y). \nonumber
\end{equation}
Mutual information is a function of the joint distribution, hence the objective and the constraint of \eqref{eq:minimality-constrained} take the values at which Proposition~\ref{prop:consistency} attains the minimum.
Let $q'$ equal the map $q$ of Proposition~\ref{prop:consistency} on $g(\mathcal{C})$, and let $q'(\tau^*)$ be any element of $\Delta(\mathcal{S})$.
Then
\begin{equation}
q'(\mathbf{T}'_x) = q\big(g(x)\big) = P\big(\Phi(Y) \mid X = x\big), \nonumber
\end{equation}
which is \eqref{eq:mode-collapse}.

\textbf{Stability.}
For outputs with $y_i \sim_{\mathrm{sem}} y_j$ generated for the input $x$, both completed contexts lie in $\mathcal{C}_{\mathrm{out}}$,
\begin{equation}
\mathbf{T}'_{y_i} = g'([x; y_i]) = \tau^* = g'([x; y_j]) = \mathbf{T}'_{y_j}, \nonumber
\end{equation}
which is \eqref{eq:lexical-invariance} with equality.
Together with mode collapse resistance, the pair attains stability.

\textbf{Separability.}
The two completed contexts of \eqref{eq:col-two} share $\tau^*$, and every $\phi \in \mathcal{H}$ maps them to one point,
\begin{equation}
d_{\mathcal{S}}\big(\phi(\mathbf{T}'_{[x_1;y_1]}), \phi(\mathbf{T}'_{[x_2;y_2]})\big) = d_{\mathcal{S}}\big(\phi(\tau^*), \phi(\tau^*)\big) = 0 \leq \delta, \nonumber
\end{equation}
where the second equality holds because a metric assigns distance zero to a point and itself.
No $\phi \in \mathcal{H}$ therefore meets \eqref{eq:separability} at any margin $\delta \geq 0$, and the pair does not attain separability.
\end{proof}

\begin{proposition}[Lexical Flag]
\label{prop:flag}
There exists a pair that attains causality, minimality and separability, and that does not attain stability.
\end{proposition}

\begin{proof}[Proof of Proposition~\ref{prop:flag}]
\textbf{Construction.}
Let $\lambda > 0$, let $r = 1$, and let
\begin{equation}
\label{eq:flag-block}
\omega([x; u]) = \lambda \cdot \mathbf{1}\big[\, u \text{ ends with a space} \,\big].
\end{equation}
The map $g'$ and the model are those of Lemma~\ref{lem:augmentation}.

\textbf{Causality and separability.}
By Lemma~\ref{lem:augmentation}.

\textbf{Minimality.}
The context $x$ has no generated segment, hence $u$ is empty and
\begin{equation}
\label{eq:flag-input}
\omega(x) = 0, \qquad \mathbf{T}'_x = g'(x) = [\, g(x) ; 0 \,] = [\, \mathbf{T}_x ; 0 \,] .
\end{equation}
The map $\tau \mapsto [\tau; 0]$ is injective, and an injective map leaves mutual information unchanged,
\begin{align}
I(X; \mathbf{T}') &= I\big(X; [\mathbf{T}; 0]\big) = I(X; \mathbf{T}), \label{eq:flag-objective} \\
I(\mathbf{T}'; Y) &= I\big([\mathbf{T}; 0]; Y\big) = I(\mathbf{T}; Y) = I(X; Y) \geq I_0, \label{eq:flag-constraint}
\end{align}
where the last two equalities are those of Proposition~\ref{prop:consistency}.
The objective and the constraint of \eqref{eq:minimality-constrained} take the values at which Proposition~\ref{prop:consistency} attains the minimum.

\textbf{Stability.}
Let $y_a$ and $y_b$ be generated outputs for one input $x$ that differ in a trailing space,
\begin{equation}
\label{eq:flag-equivalent}
\Phi(y_a) = \Phi(y_b), \qquad\text{hence}\qquad y_a \sim_{\mathrm{sem}} y_b .
\end{equation}
Their representations differ in the appended coordinate,
\begin{align}
\mathbf{T}'_{y_a} &= g'([x; y_a]) = [\, g([x; y_a]) ; \lambda \,], \label{eq:flag-first} \\
\mathbf{T}'_{y_b} &= g'([x; y_b]) = [\, g([x; y_b]) ; 0 \,], \label{eq:flag-second}
\end{align}
and their first blocks are one vector, by \eqref{eq:cons-point-mass} and \eqref{eq:flag-equivalent},
\begin{equation}
\label{eq:flag-blocks}
g([x; y_a]) = e_{i(y_a)} = e_{i(y_b)} = g([x; y_b]),
\end{equation}
hence the appended coordinate holds the whole difference,
\begin{equation}
\label{eq:flag-distance}
\mathbf{T}'_{y_a} - \mathbf{T}'_{y_b} = [\, 0 ; \lambda \,], \qquad \lVert \mathbf{T}'_{y_a} - \mathbf{T}'_{y_b} \rVert_2 = \lambda .
\end{equation}
The scalar $\lambda$ is arbitrary, hence the two representations are at an arbitrarily large distance and \eqref{eq:lexical-invariance} fails.
Definition~\ref{def:stability} states stability as the conjunction of \eqref{eq:lexical-invariance} with \eqref{eq:mode-collapse}, hence the pair does not attain stability.
\end{proof}

\restCorIndependence*

\begin{proof}[Proof of Corollary~\ref{cor:independence}]
Let $\mathcal{P} = \{\mathcal{P}_1, \mathcal{P}_2, \mathcal{P}_3, \mathcal{P}_4\}$ hold the four properties in the order of Definitions~\ref{def:causality} to~\ref{def:stability}, and let $C_i$ denote the pair constructed in the proposition that does not attain $\mathcal{P}_i$.
Proposition~\ref{prop:permutation} constructs $C_1$, Proposition~\ref{prop:hash} constructs $C_2$, Proposition~\ref{prop:collapse} constructs $C_3$, and Proposition~\ref{prop:flag} constructs $C_4$.
Each pair attains three of the four properties,
\begin{equation}
\label{eq:indep-pairs}
C_i \text{ attains every property in } \mathcal{P} \setminus \{\mathcal{P}_i\}, \qquad C_i \text{ does not attain } \mathcal{P}_i .
\end{equation}
If the conjunction of $\mathcal{P} \setminus \{\mathcal{P}_i\}$ implied $\mathcal{P}_i$, then $C_i$ would attain $\mathcal{P}_i$, which contradicts \eqref{eq:indep-pairs}.
Therefore
\begin{equation}
\label{eq:indep-conclusion}
\bigwedge \big( \mathcal{P} \setminus \{\mathcal{P}_i\} \big) \nRightarrow \mathcal{P}_i \qquad \text{for every } i \in \{1, \dots, 4\},
\end{equation}
and the four properties are independent.
\end{proof}

\subsection{Completeness}
\label{app:proofs-completeness}

Causality gives the substitution, separability and lexical invariance give the bijection on completed outputs, and minimality gives the bijection on inputs.

\restThmCompleteness*

\begin{proof}[Proof of Theorem~\ref{thm:completeness}]
\textbf{Substitution.}
At a context $[x; u] \in \mathcal{C}$, the thought representation derived from $u$ is the value of the generator,
\begin{equation}
\label{eq:comp-derived}
\mathbf{T} = \mathbf{T}_{[x;u]} = g([x; u]).
\end{equation}
Since Definition~\ref{def:causality} holds for every partition $y = [u; v]$ of a generated output, it applies at every context in $\mathcal{C}$,
\begin{equation}
\label{eq:comp-substitution}
P(v \mid [x; u]) = P(v \mid [x; \mathbf{T}]) = P\big(v \mid [x; \mathbf{T}_{[x;u]}]\big),
\end{equation}
which is \eqref{eq:valid-substitution}.

\textbf{Completed outputs.}
Conditioning on a complete output fixes its semantic output, hence the distribution of a completed context assigns probability one to one element of $\mathcal{S}$,
\begin{equation}
\label{eq:comp-point-mass}
\mu_{[x;y]}(s) = P\big(\Phi(Y) = s \mid [x; y]\big) = \mathbf{1}[\, s = \Phi(y) \,], \qquad s \in \mathcal{S}.
\end{equation}
Two completed contexts therefore induce one distribution if and only if their outputs are semantically equivalent,
\begin{equation}
\label{eq:comp-out-equivalent}
\mu_{[x_1;y_1]} = \mu_{[x_2;y_2]} \quad \Longleftrightarrow \quad \Phi(y_1) = \Phi(y_2).
\end{equation}
For the forward arrow of \eqref{eq:valid-bijection}, let $\Phi(y_1) \neq \Phi(y_2)$.
Separability at the two contexts gives some $\phi \in \mathcal{H}$ with
\begin{align}
d_{\mathcal{S}}\big(\phi(\mathbf{T}_{[x_1;y_1]}), \phi(\mathbf{T}_{[x_2;y_2]})\big) &> \delta \geq 0, && \eqref{eq:separability} \nonumber \\
\phi(\mathbf{T}_{[x_1;y_1]}) &\neq \phi(\mathbf{T}_{[x_2;y_2]}), && d_{\mathcal{S}} \text{ is a metric} \nonumber \\
\mathbf{T}_{[x_1;y_1]} &\neq \mathbf{T}_{[x_2;y_2]}, && \phi \text{ is a map} \label{eq:comp-out-forward}
\end{align}
which is the contrapositive of the forward arrow.
For the backward arrow,
\begin{align}
\mu_{[x_1;y_1]} = \mu_{[x_2;y_2]} &\quad\Longrightarrow\quad \Phi(y_1) = \Phi(y_2) && \eqref{eq:comp-out-equivalent} \nonumber \\
&\quad\Longrightarrow\quad y_1 \sim_{\mathrm{sem}} y_2 && \eqref{eq:semantic-equivalence} \nonumber \\
&\quad\Longrightarrow\quad \mathbf{T}_{y_1} = \mathbf{T}_{y_2} && \eqref{eq:lexical-invariance} \nonumber \\
&\quad\Longrightarrow\quad g([x_1; y_1]) = g([x_2; y_2]), && \mathbf{T}_y = g([x; y]) \label{eq:comp-out-backward}
\end{align}
where the third line applies lexical invariance with equality.

\textbf{Inputs, forward arrow.}
Let $\mathbf{T}$ denote $g(X)$.
The constraint of minimality at $I_0 = I(X; Y)$ gives a lower bound on $I(\mathbf{T}; Y)$.
Since $\mathbf{T}$ is a function of $X$, the data processing inequality gives an upper bound,
\begin{align}
I(\mathbf{T}; Y) &\geq I(X; Y), && \eqref{eq:minimality-constrained} \nonumber \\
I(\mathbf{T}; Y) &\leq I(X; Y), && \text{data processing inequality} \nonumber \\
I(\mathbf{T}; Y) &= I(X; Y). \label{eq:comp-sufficient-mi}
\end{align}
Lemma~\ref{lem:decomposition} with $Z = \mathbf{T}$ then gives
\begin{equation}
\label{eq:comp-cond-zero}
I(X; Y) = I(\mathbf{T}; Y) + I(X; Y \mid \mathbf{T}), \qquad I(X; Y \mid \mathbf{T}) = 0 .
\end{equation}
Conditional mutual information is zero if and only if the two variables are independent given the third.
Since $X = x$ fixes $\mathbf{T} = g(x)$, every input $x$ of positive probability satisfies
\begin{equation}
\label{eq:comp-sufficient}
P(Y \mid X = x) = P\big(Y \mid X = x, \mathbf{T} = g(x)\big) = P\big(Y \mid \mathbf{T} = g(x)\big).
\end{equation}
Two inputs with one thought representation $\tau$ therefore induce one distribution,
\begin{equation}
\label{eq:comp-in-forward}
g(x_1) = g(x_2) = \tau \quad \Longrightarrow \quad \mu_{x_1} = P(Y \mid \mathbf{T} = \tau) = \mu_{x_2}.
\end{equation}

\textbf{Inputs, backward arrow.}
Let $W$ denote $\mu_X$, the distribution over outputs that the random input induces, which is a function of $X$ with value $\mu_x$ at $X = x$.
Conditioning on $W = w$ averages $P(Y \mid X = x)$ over the inputs with $\mu_x = w$,
\begin{align}
P(Y \mid W = w) &= \sum_{x \,:\, \mu_x = w} P(X = x \mid W = w) \, \mu_x && \text{total probability} \nonumber \\
&= \sum_{x \,:\, \mu_x = w} P(X = x \mid W = w) \, w && \mu_x = w \text{ in each term} \nonumber \\
&= w \sum_{x \,:\, \mu_x = w} P(X = x \mid W = w) && \text{distributivity} \nonumber \\
&= w, && \text{probabilities sum to one} \label{eq:comp-w-sufficient}
\end{align}
hence the distribution of $Y$ given the input equals its distribution given $W$, which is the independence of $X$ and $Y$ given $W$,
\begin{equation}
\label{eq:comp-w-independent}
P(Y \mid X = x) = \mu_x = P(Y \mid W = \mu_x), \qquad I(X; Y \mid W) = 0 .
\end{equation}
Lemma~\ref{lem:decomposition} with $Z = W$ gives
\begin{equation}
\label{eq:comp-w-feasible}
I(W; Y) = I(X; Y) - I(X; Y \mid W) = I(X; Y) \geq I_0,
\end{equation}
hence $W$ meets the constraint of \eqref{eq:minimality-constrained}.
Since $\mathbf{T}$ attains the minimum of \eqref{eq:minimality-constrained} over the representations that meet the constraint, $W$ gives an upper bound on its objective,
\begin{equation}
\label{eq:comp-min-compare}
I(X; \mathbf{T}) \leq I(X; W).
\end{equation}
For a function $Z$ of $X$, the conditional entropy $\mathbb{H}(Z \mid X)$ is zero, hence
\begin{equation}
\label{eq:comp-mi-entropy}
I(X; Z) = \mathbb{H}(Z) - \mathbb{H}(Z \mid X) = \mathbb{H}(Z).
\end{equation}
Substituting $Z = \mathbf{T}$ and $Z = W$ into \eqref{eq:comp-min-compare} gives
\begin{equation}
\label{eq:comp-entropy-upper}
\mathbb{H}(\mathbf{T}) \leq \mathbb{H}(W).
\end{equation}
By \eqref{eq:comp-sufficient}, $W$ is a function of $\mathbf{T}$, with $q_{\mathrm{in}}(\tau) = P(Y \mid \mathbf{T} = \tau)$,
\begin{equation}
\label{eq:comp-w-function}
W = \mu_X = P\big(Y \mid \mathbf{T} = g(X)\big) = q_{\mathrm{in}}(\mathbf{T}).
\end{equation}
Expanding the joint entropy of $\mathbf{T}$ and $W$ by the chain rule in both orders,
\begin{align}
\mathbb{H}(\mathbf{T}, W) &= \mathbb{H}(\mathbf{T}) + \mathbb{H}(W \mid \mathbf{T}) = \mathbb{H}(\mathbf{T}), && \eqref{eq:comp-w-function} \nonumber \\
\mathbb{H}(\mathbf{T}, W) &= \mathbb{H}(W) + \mathbb{H}(\mathbf{T} \mid W), && \text{chain rule} \nonumber \\
\mathbb{H}(\mathbf{T}) &= \mathbb{H}(W) + \mathbb{H}(\mathbf{T} \mid W). \label{eq:comp-entropy-split}
\end{align}
A conditional entropy is nonnegative, hence \eqref{eq:comp-entropy-split} bounds $\mathbb{H}(\mathbf{T})$ from below, and \eqref{eq:comp-entropy-upper} bounds it from above,
\begin{align}
\mathbb{H}(W) &\leq \mathbb{H}(\mathbf{T}) \leq \mathbb{H}(W), && \eqref{eq:comp-entropy-split}, \eqref{eq:comp-entropy-upper} \nonumber \\
\mathbb{H}(\mathbf{T} \mid W) &= \mathbb{H}(\mathbf{T}) - \mathbb{H}(W) = 0. && \eqref{eq:comp-entropy-split} \label{eq:comp-cond-entropy}
\end{align}
For a value $w$ of positive probability, let $p_\tau = P(\mathbf{T} = \tau \mid W = w)$.
The conditional entropy averages the entropies given each value of $W$, and each of these is a sum of nonnegative terms,
\begin{align}
0 &= \sum_{w} P(W = w) \, \mathbb{H}(\mathbf{T} \mid W = w), && \eqref{eq:comp-cond-entropy} \nonumber \\
\mathbb{H}(\mathbf{T} \mid W = w) &= -\sum_{\tau} p_\tau \log p_\tau, \qquad -p_\tau \log p_\tau \geq 0, && \text{definition of } \mathbb{H} \nonumber
\end{align}
hence every term vanishes, and $-p_\tau \log p_\tau = 0$ holds only at $p_\tau \in \{0, 1\}$.
Given $W = w$, the thought representation therefore takes a single value.
Two inputs of positive probability with one distribution both lie in that event,
\begin{align}
&\mu_{x_1} = \mu_{x_2} = w, \nonumber \\
&X = x_1 \quad \Longrightarrow \quad W = w \ \text{ and } \ \mathbf{T} = g(x_1), \nonumber \\
&X = x_2 \quad \Longrightarrow \quad W = w \ \text{ and } \ \mathbf{T} = g(x_2), \nonumber
\end{align}
hence
\begin{equation}
\label{eq:comp-in-backward}
g(x_1) = g(x_2),
\end{equation}
which is the backward arrow of \eqref{eq:valid-bijection} on $\mathcal{C}_{\mathrm{in}}$.

Both arrows of \eqref{eq:valid-bijection} hold on $\mathcal{C}_{\mathrm{out}}$ by \eqref{eq:comp-out-forward} and \eqref{eq:comp-out-backward}, and on $\mathcal{C}_{\mathrm{in}}$ by \eqref{eq:comp-in-forward} and \eqref{eq:comp-in-backward}.
These arrows and \eqref{eq:comp-substitution} make the thought representation space valid.
\end{proof}

\begin{corollary}[Bijection with the Semantic Output Space]
\label{cor:semantic-bijection}
Under the conditions of Theorem~\ref{thm:completeness}, let $\Phi(\mathcal{Y}) = \{\, \Phi(y) : y \in \mathcal{Y} \,\}$ denote the semantic outputs of the generated outputs.
The map $\Psi$ with $\Psi(\Phi(y)) = \mathbf{T}_{[x;y]}$ is a bijection from $\Phi(\mathcal{Y})$ onto $\mathcal{T}_{\mathrm{out}}$.
\end{corollary}

\begin{proof}[Proof of Corollary~\ref{cor:semantic-bijection}]
Two point masses are equal if and only if their supports coincide, hence the map $s \mapsto \mathbf{1}[\, \cdot = s \,]$ is injective,
\begin{equation}
\label{eq:cor-point-mass}
\mathbf{1}[\, \cdot = s \,] = \mathbf{1}[\, \cdot = s' \,] \quad \Longleftrightarrow \quad s = s' .
\end{equation}
By \eqref{eq:comp-point-mass}, it maps $\Phi(\mathcal{Y})$ onto the distributions of the completed contexts,
\begin{equation}
\label{eq:cor-onto}
\big\{\, \mu_{[x;y]} : [x; y] \in \mathcal{C}_{\mathrm{out}} \,\big\} = \big\{\, \mathbf{1}[\, \cdot = s \,] : s \in \Phi(\mathcal{Y}) \,\big\}.
\end{equation}
Theorem~\ref{thm:completeness} makes $q_{\mathrm{out}}$ a bijection from $\mathcal{T}_{\mathrm{out}}$ onto the same set, and the composition of two bijections is a bijection,
\begin{equation}
\label{eq:cor-composition}
\Psi(s) = q_{\mathrm{out}}^{-1}\big(\mathbf{1}[\, \cdot = s \,]\big), \qquad \Psi\big(\Phi(y)\big) = q_{\mathrm{out}}^{-1}\big(\mu_{[x;y]}\big) = \mathbf{T}_{[x;y]} .
\end{equation}
\end{proof}

The bijection that \cref{sec:valid-space} posits between the thought representation space and the semantic output space follows on the completed outputs.
When lexical invariance holds up to a tolerance, the backward arrow on completed outputs and the bijection of Corollary~\ref{cor:semantic-bijection} hold up to the same tolerance.

\subsection{Necessity}
\label{app:proofs-necessity}

The substitution condition of a valid space is causality, and the two bijections give minimality, separability, and stability.

\restPropNecessity*

\begin{proof}[Proof of Proposition~\ref{prop:necessity}]
\textbf{Causality.}
Let $[x; u] \in \mathcal{C}$ be a context, and let $y = [u; v]$ be a partition of a generated output.
The substitution condition of Definition~\ref{def:valid-space} and the generator of \cref{sec:valid-space} give
\begin{align}
P(v \mid [x; u]) &= P\big(v \mid [x; \mathbf{T}_{[x;u]}]\big), && \eqref{eq:valid-substitution} \nonumber \\
\mathbf{T} &= \mathbf{T}_{[x;u]} = g([x; u]), && \mathbf{T}_c = g(c) \nonumber \\
P(v \mid [x; u]) &= P(v \mid [x; \mathbf{T}]), && \text{substitution} \nonumber
\end{align}
which is \eqref{eq:causality}.
The context and the partition are arbitrary, hence causality holds at every context.

\textbf{Minimality, the constraint.}
Let $\mathbf{T} = g(X)$, and fix a value $\tau$ with $P(\mathbf{T} = \tau) > 0$ and an input $x$ with $g(x) = \tau$ and $P(X = x) > 0$.
Conditioning on $\mathbf{T} = \tau$ averages the output distributions of the inputs that map to $\tau$,
\begin{align}
P(Y \mid \mathbf{T} = \tau) &= \sum_{x' : g(x') = \tau} P(X = x' \mid \mathbf{T} = \tau) \, \mu_{x'} && \text{total probability} \nonumber \\
&= \sum_{x' : g(x') = \tau} P(X = x' \mid \mathbf{T} = \tau) \, \mu_{x} && \eqref{eq:valid-bijection} \text{ on } \mathcal{C}_{\mathrm{in}} \nonumber \\
&= \mu_x \sum_{x' : g(x') = \tau} P(X = x' \mid \mathbf{T} = \tau) && \text{distributivity} \nonumber \\
&= \mu_x . && \text{probabilities sum to one} \label{eq:nec-sufficient}
\end{align}
Since $X = x$ fixes $\mathbf{T} = g(x)$,
\begin{equation}
P\big(Y \mid X = x, \mathbf{T} = g(x)\big) = P(Y \mid X = x) = \mu_x = P\big(Y \mid \mathbf{T} = g(x)\big), \nonumber
\end{equation}
which is the independence of $X$ and $Y$ given $\mathbf{T}$.
A conditional mutual information is zero if and only if the two variables are independent given the third, hence Lemma~\ref{lem:decomposition} with $Z = \mathbf{T}$ gives
\begin{equation}
\label{eq:nec-constraint}
I(X; Y) = I(\mathbf{T}; Y) + \underbrace{I(X; Y \mid \mathbf{T})}_{= 0}, \qquad\text{hence}\qquad I(\mathbf{T}; Y) = I(X; Y) = I_0 .
\end{equation}

\textbf{Minimality, the minimum.}
Let $g'$ be a generator whose representation $\mathbf{T}' = g'(X)$ meets the constraint.
The constraint and the data processing inequality give
\begin{align}
I(\mathbf{T}'; Y) &\geq I(X; Y), && \eqref{eq:minimality-constrained} \nonumber \\
I(\mathbf{T}'; Y) &\leq I(X; Y), && \text{data processing inequality} \nonumber \\
I(\mathbf{T}'; Y) &= I(X; Y). \nonumber
\end{align}
Lemma~\ref{lem:decomposition} with $Z = \mathbf{T}'$ then gives, as in \eqref{eq:nec-constraint},
\begin{equation}
\label{eq:nec-other-sufficient}
I(X; Y \mid \mathbf{T}') = 0, \qquad\text{hence}\qquad \mu_x = P\big(Y \mid \mathbf{T}' = g'(x)\big)
\end{equation}
for every input $x$ of positive probability.
Let $W$ denote $\mu_X$, the distribution over outputs that the random input induces.
Let $q'$ and $q_{\mathrm{in}}$ map a value of $\mathbf{T}'$ and a value of $\mathbf{T}$ to the distribution of $Y$ given that value, and let $h$ map a distribution $w$ to the thought representation of the inputs that induce it,
\begin{equation}
\label{eq:nec-maps}
q'(\tau') = P(Y \mid \mathbf{T}' = \tau'), \qquad q_{\mathrm{in}}(\tau) = P(Y \mid \mathbf{T} = \tau), \qquad h(w) = g(x) \ \text{ for any input } x \text{ with } \mu_x = w .
\end{equation}
The map $h$ is well defined, since the bijection \eqref{eq:valid-bijection} on $\mathcal{C}_{\mathrm{in}}$ gives $g(x_1) = g(x_2)$ whenever $\mu_{x_1} = \mu_{x_2}$.
The three maps express $W$ as a function of $\mathbf{T}'$ and of $\mathbf{T}$, and $\mathbf{T}$ as a function of $W$,
\begin{align}
W &= q'(\mathbf{T}'), && \eqref{eq:nec-other-sufficient} \nonumber \\
W &= q_{\mathrm{in}}(\mathbf{T}), && \eqref{eq:nec-sufficient} \nonumber \\
\mathbf{T} &= h(W). && \eqref{eq:nec-maps} \label{eq:nec-functions}
\end{align}
Let $V$ be a random variable, and let $Z = \gamma(V)$ for a map $\gamma$.
Since $\mathbb{H}(Z \mid V) = 0$ and a conditional entropy is nonnegative, the chain rule in the two orders gives
\begin{equation}
\label{eq:nec-entropy-function}
\mathbb{H}(V) = \mathbb{H}(V) + \mathbb{H}(Z \mid V) = \mathbb{H}(V, Z) = \mathbb{H}(Z) + \mathbb{H}(V \mid Z) \geq \mathbb{H}(Z).
\end{equation}
Applying \eqref{eq:nec-entropy-function} to the three maps of \eqref{eq:nec-functions},
\begin{align}
\mathbb{H}(W) &\leq \mathbb{H}(\mathbf{T}'), \qquad \mathbb{H}(W) \leq \mathbb{H}(\mathbf{T}), \qquad \mathbb{H}(\mathbf{T}) \leq \mathbb{H}(W), \nonumber \\
\mathbb{H}(\mathbf{T}) &= \mathbb{H}(W) \leq \mathbb{H}(\mathbf{T}'). \nonumber
\end{align}
Both $\mathbf{T}$ and $\mathbf{T}'$ are functions of $X$, hence \eqref{eq:comp-mi-entropy} turns each mutual information with $X$ into an entropy,
\begin{equation}
\label{eq:nec-minimum}
I(X; \mathbf{T}) = \mathbb{H}(\mathbf{T}) \leq \mathbb{H}(\mathbf{T}') = I(X; \mathbf{T}'),
\end{equation}
hence $g(X)$ attains the minimum of \eqref{eq:minimality-constrained} at $I_0 = I(X; Y)$.

\textbf{Separability on completed outputs.}
Let $[x_1; y_1], [x_2; y_2] \in \mathcal{C}_{\mathrm{out}}$ with $\Phi(y_1) \neq \Phi(y_2)$.
Conditioning on a complete output fixes its semantic output, and two point masses are equal if and only if their supports coincide,
\begin{align}
\mu_{[x;y]}(s) &= P\big(\Phi(Y) = s \mid [x; y]\big) = \mathbf{1}[\, s = \Phi(y) \,], \qquad s \in \mathcal{S}, \nonumber \\
\mu_{[x_1;y_1]} &\neq \mu_{[x_2;y_2]}, && \Phi(y_1) \neq \Phi(y_2) \nonumber \\
g([x_1; y_1]) &\neq g([x_2; y_2]). && \eqref{eq:valid-bijection} \text{ on } \mathcal{C}_{\mathrm{out}} \label{eq:nec-out-distinct}
\end{align}

\textbf{Separability on inputs.}
Let $x_1, x_2 \in \mathcal{C}_{\mathrm{in}}$, and let $y_1$ and $y_2$ denote the outputs that $f_\theta$ generates for them under the fixed decoding rule of Definition~\ref{def:generator}, with $\Phi(y_1) \neq \Phi(y_2)$.
Since $\Phi$ is a map, $y_1 \neq y_2$.
Let $y_t$ denote the $t$-th token of an output $y$, let $y_{<t}$ and $y_{\leq t}$ denote its tokens before and up to position $t$, and let $Y_{<t}$ and $Y_{\leq t}$ denote the same prefixes of $Y$.
Each next-token distribution is a ratio of two marginals of $\mu_x$,
\begin{align}
P(y_t \mid x, y_{<t}) &= \frac{P(Y_{\leq t} = y_{\leq t} \mid x)}{P(Y_{<t} = y_{<t} \mid x)}, && \text{conditional probability} \nonumber \\
P(Y_{\leq t} = y_{\leq t} \mid x) &= \sum_{y' : y'_{\leq t} = y_{\leq t}} \mu_x(y'). && \text{marginalization} \nonumber
\end{align}
We assume that the fixed decoding rule chooses each token as a deterministic function of the next-token distributions of $f_\theta$, as greedy decoding and beam search do.
The decoded output is then a function of $\mu_x$.
Let $D$ denote this function.
Since the two outputs differ, so do the two distributions and the two thought representations,
\begin{align}
y &= D(\mu_x), \nonumber \\
\mu_{x_1} &\neq \mu_{x_2}, && y_1 \neq y_2,\ D \text{ is a map} \nonumber \\
g(x_1) &\neq g(x_2). && \eqref{eq:valid-bijection} \text{ on } \mathcal{C}_{\mathrm{in}} \label{eq:nec-in-distinct}
\end{align}

\textbf{Separability, the projection.}
Let $\mathbf{T}_1 \neq \mathbf{T}_2$ denote the two thought representations of \eqref{eq:nec-out-distinct} or \eqref{eq:nec-in-distinct} in $\mathcal{T} \subset \mathbb{R}^m$, and let $a = \mathbf{T}_1 - \mathbf{T}_2$ and $e = a / \lVert a \rVert_2$.
Let $b \in \mathbb{R}^d$ be a unit vector, and let $\phi$ factor through $\mathbb{R}^k$ with $k = 1$,
\begin{equation}
\phi_1(z) = \langle e, z \rangle, \qquad \phi_2(t) = t\, b, \qquad \phi = \phi_2 \circ \phi_1 . \nonumber
\end{equation}
By the Cauchy--Schwarz inequality and the operator norm of \eqref{eq:operator-norm},
\begin{align}
\lVert \phi_1 \rVert &= \sup_{z \neq 0} \frac{\lvert \langle e, z \rangle \rvert}{\lVert z \rVert_2} \leq \lVert e \rVert_2 = 1, && \text{Cauchy--Schwarz} \nonumber \\
\lVert \phi_2 \rVert &= \sup_{t \neq 0} \frac{\lVert t\, b \rVert_2}{\lvert t \rvert} = \lVert b \rVert_2 = 1, \nonumber
\end{align}
hence $\phi \in \mathcal{H}$ by \eqref{eq:hypothesis-class}.
Applying $\phi$ to both thought representations,
\begin{align}
\phi(\mathbf{T}_1) - \phi(\mathbf{T}_2) &= \langle e, \mathbf{T}_1 \rangle b - \langle e, \mathbf{T}_2 \rangle b && \text{definitions of } \phi_1, \phi_2 \nonumber \\
&= \langle e, a \rangle b && \text{linearity} \nonumber \\
&= \lVert a \rVert_2\, b, && e = a / \lVert a \rVert_2 \nonumber
\end{align}
whose norm is $\lVert a \rVert_2 > 0$, hence $\phi(\mathbf{T}_1) \neq \phi(\mathbf{T}_2)$.
A metric is zero if and only if its two arguments are equal, hence
\begin{equation}
d_{\mathcal{S}}\big(\phi(\mathbf{T}_1), \phi(\mathbf{T}_2)\big) > 0 = \delta, \nonumber
\end{equation}
which is \eqref{eq:separability} at margin $\delta = 0$ on $\mathcal{C}_{\mathrm{in}}$ and on $\mathcal{C}_{\mathrm{out}}$.

\textbf{Lexical invariance.}
For outputs with $y_i \sim_{\mathrm{sem}} y_j$ generated for the input $x$, the two point masses coincide,
\begin{align}
\mu_{[x;y_i]} &= \mathbf{1}[\,\cdot = \Phi(y_i)\,] = \mathbf{1}[\,\cdot = \Phi(y_j)\,] = \mu_{[x;y_j]}, && \eqref{eq:semantic-equivalence} \nonumber \\
\mathbf{T}_{y_i} &= g([x; y_i]) = g([x; y_j]) = \mathbf{T}_{y_j}, && \eqref{eq:valid-bijection} \text{ on } \mathcal{C}_{\mathrm{out}} \nonumber
\end{align}
which is \eqref{eq:lexical-invariance} with equality.

\textbf{Mode collapse resistance.}
Grouping the outputs by their semantic output gives the induced distribution over $\mathcal{S}$,
\begin{equation}
P\big(\Phi(Y) = s \mid X = x\big) = \sum_{y : \Phi(y) = s} \mu_x(y), \qquad s \in \mathcal{S}. \nonumber
\end{equation}
Let $\sigma_0$ be any element of $\Delta(\mathcal{S})$, which is nonempty by Assumption~\ref{as:semantic-space}, and let $q$ map each $\tau \in \mathcal{T}$ to a distribution over $\mathcal{S}$,
\begin{equation}
q(\tau)(s) =
\begin{cases}
\sum_{y : \Phi(y) = s} \mu_x(y) \ \text{ for any input } x \text{ with } g(x) = \tau, & \tau \in \mathcal{T}_{\mathrm{in}}, \\
\sigma_0(s), & \tau \notin \mathcal{T}_{\mathrm{in}}.
\end{cases} \nonumber
\end{equation}
The first case is independent of the choice of $x$, since
\begin{equation}
g(x_1) = g(x_2) \quad \Longrightarrow \quad \mu_{x_1} = \mu_{x_2}, \qquad \eqref{eq:valid-bijection} \text{ on } \mathcal{C}_{\mathrm{in}} \nonumber
\end{equation}
and at every input,
\begin{equation}
q(\mathbf{T}_x) = P\big(\Phi(Y) \mid X = x\big), \nonumber
\end{equation}
which is \eqref{eq:mode-collapse}.
Together with lexical invariance, the pair attains stability, which concludes the proof.
\end{proof}

\section{Derivation of the Information Bottleneck Residual}
\label{app:ib-derivation}

\Cref{def:minimality} states minimality as the information bottleneck objective~\eqref{eq:minimality}, whose two mutual informations depend on unknown distributions.
At the symmetric weight, the objective becomes a difference of two terms, each estimated by a cross-entropy of the model.
The measure is that difference.

Given a minimal thought representation, the output has low uncertainty, and the input has high uncertainty once the output is known.
At the symmetric weight $\beta_{\mathrm{IB}} = 2$, the negative objective $-J(2)$ equals $I(\mathbf{T}; Y) - I(X; \mathbf{T} \mid Y)$, where the second term is the information about the input that $\mathbf{T}$ encodes and $Y$ does not encode.
Each term is a difference of entropies.
The cross-entropy of a language model gives an upper bound on each conditional entropy.
The measure $\Delta_{\mathrm{IB}}$ keeps the two cross-entropy terms that depend on $\mathbf{T}$ and removes a constant that every candidate shares.

The chain rule of mutual information, applied to the pair $(X, Y)$, gives two decompositions of one quantity,
\begin{equation}
\label{eq:ib-chain}
I(\mathbf{T}; X, Y) = I(X; \mathbf{T}) + I(\mathbf{T}; Y \mid X) = I(\mathbf{T}; Y) + I(X; \mathbf{T} \mid Y).
\end{equation}
A thought representation computed from the input is a deterministic function of $X$, hence
\begin{equation}
\label{eq:ib-no-residual}
I(\mathbf{T}; Y \mid X) = 0, \qquad I(X; \mathbf{T}) = I(\mathbf{T}; Y) + I(X; \mathbf{T} \mid Y).
\end{equation}
Substituting~\eqref{eq:ib-no-residual} into the objective,
\begin{equation}
\label{eq:ib-collapse}
\begin{aligned}
J(\beta_{\mathrm{IB}}) &= I(X; \mathbf{T}) - \beta_{\mathrm{IB}} \, I(\mathbf{T}; Y) \\
&= I(\mathbf{T}; Y) + I(X; \mathbf{T} \mid Y) - \beta_{\mathrm{IB}} \, I(\mathbf{T}; Y) \\
&= (1 - \beta_{\mathrm{IB}}) \, I(\mathbf{T}; Y) + I(X; \mathbf{T} \mid Y).
\end{aligned}
\end{equation}
At $\beta_{\mathrm{IB}} = 2$ the two coefficients have equal magnitude and opposite sign,
\begin{equation}
\label{eq:ib-symmetric}
-J(2) = I(\mathbf{T}; Y) - I(X; \mathbf{T} \mid Y).
\end{equation}

Both terms of~\eqref{eq:ib-symmetric} are entropy differences,
\begin{equation}
\label{eq:ib-entropy-form}
I(\mathbf{T}; Y) = \mathbb{H}(Y) - \mathbb{H}(Y \mid \mathbf{T}), \qquad I(X; \mathbf{T} \mid Y) = \mathbb{H}(X \mid Y) - \mathbb{H}(X \mid Y, \mathbf{T}),
\end{equation}
and for random variables $A$ and $B$, the cross-entropy under an approximating conditional $Q(A \mid B)$ exceeds the conditional entropy by a divergence,
\begin{equation}
\label{eq:ce-kl}
\mathrm{CE}(A \mid B) = \mathbb{H}(A \mid B) + \mathbb{E}_{B}\big[ D_{\mathrm{KL}}\big( P(A \mid B) \,\|\, Q(A \mid B) \big) \big] \geq \mathbb{H}(A \mid B),
\end{equation}
with equality when $Q$ matches the true conditional~\citep{barber2003variational}.
The minimality measure takes $Q$ as the next-token distribution of the model under its two prompts, which gives
\begin{equation}
\label{eq:ib-surrogates}
\begin{aligned}
I(\mathbf{T}; Y) &\approx \mathbb{H}(Y) - \mathrm{CE}(Y \mid \mathbf{T}), \\
I(X; \mathbf{T} \mid Y) &\approx \mathrm{CE}(X \mid Y, \varnothing) - \mathrm{CE}(X \mid Y, \mathbf{T}),
\end{aligned}
\end{equation}
where the prompts of the two cross-entropies in the second line both include the output $Y$ and differ in whether they include $\mathbf{T}$.

Substituting~\eqref{eq:ib-surrogates} into~\eqref{eq:ib-symmetric},
\begin{equation}
\label{eq:ib-reduction}
\begin{aligned}
-J(2) &\approx \big[ \mathbb{H}(Y) - \mathrm{CE}(Y \mid \mathbf{T}) \big] - \big[ \mathrm{CE}(X \mid Y, \varnothing) - \mathrm{CE}(X \mid Y, \mathbf{T}) \big] \\
&= \mathbb{H}(Y) - \mathrm{CE}(Y \mid \mathbf{T}) - \mathrm{CE}(X \mid Y, \varnothing) + \mathrm{CE}(X \mid Y, \mathbf{T}) \\
&= \big[ \mathrm{CE}(X \mid Y, \mathbf{T}) - \mathrm{CE}(Y \mid \mathbf{T}) \big] + \big[ \mathbb{H}(Y) - \mathrm{CE}(X \mid Y, \varnothing) \big].
\end{aligned}
\end{equation}
The first bracket is the residual $\Delta_{\mathrm{IB}}$ of the minimality measure, and the second is the constant
\begin{equation}
\label{eq:ib-constant}
C_{\mathrm{IB}} = \mathbb{H}(Y) - \mathrm{CE}(X \mid Y, \varnothing).
\end{equation}
Neither term of $C_{\mathrm{IB}}$ depends on the candidate.
Since $C_{\mathrm{IB}}$ changes with the output distribution of the model, $\Delta_{\mathrm{IB}}$ is compared among the candidates of one model.

Relation~\eqref{eq:ce-kl} gives the excess of each cross-entropy over the corresponding conditional entropy,
\begin{equation}
\label{eq:ib-probe-gap}
\mathrm{CE}(X \mid Y, \mathbf{T}) = \mathbb{H}(X \mid Y, \mathbf{T}) + \varepsilon_X(\mathbf{T}), \qquad \mathrm{CE}(Y \mid \mathbf{T}) = \mathbb{H}(Y \mid \mathbf{T}) + \varepsilon_Y(\mathbf{T}),
\end{equation}
with $\varepsilon_X(\mathbf{T}) \geq 0$ and $\varepsilon_Y(\mathbf{T}) \geq 0$ by Gibbs' inequality.
Substituting~\eqref{eq:ib-probe-gap} into the residual,
\begin{equation}
\label{eq:ib-decomposed}
\Delta_{\mathrm{IB}} = \big[ \mathbb{H}(X \mid Y, \mathbf{T}) - \mathbb{H}(Y \mid \mathbf{T}) \big] + \big[ \varepsilon_X(\mathbf{T}) - \varepsilon_Y(\mathbf{T}) \big].
\end{equation}
Both cross-entropies come from one model under one pair of prompts for every candidate, hence the audit assumes that the second bracket, the difference between the two errors, is constant across the candidates, and attributes the differences in $\Delta_{\mathrm{IB}}$ among them to the first bracket.

Without~\eqref{eq:ib-no-residual}, the chain rule~\eqref{eq:ib-chain} instead gives
\begin{equation}
\label{eq:ib-correction}
\begin{aligned}
I(X; \mathbf{T}) &= I(\mathbf{T}; Y) + I(X; \mathbf{T} \mid Y) - I(\mathbf{T}; Y \mid X), \\
-J(2) &= I(\mathbf{T}; Y) - I(X; \mathbf{T} \mid Y) + I(\mathbf{T}; Y \mid X).
\end{aligned}
\end{equation}
The correction term $I(\mathbf{T}; Y \mid X)$ depends on the candidate, hence it cannot be absorbed into $C_{\mathrm{IB}}$.
A representation computed from the generated output satisfies $I(\mathbf{T}; Y \mid X) > 0$, and its cross-entropies approach
\begin{equation}
\label{eq:ib-anchor}
\mathrm{CE}(Y \mid \mathbf{T}) \to 0, \qquad \mathrm{CE}(X \mid Y, \mathbf{T}) \to \mathrm{CE}(X \mid Y, \varnothing), \qquad \Delta_{\mathrm{IB}} \to \mathrm{CE}(X \mid Y, \varnothing),
\end{equation}
which is the largest value of the measure.
The output text is therefore an anchor rather than an estimate of the objective.

\section{Implementation Details}
\label{app:details}

\textbf{Generation.}
Every model runs in its native precision and generates $K$ outputs per input by beam search (\cref{sec:setup}).
To save compute, causality and minimality on Llama-3.3-70B run with FP8 weights on a subset of questions that spans the range of output lengths and show similar trends as other models.
The experiments run on NVIDIA H100 GPUs with 80GB of memory.
The last input token candidate is the hidden state at the final position of the prompt before the first decoding step, taken from the final layer or from every layer.

\textbf{Candidates.}
Soft Thinking replaces each decoded token by the token embeddings $\mathbf{e}_v$ weighted by the next-token distribution $p = \mathrm{softmax}(z)$ over the logits $z$, $\mathbf{T} = \sum_v p_v \mathbf{e}_v$~\citep{zhang2025softthinkingunlockingreasoning}.
Stochastic Soft Thinking perturbs the logits with Gumbel noise before the softmax, $\hat{p}_v \propto \exp(z_v + \epsilon_v)$ with $\epsilon_v \sim \mathrm{Gumbel}(0, 1)$~\citep{wu2026llms}.
Latent Thinking supplies the last hidden state as the next input embedding, as in Coconut~\citep{hao2024training}, with the implementation of \citet{zou2025latentcollaborationmultiagentsystems}.
Since \cref{def:generator} requires a deterministic generator, every stochastic extraction runs under one fixed random seed, which gives each input one thought representation.
The output embedding and the input embedding are computed with Llama-Embed-Nemotron-8B~\citep{babakhin2025llamaembednemotron8buniversaltextembedding}.

\textbf{Trained methods.}
CODI and SIM-CoT use their released checkpoints on Llama-3.2-1B-Instruct and GPT-2, in which each of six latent steps is a learned projection of the last hidden state.
CoLaR samples each latent step from a learned distribution and ends the latent steps at a learned end token.
Each method samples eight outputs per question under the default decoding configuration of Llama-3.2-1B-Instruct.
The causality error takes the answer segment of each output as $v$.
The random vector matches the norm of each latent step.
The minimality prompts enclose $\mathbf{T}$ in the latent delimiters of each method.

\textbf{Causality measure.}
The causality error splits each output into the reasoning prefix $u$ and the final $n = 32$ tokens $v$, and excludes the outputs shorter than $n + 1$ tokens.
The model computes the next-token distributions over $v$ by teacher forcing under its chat template, once after $[x; u]$ and once after $[x; \mathbf{T}]$.
The random vector has as many positions as $\mathbf{T}$ and the mean norm of the token embeddings of the model.
The thought representation of another question is taken from a question of the same task.
The measure omits the hidden states of the last input token, which lie outside the input embedding space.

\textbf{Minimality measure.}
Each cross-entropy is the mean over the tokens of its target.
Since a thought representation that raises $\mathrm{CE}(X \mid Y, \mathbf{T})$ above $\mathrm{CE}(X \mid Y, \varnothing)$ does not encode additional input information, $r_X$ is clipped at zero.
When $\mathrm{CE}(X \mid Y, \varnothing) - \mathrm{CE}(X \mid Y, X) < 0.1 \, \mathrm{CE}(X \mid Y, \varnothing)$, the denominator of $r_X$ is near zero, hence $r_X$ is set to zero for that model.

\textbf{Separability measure.}
Each candidate is reduced to one vector per input by averaging over its positions.
Each output embedding is normalized to unit length before the output embeddings of an input are averaged.
The ridge regression is fitted in closed form with standardized inputs, and its penalty is selected from a logarithmic grid by the same-task score on held-out problems.
Separability holds out problems by an $80/10/10$ split of each BBEH task and by five-fold cross-validation on GSM8K.
The penalty bounds the norm of the linear map, as the class $\mathcal{H}$ of \cref{app:proofs} bounds the norm of its projections.
The same-task score compares each held-out input with every other held-out input of its task, and the cross-task score compares it with every held-out input of the other tasks.
A gap of exactly zero counts as half a separated pair.
The pooled output embedding equals the mean output embedding of its input, hence it separates every pair of questions whose outputs differ.

\textbf{Stability measure.}
The semantic entropy of an input is computed from the sizes of the classes that the threshold $\tau = 0.9$ forms among its $K$ outputs,
\begin{equation}
\label{eq:semantic-entropy}
\mathbb{H}_x = - \sum_{r} \frac{|\mathcal{E}_r|}{K} \log \frac{|\mathcal{E}_r|}{K}.
\end{equation}
As for separability, every candidate is reduced to the mean over its positions, which for a training-free method averages its latent steps and for the exact output embedding averages the $K$ outputs.
The difference-of-means direction and its AUROC are computed under five-fold stratified cross-validation over the inputs.
The DCS is the mean AUROC over the folds, reported with the standard deviation over the folds.

\textbf{Uncertainty.}
The uncertainty of every measure is estimated by a bootstrap that resamples problems with replacement, each together with all of its outputs.
The bootstrap draws $200$ resamples for separability and $1{,}000$ for causality and minimality.
All settings were fixed before any evaluation.

\section{Additional Results}
\label{app:results}

\Cref{tab:separability-results,tab:stability-results} report the separability and the DCS of every candidate at every number of latent steps on every model.
\Cref{tab:causality-results,tab:minimality-results} report the reasoning recovered, the output information, and the extra input information of every training-free method at every number of latent steps.
These four tables also report the training-free methods on Llama-3.2-1B-Instruct over GSM8K, and \cref{tab:trained-results} reports the four measures of every trained method.
The trained audit covers the train and test splits of GSM8K, with the changes between them reported in \cref{fig:split-shift}.
A negative value in \cref{tab:causality-results,tab:minimality-results,tab:trained-results} means that the thought representation or random vector raises the causality error or the cross-entropy above its value at $\varnothing$.
\Cref{fig:properties} clips such values at zero.

\textbf{Separability without training.}
Nearly every representation other than the random vector separates questions from different tasks (\cref{fig:separability-analysis}(a)).
Only questions from the same task distinguish the representations, which motivates the same-task pairs of the separability score.
At every margin, the thought representations separate fewer same-task pairs than the input embedding (b).
The gap to the input embedding is consequently independent of the choice of $\delta$.
A single latent step of Soft Thinking separates same-task questions only slightly better than the random vector (c).
The separability of Soft Thinking thus arises from the mean over its latent steps rather than from any single step.
In contrast, the first latent step of Latent Thinking separates questions nearly as well as the input embedding.
Each further step lowers this separability.
When combined with the input embedding, the output embedding raises the same-task score substantially, whereas each thought representation adds a fraction of this gain (d).
Hence the outputs hold the information that separates same-task questions and that the input embedding lacks.
The thought representations encode only a limited portion of this information.
The gap to the input embedding also widens on tasks with longer outputs (e).
Tasks with long reasoning accordingly offer a new method the largest margin for improvement.

\begin{figure}[t]
  \centering
  \includegraphics[width=\linewidth]{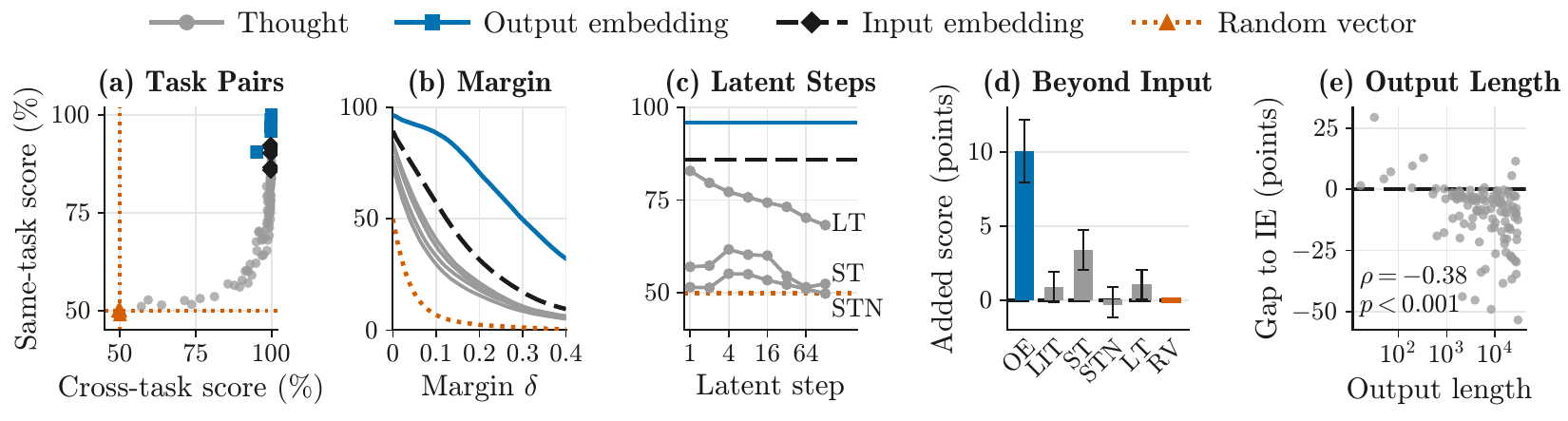}
  \caption{\textbf{Separability for training-free methods.} (a) Cross-task against same-task score. (b) Same-task pairs separated by more than $\delta$. (c) Same-task score of a single latent step. (d) Same-task score that each method adds given the input embedding. (e) Gap between input embedding against mean output length.}
  \label{fig:separability-analysis}
\end{figure}

\begin{table}[!htbp]
\caption{Separability score (\%) on same-task (Same) and cross-task (Cross) question pairs.}
\label{tab:separability-results}
\centering
\footnotesize
\setlength{\tabcolsep}{2.5pt}
\renewcommand{\arraystretch}{1.1}
\resizebox{\textwidth}{!}{%
\begin{tabular}{@{}l c *{2}{c} !{\vrule} *{2}{c} *{5}{c} *{5}{c} *{5}{c} !{\vrule} *{2}{c}@{}}
    \toprule
    & & \multicolumn{2}{c}{\textit{Output Emb.}}
      & \multicolumn{2}{c}{\textit{Last Input Tok.}}
      & \multicolumn{5}{c}{\textit{Soft Thinking}}
      & \multicolumn{5}{c}{\textit{Stochastic Soft Thinking}}
      & \multicolumn{5}{c}{\textit{Latent Thinking}}
      & \multicolumn{2}{c}{\textit{References}} \\
    \cmidrule(lr){3-4}\cmidrule(lr){5-6}\cmidrule(lr){7-11}\cmidrule(lr){12-16}\cmidrule(lr){17-21}\cmidrule(lr){22-23}
    Model & Pairing
      & Exact & Pooled
      & All & Final
      & 1 & 16 & 32 & 64 & 128
      & 1 & 16 & 32 & 64 & 128
      & 1 & 16 & 32 & 64 & 128
      & Input Emb. & Random Vec. \\
    \midrule
    \multirow{2}{*}{Llama-3.1 8B}
      & Same  & \ciSE{95.90}{0.60} & \ciSE{99.98}{0.01} & \ciSE{82.48}{1.14} & \ciSE{83.12}{1.14} & \ciSE{57.03}{1.49} & \ciSE{75.23}{1.38} & \ciSE{78.72}{1.35} & \ciSE{77.07}{1.28} & \ciSE{78.23}{1.22} & \ciSE{51.55}{1.45} & \ciSE{65.31}{1.54} & \ciSE{68.02}{1.41} & \ciSE{71.20}{1.38} & \ciSE{72.24}{1.38} & \ciSE{82.86}{1.18} & \ciSE{79.27}{1.29} & \ciSE{77.90}{1.33} & \ciSE{76.43}{1.35} & \ciSE{75.62}{1.30} & \ciSE{85.92}{1.21} & \ciSE{50.01}{1.39} \\
      & Cross & \ciSE{99.89}{0.04} & \ciSE{100.00}{0.00} & \ciSE{99.68}{0.10} & \ciSE{99.74}{0.08} & \ciSE{92.25}{0.88} & \ciSE{98.90}{0.30} & \ciSE{99.48}{0.14} & \ciSE{99.54}{0.15} & \ciSE{99.37}{0.14} & \ciSE{73.67}{1.34} & \ciSE{98.46}{0.34} & \ciSE{98.74}{0.25} & \ciSE{99.27}{0.14} & \ciSE{99.19}{0.16} & \ciSE{99.71}{0.10} & \ciSE{99.68}{0.09} & \ciSE{99.67}{0.09} & \ciSE{99.65}{0.10} & \ciSE{99.58}{0.11} & \ciSE{99.75}{0.08} & \ciSE{50.01}{1.20} \\
    \midrule
    \multirow{2}{*}{Llama-3.3 70B}
      & Same  & \ciSE{90.47}{1.00} & \ciSE{99.96}{0.02} & \ciSE{87.12}{1.10} & \ciSE{86.99}{1.08} & \ciSE{53.07}{1.35} & \ciSE{72.96}{1.31} & \ciSE{75.66}{1.29} & \ciSE{79.18}{1.25} & \ciSE{81.75}{1.15} & \ciSE{52.57}{1.33} & \ciSE{68.77}{1.29} & \ciSE{72.20}{1.33} & \ciSE{77.85}{1.27} & \ciSE{79.72}{1.21} & \ciSE{86.72}{1.05} & \ciSE{84.57}{1.19} & \ciSE{83.38}{1.21} & \ciSE{82.27}{1.20} & \ciSE{81.35}{1.22} & \ciSE{86.60}{1.01} & \ciSE{50.02}{1.35} \\
      & Cross & \ciSE{95.17}{0.55} & \ciSE{100.00}{0.00} & \ciSE{99.72}{0.10} & \ciSE{99.66}{0.09} & \ciSE{76.50}{1.20} & \ciSE{98.66}{0.21} & \ciSE{98.62}{0.26} & \ciSE{98.46}{0.28} & \ciSE{98.36}{0.29} & \ciSE{71.35}{1.18} & \ciSE{97.77}{0.32} & \ciSE{98.45}{0.29} & \ciSE{98.75}{0.27} & \ciSE{98.87}{0.26} & \ciSE{99.62}{0.11} & \ciSE{99.63}{0.11} & \ciSE{99.60}{0.11} & \ciSE{99.59}{0.11} & \ciSE{99.60}{0.11} & \ciSE{99.73}{0.07} & \ciSE{49.43}{1.35} \\
    \midrule
    \multirow{2}{*}{DS-R1-Qwen 32B}
      & Same  & \ciSE{97.10}{0.56} & \ciSE{99.95}{0.02} & \ciSE{89.68}{0.87} & \ciSE{90.12}{0.85} & \ciSE{56.45}{1.37} & \ciSE{74.71}{1.43} & \ciSE{77.21}{1.42} & \ciSE{79.72}{1.31} & \ciSE{80.94}{1.25} & \ciSE{52.72}{1.43} & \ciSE{64.07}{1.48} & \ciSE{69.20}{1.40} & \ciSE{73.40}{1.39} & \ciSE{75.65}{1.33} & \ciSE{88.47}{0.93} & \ciSE{84.29}{1.06} & \ciSE{83.54}{1.17} & \ciSE{84.16}{1.15} & \ciSE{84.04}{1.14} & \ciSE{91.57}{0.86} & \ciSE{49.07}{1.44} \\
      & Cross & \ciSE{99.65}{0.09} & \ciSE{100.00}{0.00} & \ciSE{99.84}{0.06} & \ciSE{99.86}{0.05} & \ciSE{87.76}{0.98} & \ciSE{99.42}{0.19} & \ciSE{99.63}{0.17} & \ciSE{99.73}{0.09} & \ciSE{99.50}{0.16} & \ciSE{59.38}{1.35} & \ciSE{96.81}{0.59} & \ciSE{99.22}{0.17} & \ciSE{99.57}{0.11} & \ciSE{99.55}{0.12} & \ciSE{99.92}{0.04} & \ciSE{99.92}{0.04} & \ciSE{99.90}{0.05} & \ciSE{99.87}{0.05} & \ciSE{99.72}{0.18} & \ciSE{99.94}{0.03} & \ciSE{50.09}{1.25} \\
    \midrule
    \multirow{2}{*}{Skywork-OR1 32B}
      & Same  & \ciSE{98.73}{0.33} & \ciSE{99.96}{0.02} & \ciSE{89.47}{0.94} & \ciSE{89.96}{0.90} & \ciSE{56.81}{1.34} & \ciSE{74.54}{1.33} & \ciSE{77.57}{1.27} & \ciSE{81.90}{1.18} & \ciSE{82.90}{1.16} & \ciSE{51.05}{1.50} & \ciSE{64.15}{1.48} & \ciSE{71.07}{1.50} & \ciSE{74.25}{1.37} & \ciSE{78.32}{1.28} & \ciSE{87.97}{0.89} & \ciSE{83.71}{1.19} & \ciSE{83.27}{1.20} & \ciSE{83.70}{1.18} & \ciSE{83.97}{1.25} & \ciSE{92.30}{0.81} & \ciSE{50.00}{1.51} \\
      & Cross & \ciSE{99.88}{0.04} & \ciSE{100.00}{0.00} & \ciSE{99.95}{0.02} & \ciSE{99.91}{0.04} & \ciSE{85.63}{1.07} & \ciSE{99.15}{0.35} & \ciSE{99.24}{0.30} & \ciSE{99.71}{0.14} & \ciSE{99.58}{0.19} & \ciSE{57.13}{1.42} & \ciSE{97.32}{0.50} & \ciSE{98.52}{0.40} & \ciSE{99.03}{0.35} & \ciSE{99.76}{0.08} & \ciSE{99.95}{0.02} & \ciSE{99.90}{0.04} & \ciSE{99.79}{0.09} & \ciSE{99.66}{0.20} & \ciSE{99.64}{0.17} & \ciSE{99.72}{0.21} & \ciSE{50.01}{1.25} \\
    \midrule
    \multirow{2}{*}{GPT-OSS 20B}
      & Same  & \ciSE{99.99}{0.01} & \ciSE{100.00}{0.00} & \ciSE{70.05}{1.51} & \ciSE{67.59}{1.56} & \ciSE{53.57}{1.55} & \ciSE{60.09}{1.41} & \ciSE{60.04}{1.41} & \ciSE{61.80}{1.49} & \ciSE{62.12}{1.45} & \ciSE{51.42}{1.58} & \ciSE{56.03}{1.44} & \ciSE{57.78}{1.52} & \ciSE{59.04}{1.50} & \ciSE{60.27}{1.53} & \ciSE{65.27}{1.35} & \ciSE{68.04}{1.55} & \ciSE{68.29}{1.40} & \ciSE{69.87}{1.47} & \ciSE{70.30}{1.43} & \ciSE{90.25}{1.15} & \ciSE{50.44}{1.59} \\
      & Cross & \ciSE{100.00}{0.00} & \ciSE{100.00}{0.00} & \ciSE{95.77}{0.71} & \ciSE{94.97}{0.83} & \ciSE{81.22}{1.14} & \ciSE{91.33}{0.96} & \ciSE{92.35}{0.95} & \ciSE{92.94}{0.95} & \ciSE{94.85}{0.76} & \ciSE{63.80}{1.46} & \ciSE{89.51}{0.93} & \ciSE{90.31}{1.00} & \ciSE{93.47}{0.80} & \ciSE{91.81}{0.96} & \ciSE{95.19}{0.84} & \ciSE{96.70}{0.57} & \ciSE{97.23}{0.51} & \ciSE{96.69}{0.57} & \ciSE{95.88}{0.72} & \ciSE{99.78}{0.10} & \ciSE{50.02}{1.41} \\
    \midrule
    \makecell[l]{Llama-3.2 1B\\(GSM8K)}
      & Same  & \ciSE{99.92}{0.01} & \ciSE{100.00}{0.00} & -- & -- & \ciSE{86.73}{0.22} & \ciSE{99.77}{0.03} & \ciSE{99.85}{0.02} & \ciSE{99.79}{0.02} & \ciSE{99.59}{0.03} & \ciSE{65.48}{0.31} & \ciSE{97.90}{0.09} & \ciSE{98.79}{0.07} & \ciSE{98.69}{0.08} & \ciSE{98.26}{0.09} & \ciSE{99.99}{0.00} & \ciSE{99.97}{0.01} & \ciSE{99.96}{0.01} & \ciSE{99.93}{0.01} & \ciSE{99.87}{0.02} & \ciSE{99.98}{0.01} & \ciSE{49.97}{0.30} \\
    \bottomrule
  \end{tabular}%
}
\end{table}

\begin{table}[!htbp]
\caption{Distributional consistency score (\%) of every candidate.}
\label{tab:stability-results}
\centering
\footnotesize
\setlength{\tabcolsep}{2.5pt}
\renewcommand{\arraystretch}{1.1}
\resizebox{\textwidth}{!}{%
\begin{tabular}{@{}l *{2}{c} !{\vrule} *{2}{c} *{5}{c} *{5}{c} *{5}{c} !{\vrule} *{2}{c}@{}}
    \toprule
    & \multicolumn{2}{c}{\textit{Output Emb.}}
      & \multicolumn{2}{c}{\textit{Last Input Tok.}}
      & \multicolumn{5}{c}{\textit{Soft Thinking}}
      & \multicolumn{5}{c}{\textit{Stochastic Soft Thinking}}
      & \multicolumn{5}{c}{\textit{Latent Thinking}}
      & \multicolumn{2}{c}{\textit{References}} \\
    \cmidrule(lr){2-3}\cmidrule(lr){4-5}\cmidrule(lr){6-10}\cmidrule(lr){11-15}\cmidrule(lr){16-20}\cmidrule(lr){21-22}
    Model
      & Exact & Pooled
      & All & Final
      & 1 & 16 & 32 & 64 & 128
      & 1 & 16 & 32 & 64 & 128
      & 1 & 16 & 32 & 64 & 128
      & Input Emb. & Random Vec. \\
    \midrule
    Llama-3.1 8B & \ciSE{95.53}{1.30} & \ciSE{95.96}{1.20} & \ciSE{94.30}{1.08} & \ciSE{94.35}{1.18} & \ciSE{93.88}{0.96} & \ciSE{90.53}{1.44} & \ciSE{88.99}{1.25} & \ciSE{89.18}{1.74} & \ciSE{90.41}{1.58} & \ciSE{89.74}{1.17} & \ciSE{88.70}{1.33} & \ciSE{87.43}{1.39} & \ciSE{89.96}{1.59} & \ciSE{89.17}{1.71} & \ciSE{91.91}{1.37} & \ciSE{82.15}{1.56} & \ciSE{80.27}{1.69} & \ciSE{79.47}{1.29} & \ciSE{78.26}{1.19} & \ciSE{92.99}{0.97} & \ciSE{51.52}{2.57} \\
    \midrule
    Llama-3.3 70B & \ciSE{87.48}{0.81} & \ciSE{88.51}{0.87} & \ciSE{87.25}{0.47} & \ciSE{89.19}{0.82} & \ciSE{83.95}{1.08} & \ciSE{82.69}{0.87} & \ciSE{84.06}{1.06} & \ciSE{85.32}{0.71} & \ciSE{81.42}{1.09} & \ciSE{82.57}{1.20} & \ciSE{80.99}{1.58} & \ciSE{82.32}{1.83} & \ciSE{84.28}{1.23} & \ciSE{80.63}{1.11} & \ciSE{87.19}{0.85} & \ciSE{66.15}{2.24} & \ciSE{64.52}{2.52} & \ciSE{63.65}{2.46} & \ciSE{63.65}{2.47} & \ciSE{91.95}{1.44} & \ciSE{49.94}{1.27} \\
    \midrule
    DS-R1-Qwen 32B & \ciSE{95.18}{0.49} & \ciSE{95.74}{0.45} & \ciSE{94.95}{0.81} & \ciSE{95.06}{0.65} & \ciSE{95.30}{0.62} & \ciSE{90.06}{1.30} & \ciSE{86.39}{2.19} & \ciSE{84.65}{1.58} & \ciSE{83.50}{1.97} & \ciSE{83.84}{2.06} & \ciSE{82.57}{1.15} & \ciSE{83.06}{1.14} & \ciSE{84.80}{0.85} & \ciSE{84.83}{1.80} & \ciSE{91.62}{1.01} & \ciSE{69.40}{1.10} & \ciSE{67.23}{1.70} & \ciSE{65.72}{1.96} & \ciSE{63.82}{2.11} & \ciSE{93.52}{0.80} & \ciSE{49.99}{1.44} \\
    \midrule
    Skywork-OR1 32B & \ciSE{96.29}{0.42} & \ciSE{97.14}{0.25} & \ciSE{95.52}{0.47} & \ciSE{95.11}{0.64} & \ciSE{95.14}{0.45} & \ciSE{90.42}{0.65} & \ciSE{90.04}{1.04} & \ciSE{88.29}{1.00} & \ciSE{87.99}{1.00} & \ciSE{76.63}{1.04} & \ciSE{80.48}{2.31} & \ciSE{83.29}{1.05} & \ciSE{84.55}{1.03} & \ciSE{86.47}{1.47} & \ciSE{92.10}{0.88} & \ciSE{68.38}{1.38} & \ciSE{68.23}{1.69} & \ciSE{68.65}{2.01} & \ciSE{68.34}{2.27} & \ciSE{92.95}{1.21} & \ciSE{48.68}{0.75} \\
    \midrule
    GPT-OSS 20B & \ciSE{71.89}{7.37} & \ciSE{73.73}{7.41} & \ciSE{57.55}{6.95} & \ciSE{54.79}{9.18} & \ciSE{41.65}{3.94} & \ciSE{35.46}{9.00} & \ciSE{47.30}{7.35} & \ciSE{47.19}{8.81} & \ciSE{55.10}{6.82} & \ciSE{46.40}{2.51} & \ciSE{43.98}{6.74} & \ciSE{38.60}{6.93} & \ciSE{43.87}{7.41} & \ciSE{46.24}{4.84} & \ciSE{59.13}{8.87} & \ciSE{45.78}{5.26} & \ciSE{52.86}{6.25} & \ciSE{55.88}{6.82} & \ciSE{54.99}{5.48} & \ciSE{59.19}{11.96} & \ciSE{56.03}{5.55} \\
    \midrule
    \makecell[l]{Llama-3.2 1B\\(GSM8K)} & \ciSE{75.92}{0.75} & -- & -- & -- & \ciSE{66.14}{2.03} & \ciSE{61.07}{2.54} & \ciSE{59.57}{2.91} & \ciSE{60.34}{2.97} & \ciSE{56.18}{2.35} & \ciSE{59.19}{0.60} & \ciSE{57.33}{1.35} & \ciSE{56.17}{1.40} & \ciSE{57.90}{1.41} & \ciSE{56.21}{2.69} & \ciSE{72.40}{1.76} & \ciSE{60.56}{3.49} & \ciSE{59.41}{3.46} & \ciSE{58.79}{3.45} & \ciSE{58.45}{3.42} & \ciSE{65.97}{2.35} & \ciSE{50.15}{2.35} \\
    \bottomrule
  \end{tabular}%
}
\end{table}

\begin{table}[!htbp]
\caption{Reasoning recovered $R$ (\%) of every training-free method.}
\label{tab:causality-results}
\centering
\footnotesize
\setlength{\tabcolsep}{2.5pt}
\renewcommand{\arraystretch}{1.1}
\resizebox{\textwidth}{!}{%
\begin{tabular}{@{}l l *{5}{c} *{5}{c} *{5}{c}@{}}
    \toprule
    & & \multicolumn{5}{c}{\textit{Soft Thinking}}
      & \multicolumn{5}{c}{\textit{Stochastic Soft Thinking}}
      & \multicolumn{5}{c}{\textit{Latent Thinking}} \\
    \cmidrule(lr){3-7}\cmidrule(lr){8-12}\cmidrule(lr){13-17}
    Model & Arm
      & 1 & 16 & 32 & 64 & 128
      & 1 & 16 & 32 & 64 & 128
      & 1 & 16 & 32 & 64 & 128 \\
    \midrule
    \multirow{3}{*}{Llama-3.1 8B} & Thought & \ciSE{2.5}{1.0} & \ciSE{-2.2}{1.4} & \ciSE{0.8}{1.6} & \ciSE{0.7}{1.8} & \ciSE{2.1}{1.7} & \ciSE{2.9}{1.0} & \ciSE{-2.5}{1.4} & \ciSE{-1.5}{1.6} & \ciSE{-2.3}{1.8} & \ciSE{-1.7}{1.9} & \ciSE{16.0}{1.5} & \ciSE{-215}{18.1} & \ciSE{-283}{20.7} & \ciSE{-318}{22.4} & \ciSE{-358}{23.3} \\
     & Other question & \ciSE{3.0}{1.0} & \ciSE{1.6}{1.4} & \ciSE{8.3}{1.5} & \ciSE{9.4}{1.5} & \ciSE{6.5}{1.4} & \ciSE{3.2}{1.1} & \ciSE{2.8}{1.5} & \ciSE{6.3}{1.4} & \ciSE{5.5}{1.5} & \ciSE{4.5}{1.6} & \ciSE{16.2}{1.5} & \ciSE{-212}{17.7} & \ciSE{-279}{20.1} & \ciSE{-316}{21.3} & \ciSE{-352}{22.5} \\
     & Random & \ciSE{12.3}{1.0} & \ciSE{22.1}{0.9} & \ciSE{17.5}{1.1} & \ciSE{17.2}{1.2} & \ciSE{17.1}{1.1} & \ciSE{12.3}{1.0} & \ciSE{22.1}{0.9} & \ciSE{17.5}{1.1} & \ciSE{17.2}{1.2} & \ciSE{17.1}{1.1} & \ciSE{12.3}{1.0} & \ciSE{22.1}{0.9} & \ciSE{17.5}{1.1} & \ciSE{17.2}{1.2} & \ciSE{17.1}{1.1} \\
    \midrule
    \multirow{3}{*}{Llama-3.3 70B} & Thought & \ciSE{0.7}{4.2} & \ciSE{1.9}{8.3} & \ciSE{-9.6}{8.5} & \ciSE{-8.7}{7.9} & \ciSE{-11.0}{8.0} & \ciSE{0.0}{4.1} & \ciSE{3.2}{7.9} & \ciSE{-6.2}{9.5} & \ciSE{-9.7}{9.4} & \ciSE{-16.6}{12.1} & \ciSE{26.4}{4.1} & \ciSE{34.2}{3.3} & \ciSE{34.1}{3.6} & \ciSE{32.7}{3.6} & \ciSE{30.6}{4.0} \\
     & Other question & \ciSE{1.5}{3.9} & \ciSE{4.4}{4.6} & \ciSE{3.5}{3.9} & \ciSE{7.8}{8.0} & \ciSE{6.1}{9.4} & \ciSE{3.0}{4.7} & \ciSE{8.0}{6.4} & \ciSE{10.0}{5.2} & \ciSE{9.4}{5.1} & \ciSE{8.2}{8.1} & \ciSE{24.9}{5.8} & \ciSE{34.1}{3.0} & \ciSE{33.5}{2.9} & \ciSE{31.9}{2.2} & \ciSE{30.8}{2.1} \\
     & Random & \ciSE{21.1}{3.7} & \ciSE{31.0}{2.7} & \ciSE{32.6}{3.2} & \ciSE{32.1}{3.4} & \ciSE{32.4}{3.7} & \ciSE{21.1}{3.7} & \ciSE{31.0}{2.7} & \ciSE{32.6}{3.2} & \ciSE{32.1}{3.4} & \ciSE{32.4}{3.7} & \ciSE{21.1}{3.7} & \ciSE{31.0}{2.7} & \ciSE{32.6}{3.2} & \ciSE{32.1}{3.4} & \ciSE{32.4}{3.7} \\
    \midrule
    \multirow{3}{*}{DS-R1-Qwen 32B} & Thought & \ciSE{-16.7}{0.9} & \ciSE{-32.9}{1.6} & \ciSE{-33.4}{1.8} & \ciSE{-32.8}{1.9} & \ciSE{-33.8}{2.0} & \ciSE{-17.1}{0.9} & \ciSE{-31.0}{1.5} & \ciSE{-32.9}{1.6} & \ciSE{-32.9}{1.8} & \ciSE{-32.8}{1.9} & \ciSE{5.1}{0.8} & \ciSE{9.3}{0.7} & \ciSE{-9.3}{2.0} & \ciSE{-158}{10.3} & \ciSE{-244}{13.5} \\
     & Other question & \ciSE{-16.0}{0.9} & \ciSE{-27.1}{1.6} & \ciSE{-26.5}{1.6} & \ciSE{-24.0}{1.6} & \ciSE{-22.0}{1.7} & \ciSE{-16.7}{1.0} & \ciSE{-26.0}{1.6} & \ciSE{-26.0}{1.5} & \ciSE{-24.4}{1.6} & \ciSE{-23.1}{1.8} & \ciSE{5.4}{0.8} & \ciSE{9.4}{0.7} & \ciSE{-9.8}{1.9} & \ciSE{-162}{10.8} & \ciSE{-246}{13.3} \\
     & Random & \ciSE{-13.2}{0.9} & \ciSE{5.9}{0.8} & \ciSE{9.5}{0.8} & \ciSE{12.0}{0.9} & \ciSE{-2.0}{1.5} & \ciSE{-13.2}{0.9} & \ciSE{5.9}{0.8} & \ciSE{9.5}{0.8} & \ciSE{12.0}{0.9} & \ciSE{-2.0}{1.5} & \ciSE{-13.2}{0.9} & \ciSE{5.9}{0.8} & \ciSE{9.5}{0.8} & \ciSE{12.0}{0.9} & \ciSE{-2.0}{1.5} \\
    \midrule
    \multirow{3}{*}{Skywork-OR1 32B} & Thought & \ciSE{-13.5}{0.7} & \ciSE{-31.2}{1.4} & \ciSE{-34.7}{1.4} & \ciSE{-36.7}{1.4} & \ciSE{-42.0}{1.9} & \ciSE{-13.9}{0.8} & \ciSE{-29.6}{1.3} & \ciSE{-32.7}{1.3} & \ciSE{-36.3}{1.5} & \ciSE{-40.4}{1.9} & \ciSE{5.2}{0.7} & \ciSE{8.8}{0.6} & \ciSE{-7.8}{1.5} & \ciSE{-142}{9.1} & \ciSE{-200}{10.9} \\
     & Other question & \ciSE{-13.3}{0.8} & \ciSE{-26.0}{1.4} & \ciSE{-26.9}{1.5} & \ciSE{-26.0}{1.4} & \ciSE{-27.2}{1.6} & \ciSE{-13.8}{0.8} & \ciSE{-24.3}{1.4} & \ciSE{-25.2}{1.4} & \ciSE{-25.2}{1.4} & \ciSE{-26.3}{1.6} & \ciSE{5.5}{0.7} & \ciSE{8.7}{0.6} & \ciSE{-8.0}{1.7} & \ciSE{-140}{9.3} & \ciSE{-199}{11.0} \\
     & Random & \ciSE{-8.5}{0.7} & \ciSE{6.0}{0.7} & \ciSE{9.0}{0.7} & \ciSE{10.6}{1.0} & \ciSE{4.2}{1.3} & \ciSE{-8.5}{0.7} & \ciSE{6.0}{0.7} & \ciSE{9.0}{0.7} & \ciSE{10.6}{1.0} & \ciSE{4.2}{1.3} & \ciSE{-8.5}{0.7} & \ciSE{6.0}{0.7} & \ciSE{9.0}{0.7} & \ciSE{10.6}{1.0} & \ciSE{4.2}{1.3} \\
    \midrule
    \multirow{3}{*}{GPT-OSS 20B} & Thought & \ciSE{13.1}{1.7} & \ciSE{17.5}{2.5} & \ciSE{15.2}{2.6} & \ciSE{11.7}{2.9} & \ciSE{9.2}{3.1} & \ciSE{10.9}{2.1} & \ciSE{14.0}{2.9} & \ciSE{12.0}{3.0} & \ciSE{10.8}{3.2} & \ciSE{10.2}{3.3} & \ciSE{2.7}{1.4} & \ciSE{0.4}{2.1} & \ciSE{2.6}{1.9} & \ciSE{0.8}{2.1} & \ciSE{-2.0}{2.3} \\
     & Other question & \ciSE{12.7}{1.6} & \ciSE{16.6}{2.3} & \ciSE{14.8}{2.6} & \ciSE{11.3}{2.9} & \ciSE{9.0}{3.2} & \ciSE{10.2}{1.7} & \ciSE{14.0}{3.0} & \ciSE{12.0}{3.1} & \ciSE{10.7}{3.3} & \ciSE{10.0}{3.3} & \ciSE{3.5}{1.4} & \ciSE{0.5}{2.1} & \ciSE{2.1}{2.1} & \ciSE{0.9}{2.1} & \ciSE{-1.6}{2.1} \\
     & Random & \ciSE{-55.3}{4.4} & \ciSE{-48.1}{4.3} & \ciSE{-46.9}{4.5} & \ciSE{-48.0}{4.6} & \ciSE{-49.2}{4.5} & \ciSE{-55.3}{4.4} & \ciSE{-48.1}{4.3} & \ciSE{-46.9}{4.5} & \ciSE{-48.0}{4.6} & \ciSE{-49.2}{4.5} & \ciSE{-55.3}{4.4} & \ciSE{-48.1}{4.3} & \ciSE{-46.9}{4.5} & \ciSE{-48.0}{4.6} & \ciSE{-49.2}{4.5} \\
    \midrule
    \multirow{2}{*}{\makecell[l]{Llama-3.2 1B\\(GSM8K)}} & Thought & \ciSE{1.3}{0.2} & \ciSE{2.3}{0.3} & \ciSE{2.3}{0.4} & \ciSE{-9.8}{0.6} & \ciSE{-33.4}{0.6} & \ciSE{1.2}{0.2} & \ciSE{-2.1}{0.3} & \ciSE{-4.7}{0.4} & \ciSE{-13.8}{0.6} & \ciSE{-33.5}{0.6} & \ciSE{14.2}{0.2} & \ciSE{-15.9}{0.6} & \ciSE{-25.3}{1.0} & \ciSE{-36.3}{1.3} & \ciSE{-46.8}{1.6} \\
     & Random & \ciSE{17.7}{0.1} & \ciSE{11.9}{0.2} & \ciSE{9.5}{0.2} & \ciSE{5.3}{0.2} & \ciSE{1.1}{0.2} & \ciSE{17.7}{0.1} & \ciSE{11.9}{0.2} & \ciSE{9.5}{0.2} & \ciSE{5.3}{0.2} & \ciSE{1.1}{0.2} & \ciSE{17.7}{0.1} & \ciSE{11.9}{0.2} & \ciSE{9.5}{0.2} & \ciSE{5.3}{0.2} & \ciSE{1.1}{0.2} \\
    \bottomrule
\end{tabular}%
}
\end{table}

\begin{table}[!htbp]
\caption{Output information $r_Y$ (\%) and extra input information $r_X$ (\%) of every training-free method.}
\label{tab:minimality-results}
\centering
\footnotesize
\setlength{\tabcolsep}{2.5pt}
\renewcommand{\arraystretch}{1.1}
\resizebox{\textwidth}{!}{%
\begin{tabular}{@{}l l *{5}{c} *{5}{c} *{5}{c}@{}}
    \toprule
    & & \multicolumn{5}{c}{\textit{Soft Thinking}}
      & \multicolumn{5}{c}{\textit{Stochastic Soft Thinking}}
      & \multicolumn{5}{c}{\textit{Latent Thinking}} \\
    \cmidrule(lr){3-7}\cmidrule(lr){8-12}\cmidrule(lr){13-17}
    Model & Arm
      & 1 & 16 & 32 & 64 & 128
      & 1 & 16 & 32 & 64 & 128
      & 1 & 16 & 32 & 64 & 128 \\
    \midrule
    \multicolumn{17}{@{}l}{\textit{Output information} $r_Y$} \\
    \midrule
    \multirow{2}{*}{Llama-3.1 8B} & Thought & \ciSE{6.6}{1.4} & \ciSE{36.6}{3.8} & \ciSE{43.7}{4.1} & \ciSE{49.1}{3.8} & \ciSE{54.0}{3.5} & \ciSE{6.7}{1.3} & \ciSE{32.3}{3.6} & \ciSE{37.2}{3.8} & \ciSE{43.3}{4.0} & \ciSE{49.2}{3.7} & \ciSE{1.1}{3.2} & \ciSE{-686}{86.1} & \ciSE{-800}{94.6} & \ciSE{-898}{102} & \ciSE{-963}{105} \\
     & Random & \ciSE{-4.5}{1.4} & \ciSE{8.7}{1.9} & \ciSE{9.8}{2.0} & \ciSE{11.1}{2.1} & \ciSE{4.7}{2.2} & \ciSE{-4.5}{1.4} & \ciSE{8.7}{1.9} & \ciSE{9.8}{2.0} & \ciSE{11.1}{2.1} & \ciSE{4.7}{2.2} & \ciSE{-4.5}{1.4} & \ciSE{8.7}{1.9} & \ciSE{9.8}{2.0} & \ciSE{11.1}{2.1} & \ciSE{4.7}{2.2} \\
    \addlinespace
    \multirow{2}{*}{Llama-3.3 70B} & Thought & \ciSE{2.0}{1.9} & \ciSE{23.5}{15.6} & \ciSE{23.1}{13.7} & \ciSE{26.4}{13.5} & \ciSE{27.5}{11.1} & \ciSE{1.9}{2.0} & \ciSE{19.9}{12.5} & \ciSE{16.0}{12.9} & \ciSE{22.7}{12.0} & \ciSE{22.9}{12.3} & \ciSE{14.3}{8.7} & \ciSE{9.5}{5.9} & \ciSE{8.9}{5.1} & \ciSE{7.8}{4.5} & \ciSE{9.2}{5.1} \\
     & Random & \ciSE{1.8}{3.6} & \ciSE{7.1}{4.2} & \ciSE{6.2}{4.0} & \ciSE{6.2}{4.8} & \ciSE{7.7}{6.2} & \ciSE{1.8}{3.6} & \ciSE{7.1}{4.2} & \ciSE{6.2}{4.0} & \ciSE{6.2}{4.8} & \ciSE{7.7}{6.2} & \ciSE{1.8}{3.6} & \ciSE{7.1}{4.2} & \ciSE{6.2}{4.0} & \ciSE{6.2}{4.8} & \ciSE{7.7}{6.2} \\
    \addlinespace
    \multirow{2}{*}{DS-R1-Qwen 32B} & Thought & \ciSE{0.2}{0.1} & \ciSE{14.4}{4.0} & \ciSE{16.8}{4.1} & \ciSE{21.4}{3.8} & \ciSE{27.8}{3.7} & \ciSE{-0.1}{0.2} & \ciSE{9.3}{3.3} & \ciSE{12.3}{3.0} & \ciSE{16.6}{3.0} & \ciSE{24.2}{3.1} & \ciSE{0.7}{0.2} & \ciSE{3.0}{0.8} & \ciSE{-10.2}{2.2} & \ciSE{-65.3}{8.4} & \ciSE{-99.8}{9.2} \\
     & Random & \ciSE{3.4}{0.6} & \ciSE{5.6}{1.0} & \ciSE{4.6}{0.8} & \ciSE{4.1}{0.9} & \ciSE{-4.6}{1.4} & \ciSE{3.4}{0.6} & \ciSE{5.6}{1.0} & \ciSE{4.6}{0.8} & \ciSE{4.1}{0.9} & \ciSE{-4.6}{1.4} & \ciSE{3.4}{0.6} & \ciSE{5.6}{1.0} & \ciSE{4.6}{0.8} & \ciSE{4.1}{0.9} & \ciSE{-4.6}{1.4} \\
    \addlinespace
    \multirow{2}{*}{Skywork-OR1 32B} & Thought & \ciSE{-0.2}{0.2} & \ciSE{11.9}{5.1} & \ciSE{15.1}{5.1} & \ciSE{18.4}{4.9} & \ciSE{24.5}{5.0} & \ciSE{-0.4}{0.2} & \ciSE{8.9}{3.9} & \ciSE{13.0}{3.9} & \ciSE{17.0}{4.5} & \ciSE{24.3}{5.8} & \ciSE{0.8}{0.2} & \ciSE{3.3}{0.8} & \ciSE{-10.0}{2.0} & \ciSE{-67.5}{8.4} & \ciSE{-101}{9.7} \\
     & Random & \ciSE{2.2}{0.5} & \ciSE{5.2}{1.1} & \ciSE{3.7}{0.8} & \ciSE{2.9}{0.6} & \ciSE{-6.7}{2.0} & \ciSE{2.2}{0.5} & \ciSE{5.2}{1.1} & \ciSE{3.7}{0.8} & \ciSE{2.9}{0.6} & \ciSE{-6.7}{2.0} & \ciSE{2.2}{0.5} & \ciSE{5.2}{1.1} & \ciSE{3.7}{0.8} & \ciSE{2.9}{0.6} & \ciSE{-6.7}{2.0} \\
    \addlinespace
    \multirow{2}{*}{GPT-OSS 20B} & Thought & \ciSE{-0.7}{1.1} & \ciSE{2.4}{2.5} & \ciSE{1.0}{2.3} & \ciSE{1.0}{2.8} & \ciSE{-0.6}{4.0} & \ciSE{1.6}{1.6} & \ciSE{4.0}{3.3} & \ciSE{1.8}{4.2} & \ciSE{3.1}{4.9} & \ciSE{-1.4}{5.9} & \ciSE{-316}{38.7} & \ciSE{-884}{98.9} & \ciSE{-937}{105} & \ciSE{-964}{108} & \ciSE{-993}{114} \\
     & Random & \ciSE{-341}{40.7} & \ciSE{-221}{37.7} & \ciSE{-194}{35.7} & \ciSE{-177}{35.0} & \ciSE{-169}{35.1} & \ciSE{-341}{40.7} & \ciSE{-221}{37.7} & \ciSE{-194}{35.7} & \ciSE{-177}{35.0} & \ciSE{-169}{35.1} & \ciSE{-341}{40.7} & \ciSE{-221}{37.7} & \ciSE{-194}{35.7} & \ciSE{-177}{35.0} & \ciSE{-169}{35.1} \\
    \addlinespace
    \multirow{2}{*}{\makecell[l]{Llama-3.2 1B\\(GSM8K)}} & Thought & \ciSE{0.7}{0.1} & \ciSE{28.0}{0.4} & \ciSE{39.5}{0.4} & \ciSE{47.6}{0.4} & \ciSE{46.2}{0.4} & \ciSE{1.2}{0.2} & \ciSE{25.8}{0.4} & \ciSE{36.4}{0.4} & \ciSE{43.7}{0.4} & \ciSE{42.9}{0.4} & \ciSE{1.6}{0.1} & \ciSE{-20.6}{0.6} & \ciSE{-28.3}{1.1} & \ciSE{-37.9}{1.6} & \ciSE{-48.9}{2.2} \\
     & Random & \ciSE{-0.8}{0.1} & \ciSE{-2.2}{0.2} & \ciSE{-3.5}{0.2} & \ciSE{-4.9}{0.2} & \ciSE{-4.5}{0.2} & \ciSE{-0.8}{0.1} & \ciSE{-2.2}{0.2} & \ciSE{-3.5}{0.2} & \ciSE{-4.9}{0.2} & \ciSE{-4.5}{0.2} & \ciSE{-0.8}{0.1} & \ciSE{-2.2}{0.2} & \ciSE{-3.5}{0.2} & \ciSE{-4.9}{0.2} & \ciSE{-4.5}{0.2} \\
    \midrule
    \multicolumn{17}{@{}l}{\textit{Extra input information} $r_X$} \\
    \midrule
    \multirow{2}{*}{Llama-3.1 8B} & Thought & \ciSE{-0.2}{0.1} & \ciSE{-0.5}{0.3} & \ciSE{0.4}{0.4} & \ciSE{1.5}{0.6} & \ciSE{2.8}{0.8} & \ciSE{-0.1}{0.1} & \ciSE{-0.1}{0.2} & \ciSE{0.7}{0.4} & \ciSE{2.0}{0.6} & \ciSE{3.9}{0.9} & \ciSE{-2.8}{0.4} & \ciSE{-938}{92.5} & \ciSE{-1064}{104} & \ciSE{-1171}{113} & \ciSE{-1252}{114} \\
     & Random & \ciSE{0.0}{0.1} & \ciSE{-0.4}{0.1} & \ciSE{-0.4}{0.2} & \ciSE{-0.4}{0.2} & \ciSE{-0.2}{0.2} & \ciSE{0.0}{0.1} & \ciSE{-0.4}{0.1} & \ciSE{-0.4}{0.2} & \ciSE{-0.4}{0.2} & \ciSE{-0.2}{0.2} & \ciSE{0.0}{0.1} & \ciSE{-0.4}{0.1} & \ciSE{-0.4}{0.2} & \ciSE{-0.4}{0.2} & \ciSE{-0.2}{0.2} \\
    \addlinespace
    \multirow{2}{*}{Llama-3.3 70B} & Thought & \ciSE{0.4}{0.1} & \ciSE{0.5}{0.5} & \ciSE{0.1}{0.5} & \ciSE{0.0}{0.6} & \ciSE{0.8}{1.1} & \ciSE{0.5}{0.2} & \ciSE{0.5}{0.6} & \ciSE{0.2}{0.6} & \ciSE{-0.4}{0.6} & \ciSE{0.8}{1.6} & \ciSE{0.8}{0.3} & \ciSE{1.6}{0.5} & \ciSE{1.5}{0.6} & \ciSE{1.4}{0.6} & \ciSE{0.9}{1.0} \\
     & Random & \ciSE{0.8}{0.3} & \ciSE{1.4}{0.3} & \ciSE{1.4}{0.4} & \ciSE{1.3}{0.4} & \ciSE{1.2}{0.6} & \ciSE{0.8}{0.3} & \ciSE{1.4}{0.3} & \ciSE{1.4}{0.4} & \ciSE{1.3}{0.4} & \ciSE{1.2}{0.6} & \ciSE{0.8}{0.3} & \ciSE{1.4}{0.3} & \ciSE{1.4}{0.4} & \ciSE{1.3}{0.4} & \ciSE{1.2}{0.6} \\
    \addlinespace
    \multirow{2}{*}{DS-R1-Qwen 32B} & Thought & \ciSE{0.0}{0.0} & \ciSE{0.4}{0.2} & \ciSE{1.0}{0.3} & \ciSE{1.9}{0.5} & \ciSE{3.3}{1.0} & \ciSE{0.1}{0.0} & \ciSE{0.7}{0.2} & \ciSE{1.5}{0.4} & \ciSE{3.1}{0.7} & \ciSE{5.1}{1.1} & \ciSE{0.0}{0.0} & \ciSE{0.5}{0.1} & \ciSE{-1.1}{0.3} & \ciSE{-6.6}{1.0} & \ciSE{-20.4}{2.4} \\
     & Random & \ciSE{0.1}{0.0} & \ciSE{1.0}{0.1} & \ciSE{1.0}{0.1} & \ciSE{-0.1}{0.1} & \ciSE{-6.5}{0.3} & \ciSE{0.1}{0.0} & \ciSE{1.0}{0.1} & \ciSE{1.0}{0.1} & \ciSE{-0.1}{0.1} & \ciSE{-6.5}{0.3} & \ciSE{0.1}{0.0} & \ciSE{1.0}{0.1} & \ciSE{1.0}{0.1} & \ciSE{-0.1}{0.1} & \ciSE{-6.5}{0.3} \\
    \addlinespace
    \multirow{2}{*}{Skywork-OR1 32B} & Thought & \ciSE{0.1}{0.0} & \ciSE{0.6}{0.2} & \ciSE{1.3}{0.3} & \ciSE{1.8}{0.6} & \ciSE{3.3}{0.8} & \ciSE{0.0}{0.0} & \ciSE{0.6}{0.2} & \ciSE{1.2}{0.4} & \ciSE{2.1}{0.7} & \ciSE{4.4}{1.0} & \ciSE{-0.1}{0.0} & \ciSE{0.9}{0.1} & \ciSE{-0.8}{0.5} & \ciSE{-7.1}{1.6} & \ciSE{-18.2}{2.9} \\
     & Random & \ciSE{-0.1}{0.0} & \ciSE{1.3}{0.1} & \ciSE{1.3}{0.1} & \ciSE{0.5}{0.2} & \ciSE{-9.2}{0.6} & \ciSE{-0.1}{0.0} & \ciSE{1.3}{0.1} & \ciSE{1.3}{0.1} & \ciSE{0.5}{0.2} & \ciSE{-9.2}{0.6} & \ciSE{-0.1}{0.0} & \ciSE{1.3}{0.1} & \ciSE{1.3}{0.1} & \ciSE{0.5}{0.2} & \ciSE{-9.2}{0.6} \\
    \addlinespace
    \multirow{2}{*}{GPT-OSS 20B} & Thought & \ciSE{-1.0}{0.4} & \ciSE{-0.5}{0.9} & \ciSE{-3.5}{1.1} & \ciSE{-3.3}{1.1} & \ciSE{-4.1}{1.4} & \ciSE{-0.4}{0.6} & \ciSE{0.9}{1.1} & \ciSE{-0.7}{1.3} & \ciSE{0.0}{1.5} & \ciSE{-0.3}{1.7} & \ciSE{-178}{23.8} & \ciSE{-626}{85.2} & \ciSE{-690}{95.9} & \ciSE{-753}{103} & \ciSE{-802}{110} \\
     & Random & \ciSE{-114}{19.6} & \ciSE{-206}{35.3} & \ciSE{-180}{33.3} & \ciSE{-184}{34.1} & \ciSE{-172}{34.8} & \ciSE{-114}{19.6} & \ciSE{-206}{35.3} & \ciSE{-180}{33.3} & \ciSE{-184}{34.1} & \ciSE{-172}{34.8} & \ciSE{-114}{19.6} & \ciSE{-206}{35.3} & \ciSE{-180}{33.3} & \ciSE{-184}{34.1} & \ciSE{-172}{34.8} \\
    \addlinespace
    \multirow{2}{*}{\makecell[l]{Llama-3.2 1B\\(GSM8K)}} & Thought & \ciSE{-1.3}{0.0} & \ciSE{-1.5}{0.1} & \ciSE{1.3}{0.2} & \ciSE{4.6}{0.2} & \ciSE{7.0}{0.2} & \ciSE{-1.6}{0.0} & \ciSE{-0.4}{0.1} & \ciSE{2.7}{0.2} & \ciSE{6.2}{0.2} & \ciSE{8.9}{0.2} & \ciSE{-0.5}{0.0} & \ciSE{-9.5}{0.6} & \ciSE{-17.2}{1.1} & \ciSE{-26.1}{1.7} & \ciSE{-38.6}{2.1} \\
     & Random & \ciSE{-0.6}{0.0} & \ciSE{-0.1}{0.1} & \ciSE{-0.5}{0.1} & \ciSE{-1.3}{0.1} & \ciSE{-0.3}{0.1} & \ciSE{-0.6}{0.0} & \ciSE{-0.1}{0.1} & \ciSE{-0.5}{0.1} & \ciSE{-1.3}{0.1} & \ciSE{-0.3}{0.1} & \ciSE{-0.6}{0.0} & \ciSE{-0.1}{0.1} & \ciSE{-0.5}{0.1} & \ciSE{-1.3}{0.1} & \ciSE{-0.3}{0.1} \\
    \bottomrule
\end{tabular}%
}
\end{table}

\begin{table}[!htbp]
\caption{The four measures (\%) and latent accuracy (\%) of every trained method on GSM8K.}
\label{tab:trained-results}
\centering
\footnotesize
\setlength{\tabcolsep}{8pt}
\renewcommand{\arraystretch}{1.1}
\begin{tabular}{@{}l *{5}{c}@{}}
    \toprule
    & \multicolumn{3}{c}{\textit{Llama-3.2 1B}} & \multicolumn{2}{c}{\textit{GPT-2}} \\
    \cmidrule(lr){2-4}\cmidrule(lr){5-6}
    Arm & CODI & SIM-CoT & CoLaR & CODI & SIM-CoT \\
    \midrule
    \multicolumn{6}{@{}l}{\textit{Separability}} \\
    \midrule
    Thought & \ciSE{98.75}{0.06} & \ciSE{99.55}{0.04} & \ciSE{99.18}{0.04} & \ciSE{97.65}{0.07} & \ciSE{99.46}{0.03} \\
    Output Emb. & \ciSE{99.99}{0.00} & \ciSE{99.98}{0.01} & \ciSE{99.99}{0.00} & \ciSE{99.98}{0.01} & \ciSE{99.98}{0.01} \\
    Input Emb. & \ciSE{91.25}{0.18} & \ciSE{91.56}{0.17} & \ciSE{95.39}{0.13} & \ciSE{93.27}{0.15} & \ciSE{93.11}{0.15} \\
    Random Vec. & \ciSE{49.78}{0.29} & \ciSE{49.85}{0.28} & \ciSE{49.94}{0.30} & \ciSE{49.96}{0.28} & \ciSE{49.98}{0.29} \\
    \midrule
    \multicolumn{6}{@{}l}{\textit{Distributional consistency score}} \\
    \midrule
    Thought & \ciSE{78.97}{1.18} & \ciSE{80.84}{1.00} & \ciSE{62.55}{1.28} & \ciSE{78.50}{0.72} & \ciSE{80.60}{0.85} \\
    Output Emb. & \ciSE{65.03}{1.24} & \ciSE{65.31}{1.44} & \ciSE{68.94}{1.87} & \ciSE{66.81}{0.21} & \ciSE{67.62}{1.34} \\
    Input Emb. & \ciSE{56.42}{1.51} & \ciSE{57.36}{2.45} & \ciSE{60.22}{0.82} & \ciSE{62.05}{1.03} & \ciSE{61.56}{1.97} \\
    Random Vec. & \ciSE{48.71}{1.21} & \ciSE{49.96}{1.05} & \ciSE{50.57}{1.81} & \ciSE{51.20}{1.02} & \ciSE{51.86}{1.86} \\
    \midrule
    \multicolumn{6}{@{}l}{\textit{Reasoning recovered} $R$} \\
    \midrule
    Thought & \ciSE{89.0}{0.5} & \ciSE{92.2}{0.3} & \ciSE{2.1}{2.2} & \ciSE{83.8}{0.4} & \ciSE{76.5}{0.5} \\
    Other question & \ciSE{39.9}{0.9} & \ciSE{45.8}{0.6} & \ciSE{-104}{3.0} & \ciSE{55.5}{0.6} & \ciSE{38.6}{0.8} \\
    Random & \ciSE{64.2}{0.5} & \ciSE{74.1}{0.3} & \ciSE{4.7}{0.8} & \ciSE{16.3}{0.3} & \ciSE{10.1}{0.3} \\
    \midrule
    \multicolumn{6}{@{}l}{\textit{Output information} $r_Y$} \\
    \midrule
    Thought & \ciSE{4.3}{0.3} & \ciSE{10.4}{0.2} & \ciSE{28.3}{0.3} & \ciSE{33.4}{0.5} & \ciSE{25.3}{0.4} \\
    Random & \ciSE{3.5}{0.2} & \ciSE{1.8}{0.1} & \ciSE{-1.3}{0.1} & \ciSE{-17.2}{0.6} & \ciSE{-1.1}{0.6} \\
    \midrule
    \multicolumn{6}{@{}l}{\textit{Extra input information} $r_X$} \\
    \midrule
    Thought & \ciSE{0.0}{0.0} & \ciSE{0.0}{0.0} & \ciSE{-4.9}{0.1} & \ciSE{0.0}{0.0} & \ciSE{0.0}{0.0} \\
    Random & \ciSE{0.0}{0.0} & \ciSE{0.0}{0.0} & \ciSE{-2.6}{0.1} & \ciSE{0.0}{0.0} & \ciSE{0.0}{0.0} \\
    \midrule
    \multicolumn{6}{@{}l}{\textit{Latent accuracy}} \\
    \midrule
    Thought & $55.5$ & $55.7$ & $25.5$ & $42.5$ & $39.3$ \\
    \bottomrule
\end{tabular}
\end{table}

\begin{figure}[!htbp]
  \centering
  \includegraphics[width=\linewidth]{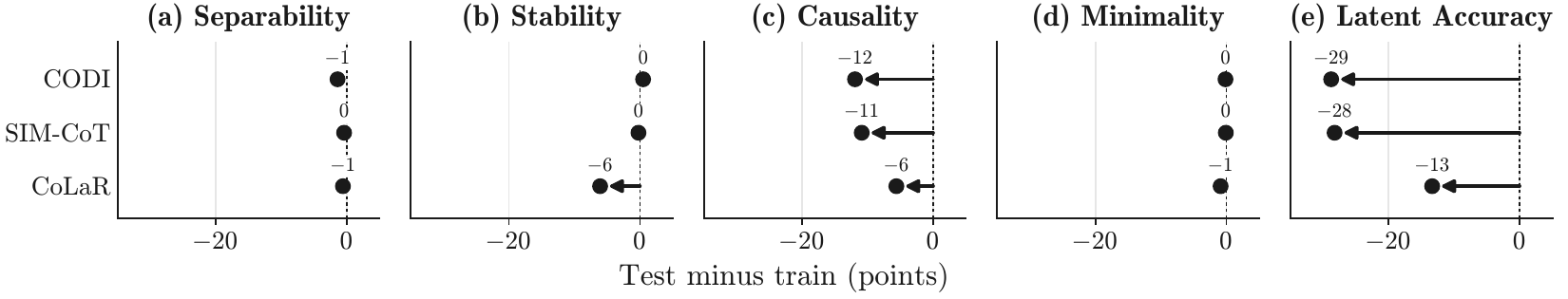}
  \caption{\textbf{Train and test.} Change of each measure from the GSM8K train questions to the test questions.}
  \label{fig:split-shift}
\end{figure}

\clearpage

\section{Supporting Diagnostics}
\label{app:diagnostics}

\paragraph{Robustness of the measures.}
The causality measure averages the divergence over the last $n$ tokens of each output.
The output lengths of the models differ by an order of magnitude (\cref{app:dataset}).
\Cref{fig:measure-robustness}(a) recomputes the reasoning recovered at windows from $16$ to $512$ tokens.
\Cref{fig:measure-robustness}(b) pools it in bins of output length, with the outputs that reach the generation cap in a separate bin.
The reasoning recovered of every training-free method decreases as the window grows.
At every window and in every length bin, the training-free methods recover at most a limited portion of the reasoning, about as much as the thought representation of another question, as in \cref{fig:properties}(c).
The stability measure binarizes the cosine similarity between the output embeddings at the threshold $\tau$.
\Cref{fig:measure-robustness}(c) recomputes the DCS at thresholds from $0.7$ to $0.95$.
At every threshold, every training-free method exceeds the random vector and at most matches the input embedding, as in \cref{fig:properties}(b).
Hence the verdicts of \cref{sec:results-main} hold at every window, output length, and threshold.

\begin{figure}[t]
  \centering
  \includegraphics[width=\linewidth]{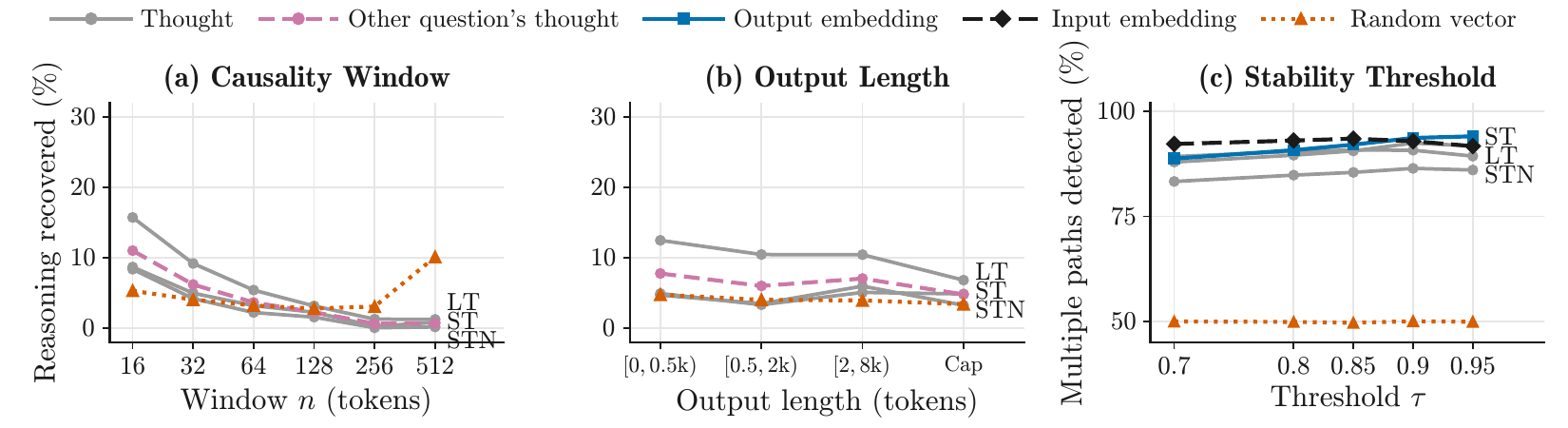}
  \caption{\textbf{Robustness of the measures.} Training-free methods at their best step, averaged over models. (a) Reasoning recovered at each window of $n$ tokens. (b) Reasoning recovered in bins of output length. (c) DCS at each threshold $\tau$ on the four dense models.}
  \label{fig:measure-robustness}
\end{figure}

\paragraph{Stability on GPT-OSS-20B.}
The difference-of-means direction separates the inputs whose outputs form more than one semantic class from the rest, hence the measure requires inputs on which the model generates more than one class.
At the threshold of the audit, GPT-OSS-20B generates more than one class on one percent of the inputs, compared with 16 to 46 percent on the four dense models.
The DCS of the random vector exceeds one half on this model, since the positive class contains few inputs.

\paragraph{Linear encoding of the semantic entropy.}
The semantic entropy is computed from the spread of the output embeddings, whereas a difference-of-means direction scores their mean position, hence the output embedding is a reference informed by the outputs rather than an upper bound.
A difference-of-means direction recovers the semantic entropy from a representation that encodes it linearly~\citep{cencerrado2026answerneededpredictingllm}.
A representation that encodes the semantic entropy through a nonlinear structure scores near the random vector under this direction.
A DCS near the random vector is therefore evidence against a linear encoding of the semantic entropy.

\section{Dataset Details}
\label{app:dataset}

\paragraph{Tasks.}
BBEH replaces each of the 23 tasks of BIG-Bench Hard with a harder task that tests a similar reasoning skill, such as many-hop deduction, tracking of objects and time over long contexts, and learning of new rules from examples~\citep{kazemi2025bigbenchextrahard}.
Every task contains $200$ problems except Disambiguation QA, which contains $120$.
Every model receives the prompt of the benchmark.

\paragraph{Accuracy and output length.}
\Cref{tab:bbeh-summary} reports the accuracy and the output length of each model.
The adjusted harmonic mean of Llama-3.1-8B and of DeepSeek-R1-Distill-Qwen-32B differs by less than one percentage point from the value that BBEH reports, hence the generation reproduces the published accuracy of these models.
A failed extraction lowers the accuracy without changing any measure, since every measure is computed on the generated tokens.
Output length differs by an order of magnitude across the models.
\Cref{fig:measure-robustness}(b) reports the reasoning recovered in bins of output length.

\begin{table}[!htbp]
\caption{BBEH accuracy (\%), failed extraction (\%), and median output length in characters of each model.}
\label{tab:bbeh-summary}
\centering
\small
\begin{tabular}{@{}l c c c c c@{}}
\toprule
Model & Pass@$1$ & Pass@$8$ & Harmonic mean & Failed extraction & Median length \\
\midrule
Llama-3.1 8B & $7.9$ & $21.0$ & $3.1$ & $42.9$ & $2{,}076$ \\
Llama-3.3 70B & $20.3$ & $31.3$ & $8.7$ & $4.3$ & $2{,}091$ \\
DS-R1-Qwen 32B & $17.2$ & $23.7$ & $5.5$ & $34.9$ & $7{,}458$ \\
Skywork-OR1 32B & $20.2$ & $23.1$ & $5.3$ & $49.6$ & $17{,}043$ \\
GPT-OSS 20B & $19.1$ & $20.8$ & $4.0$ & $58.3$ & $19{,}316$ \\
\bottomrule
\end{tabular}
\end{table}

\paragraph{Qualitative examples.}
Two outputs of the same input share one input embedding, yet their output embeddings may be far apart.
The boxes below give two outputs that Llama-3.3-70B generates for one input, with the cosine similarity of their output embeddings under Llama-Embed-Nemotron-8B.
The two outputs differ both in their reasoning and in their verdict, since one output derives the correct answer over several steps and the other states an incorrect answer in one line.

\begin{outputbox}{Output 2 of Llama-3.3-70B-Instruct (Causal Understanding, example 113)}
\textbf{Question (truncated):}\\
Jim, Carol, Bob, and Nancy share electricity from 8--9 PM. The breaker fails if three lamps are on at the same time. Jim is not supposed to use his lamp on Thursdays. Carol, Bob, and Nancy may use theirs. This Thursday Jim, Carol, and Bob all turn on their lamps at 8 PM and the breaker fails. Did Jim cause the breaker to fail? Reply Yes / No / Ambiguous.

\vspace{4pt}
\hrule
\vspace{8pt}
\textbf{Output 2:}\\
To determine if Jim turning on his lamp at 8 PM caused the circuit breaker to fail, let's analyze the situation step by step.
1.~The breaker will fail if three people turn on their lamps at the same time.
2.~Jim is not supposed to use his lamp on Thursdays, but he turned it on anyway.
3.~Carol and Bob also turned on their lamps at the same time as Jim.
4.~Since Carol and Bob are allowed to use their lamps on Thursdays, their actions did not cause the breaker to fail.
5.~The breaker failed because three people (Jim, Carol, and Bob) turned on their lamps at the same time.
Considering these points, the majority of people would likely say that Jim turning on his lamp at 8 PM caused the circuit breaker to fail\ldots

\textbf{Verdict:} Yes \textcolor{improvetxt}{\textit{Correct Answer} \ding{51}}
\end{outputbox}

\begin{outputbox}{Output 7 of Llama-3.3-70B-Instruct (Causal Understanding, example 113)}
\textbf{Question (truncated):}\\
Jim, Carol, Bob, and Nancy share electricity from 8--9 PM. The breaker fails if three lamps are on at the same time. Jim is not supposed to use his lamp on Thursdays. Carol, Bob, and Nancy may use theirs. This Thursday Jim, Carol, and Bob all turn on their lamps at 8 PM and the breaker fails. Did Jim cause the breaker to fail? Reply Yes / No / Ambiguous.

\vspace{4pt}
\hrule
\vspace{8pt}
\textbf{Output 7:}\\
The answer is: No.

\textbf{Verdict:} No \textcolor{worsetxt}{\textit{Wrong Answer} \ding{55}}

\vspace{4pt}
\hrule
\vspace{8pt}
\textbf{Cosine similarity to Output 2:} $0.063$
\end{outputbox}

The eight outputs of this input form three groups under the embedder, and the outputs within a group are nearly identical.
One input thus yields outputs that differ in their reasoning and in their verdict, whereas all eight share one input embedding.
A thought representation that exceeds the input embedding on separability and stability therefore encodes information about the outputs that the input lacks.

\section{Extended Related Work}
\label{app:related}

\paragraph{Surveys and theoretical analyses.}
Surveys organize latent reasoning methods into token-wise methods, which replace discrete tokens with continuous vectors, and layer-wise methods, which propagate a latent state across the depth of the transformer~\citep{chen2025reasoninglanguagecomprehensivesurvey}.
Discrete chain of thought restricts reasoning to high-certainty decisions, whereas continuous representations support exploration at the cost of amplified noise on computational tasks, a trade-off governed by a scalar index of decisional certainty~\citep{zou2026the}.
Evaluations of reasoning in language models report downstream accuracy, which conflates the quality of the reasoning with pattern matching~\citep{mondorf2024accuracyevaluatingreasoningbehavior}.
Surveys of efficient reasoning compare methods by their compute cost and their length budget~\citep{sui2025stopoverthinkingsurveyefficient,feng2025efficientreasoningmodelssurvey}.
These works catalogue methods and their downstream trade-offs, whereas the measures of \cref{sec:properties} evaluate the thought representation independently of the decoding strategy and of downstream task accuracy.

\paragraph{Internal world models and the localization of computation.}
Propositional probes extract the states of the internal world models of language models as structured propositions~\citep{feng2025monitoring}.
Object concept representations that resemble those of humans emerge in multimodal language models without explicit supervision~\citep{Du_2025}.
Swapping ranges of transformer layers transfers cross-lingual knowledge between models~\citep{bandarkar2025layerswappingzeroshotcrosslingual}.
Adaptation concentrates in the early and middle adapter layers~\citep{alabi-etal-2024-hidden}.
Semantically meaningful computation is thus localized at identifiable depths.
Middle layers can be reordered or skipped at a small cost in accuracy, which indicates that representations at middle depths share their structure~\citep{sun2025transformerlayerspainters}.
Circuit tracing through the contributions to the residual stream shows that feature-level computation shifts across layers depending on the prompt~\citep{ameisen2025circuit}.
Since the depth of the relevant computation changes with the prompt, the audit examines candidates taken from the final layer, from every layer, and from latent steps (\cref{sec:setup}).

\paragraph{Sentence-level representations.}
Large Concept Models~\citep{lcm2024} predict the next sentence embedding instead of the next token, with SONAR~\citep{duquenne2023sonarsentencelevelmultimodallanguageagnostic} as the shared multilingual sentence encoder.
SONAR-LLM~\citep{dragunov2025sonarllm} reasons autoregressively in the sentence embedding space and decodes back to tokens, at a granularity coarser than a token and finer than a document.
These methods fix the semantic granularity of their representations in advance.
Since \cref{def:valid-space} states validity for a thought representation of any granularity, the four measures apply to sentence-level representations as well as to the token-level candidates of \cref{sec:experiments}.

\end{document}